\RequirePackage{fix-cm}
\documentclass{article}
\usepackage{iclr2027_conference}
\usepackage{times}

\setcitestyle{numbers,square,comma}

\usepackage[T1]{fontenc}
\usepackage{microtype}
\usepackage{placeins}
\usepackage{hyphenat}
\usepackage{setspace}
\usepackage[english]{babel}
\usepackage{etoolbox}
\usepackage{graphicx}
\usepackage{subcaption}
\usepackage{booktabs}
\usepackage{nicematrix}
\usepackage{multirow}
\usepackage{bm}
\usepackage[most]{tcolorbox}
\usepackage{xcolor}
\usepackage{caption}
\usepackage{hyperref}
\usepackage{fontawesome5}
\usepackage[noabbrev,nameinlink]{cleveref}

\definecolor{metablue}{HTML}{0064E0}
\definecolor{metafg}{HTML}{1C2B33}
\definecolor{metabg}{HTML}{F1F4F7}
\newcommand{\beginappendix}{\appendix{\huge\bfseries Appendix\par}\vspace{0.5em}}

\usepackage{enumitem}
\usepackage{wrapfig}
\usepackage{needspace}%
\AddToHook{env/wrapfigure/begin}{\setlength{\intextsep}{0pt}}
\AddToHook{cmd/WF@putfigmaybe/before}{\setlength{\intextsep}{0pt}}
\usepackage{amsthm}
\newtheorem{theorem}{Theorem}
\newtheorem{proposition}{Proposition}

\theoremstyle{definition}
\newtheorem{definition}{Definition}
\theoremstyle{plain}
\Crefname{theorem}{Theorem}{Theorems}
\Crefname{proposition}{Proposition}{Propositions}
\Crefname{lemma}{Lemma}{Lemmas}
\Crefname{definition}{Definition}{Definitions}
\crefformat{section}{\S#2#1#3}\Crefformat{section}{\S#2#1#3}
\crefformat{subsection}{\S#2#1#3}\Crefformat{subsection}{\S#2#1#3}
\crefformat{subsubsection}{\S#2#1#3}\Crefformat{subsubsection}{\S#2#1#3}
\usepackage{amssymb}
\usepackage{pifont}
\usepackage{float}
\usepackage{amsmath}
\usepackage{makecell}
\usepackage{titletoc}
\usepackage{tikz}
\usetikzlibrary{arrows.meta, positioning, calc, decorations.pathreplacing,
                fit, backgrounds, shapes.geometric, shapes.gates.logic.US}
\graphicspath{{figures/}}

\newcommand{\xmark}{\textcolor{red!85!black}{\ding{55}}}
\newcommand{\cmark}{\textcolor{green!55!black}{\ding{51}}}
\newcommand{\pmark}{\textcolor{gray}{$\sim$}}

\newcommand{\AntTmonoLone}{0.48}
\newcommand{\AntTmonoHwmLtwo}{0.52}
\newcommand{\AntTmonoHwmProj}{0.94}
\newcommand{\AntTmonoHjepaLtwo}{0.68}
\newcommand{\AntTmonoHjepaProj}{0.94}
\newcommand{\PushtTmonoLone}{0.75}
\newcommand{\PushtTmonoHwmLtwo}{0.77}
\newcommand{\PushtTmonoHwmProj}{0.85}
\newcommand{\PushtTmonoHjepaLtwo}{0.77}
\newcommand{\PushtTmonoHjepaProj}{0.89}
\newcommand{\OgbTmonoLone}{0.65}
\newcommand{\OgbTmonoHwmLtwo}{0.74}
\newcommand{\OgbTmonoHwmProj}{0.79}
\newcommand{\OgbTmonoHjepaLtwo}{1.00}
\newcommand{\OgbTmonoHjepaProj}{0.99}
\newcommand{\DistrTmonoLone}{0.56}
\newcommand{\DistrTmonoHwmLtwo}{0.62}
\newcommand{\DistrTmonoHwmProj}{0.92}
\newcommand{\DistrTmonoHjepaLtwo}{0.85}
\newcommand{\DistrTmonoHjepaProj}{1.00}

\newcommand{\AntDLTwoNative}{$23.3 \pm 3.3$}
\newcommand{\AntDLTwoLtwo}{$31.3 \pm 2.7$}
\newcommand{\AntDLThreeNative}{$16.7 \pm 1.8$}
\newcommand{\AntDLThreeLtwo}{$20.7 \pm 4.8$}
\newcommand{\AntDLThreeLthree}{$22.0 \pm 2.3$}
\newcommand{\AntDLFourNative}{$10.0 \pm 0.0$}
\newcommand{\AntDLFourLtwo}{$20.7 \pm 1.8$}
\newcommand{\AntDLFourLthree}{$21.3 \pm 4.7$}
\newcommand{\AntDLFourLfour}{$26.7 \pm 2.9$}
\newcommand{\DistrDLTwoNative}{$9.3 \pm 0.7$}
\newcommand{\DistrDLTwoLtwo}{$38.7 \pm 6.6$}
\newcommand{\DistrDLThreeNative}{$43.3 \pm 8.7$}
\newcommand{\DistrDLThreeLtwo}{$65.3 \pm 5.5$}
\newcommand{\DistrDLThreeLthree}{$71.3 \pm 4.1$}
\newcommand{\DistrDLFourNative}{$68.7 \pm 4.7$}
\newcommand{\DistrDLFourLtwo}{$72.0 \pm 3.1$}
\newcommand{\DistrDLFourLthree}{$66.0 \pm 3.1$}
\newcommand{\DistrDLFourLfour}{$70.7 \pm 2.7$}
\newcommand{\CubeDLTwoNative}{$24.0 \pm 1.2$}
\newcommand{\CubeDLTwoLtwo}{$13.3 \pm 0.7$}
\newcommand{\CubeDLThreeNative}{$24.7 \pm 1.3$}
\newcommand{\CubeDLThreeLtwo}{$28.0 \pm 4.2$}
\newcommand{\CubeDLThreeLthree}{$14.0 \pm 3.5$}
\newcommand{\CubeDLFourNative}{$26.7 \pm 2.7$}
\newcommand{\CubeDLFourLtwo}{$25.3 \pm 2.4$}
\newcommand{\CubeDLFourLthree}{$12.0 \pm 2.0$}
\newcommand{\CubeDLFourLfour}{$12.0 \pm 2.3$}
\newcommand{\PushtDLTwoNative}{$36.7 \pm 1.8$}
\newcommand{\PushtDLTwoLtwo}{$38.0 \pm 3.1$}
\newcommand{\PushtDLThreeNative}{$30.7 \pm 1.8$}
\newcommand{\PushtDLThreeLtwo}{$28.0 \pm 4.0$}
\newcommand{\PushtDLThreeLthree}{$24.7 \pm 4.1$}
\newcommand{\PushtDLFourNative}{$1.3 \pm 0.7$}
\newcommand{\PushtDLFourLtwo}{$1.3 \pm 0.7$}
\newcommand{\PushtDLFourLthree}{$0.7 \pm 0.7$}
\newcommand{\PushtDLFourLfour}{$1.3 \pm 1.3$}

\newcommand{\AntDLLewm}{$18.0 \pm 3.5$}
\newcommand{\AntDLTwoHier}{$39.3 \pm 3.7$}
\newcommand{\AntDLThreeHier}{$73.3 \pm 3.5$}
\newcommand{\AntDLFourHier}{$63.3 \pm 7.7$}
\newcommand{\DistrDLLewm}{$40.7 \pm 6.8$}
\newcommand{\DistrDLTwoHier}{$81.3 \pm 6.7$}
\newcommand{\DistrDLThreeHier}{$96.0 \pm 1.1$}
\newcommand{\DistrDLFourHier}{$80.0 \pm 4.0$}
\newcommand{\CubeDLLewm}{$32.0 \pm 6.1$}
\newcommand{\CubeDLTwoHier}{$47.3 \pm 2.7$}
\newcommand{\CubeDLThreeHier}{$60.0 \pm 5.0$}
\newcommand{\CubeDLFourHier}{$55.3 \pm 4.4$}
\newcommand{\PushtDLLewm}{$40.0 \pm 2.3$}
\newcommand{\PushtDLTwoHier}{$45.3 \pm 1.3$}
\newcommand{\PushtDLThreeHier}{$17.3 \pm 3.5$}
\newcommand{\PushtDLFourHier}{$0.7 \pm 0.7$}
\newcommand{\ImplRoomImage}{224}
\newcommand{\ImplRoomPatch}{14}
\newcommand{\ImplRoomAction}{2}
\newcommand{\ImplRoomLatent}{192}
\newcommand{\ImplRoomLayers}{4}
\newcommand{\ImplRoomHeads}{8}
\newcommand{\ImplRoomWidth}{192}
\newcommand{\ImplRoomMlp}{1024}
\newcommand{\ImplRoomLeWMLambda}{0.72}
\newcommand{\ImplRoomLambda}{0.18}
\newcommand{\ImplRoomEpochs}{7}
\newcommand{\ImplAntImage}{64}
\newcommand{\ImplAntPatch}{4}
\newcommand{\ImplAntAction}{8}
\newcommand{\ImplAntLatent}{384}
\newcommand{\ImplAntLayers}{5}
\newcommand{\ImplAntHeads}{12}
\newcommand{\ImplAntWidth}{224}
\newcommand{\ImplAntMlp}{1536}
\newcommand{\ImplAntLeWMLambda}{0.09}
\newcommand{\ImplAntLambda}{0.09}
\newcommand{\ImplAntEpochs}{5}
\newcommand{\ImplPushImage}{224}
\newcommand{\ImplPushPatch}{14}
\newcommand{\ImplPushAction}{2}
\newcommand{\ImplPushLatent}{192}
\newcommand{\ImplPushLayers}{6}
\newcommand{\ImplPushHeads}{16}
\newcommand{\ImplPushWidth}{192}
\newcommand{\ImplPushMlp}{2048}
\newcommand{\ImplPushLeWMLambda}{0.09}
\newcommand{\ImplPushLambda}{0.09}
\newcommand{\ImplPushEpochs}{9}
\newcommand{\ImplCubeImage}{224}
\newcommand{\ImplCubePatch}{14}
\newcommand{\ImplCubeAction}{5}
\newcommand{\ImplCubeLatent}{192}
\newcommand{\ImplCubeLayers}{6}
\newcommand{\ImplCubeHeads}{16}
\newcommand{\ImplCubeWidth}{192}
\newcommand{\ImplCubeMlp}{2048}
\newcommand{\ImplCubeLeWMLambda}{0.09}
\newcommand{\ImplCubeLambda}{0.09}
\newcommand{\ImplCubeEpochs}{8}
\newcommand{\ImplProprio}{27}
\newcommand{\ImplFramesOne}{4}
\newcommand{\ImplFramesTwo}{7}
\newcommand{\ImplFramesThree}{13}
\newcommand{\ImplFramesFour}{25}
\newcommand{\ImplSpanOne}{16}
\newcommand{\ImplSpanTwo}{31}
\newcommand{\ImplSpanThree}{61}
\newcommand{\ImplSpanFour}{121}
\newcommand{\ImplRoomClipsOne}{952{,}708}
\newcommand{\ImplRoomClipsTwo}{915{,}298}
\newcommand{\ImplRoomClipsThree}{840{,}478}
\newcommand{\ImplRoomClipsFour}{690{,}838}
\newcommand{\ImplRoomClipsTwoPct}{96}
\newcommand{\ImplRoomClipsThreePct}{88}
\newcommand{\ImplRoomClipsFourPct}{73}
\newcommand{\ImplAntClipsOne}{9{,}575{,}000}
\newcommand{\ImplAntClipsTwo}{9{,}200{,}000}
\newcommand{\ImplAntClipsThree}{8{,}450{,}000}
\newcommand{\ImplAntClipsFour}{6{,}950{,}000}
\newcommand{\ImplAntClipsTwoPct}{96}
\newcommand{\ImplAntClipsThreePct}{88}
\newcommand{\ImplAntClipsFourPct}{73}
\newcommand{\ImplPushClipsOne}{1{,}981{,}721}
\newcommand{\ImplPushClipsTwo}{1{,}701{,}446}
\newcommand{\ImplPushClipsThree}{1{,}143{,}623}
\newcommand{\ImplPushClipsFour}{285{,}931}
\newcommand{\ImplPushClipsTwoPct}{86}
\newcommand{\ImplPushClipsThreePct}{58}
\newcommand{\ImplPushClipsFourPct}{14}
\newcommand{\ImplCubeClipsOne}{1{,}783{,}600}
\newcommand{\ImplCubeClipsTwo}{1{,}636{,}600}
\newcommand{\ImplCubeClipsThree}{1{,}342{,}600}
\newcommand{\ImplCubeClipsFour}{754{,}600}
\newcommand{\ImplCubeClipsTwoPct}{92}
\newcommand{\ImplCubeClipsThreePct}{75}
\newcommand{\ImplCubeClipsFourPct}{42}
\newcommand{\PlanRoomBudget}{210}
\newcommand{\PlanRoomFlatLR}{0.1}
\newcommand{\PlanRoomHierLR}{0.08, 0.1}
\newcommand{\PlanRoomSamplesOne}{200}
\newcommand{\PlanRoomSamplesTwo}{100, 800}
\newcommand{\PlanRoomSamplesThree}{50, 200, 200}
\newcommand{\PlanRoomSamplesFour}{50, 200, 200, 400}
\newcommand{\PlanRoomHorizonsOne}{23}
\newcommand{\PlanRoomHorizonsTwo}{2, 12}
\newcommand{\PlanRoomHorizonsThree}{2, 2, 6}
\newcommand{\PlanRoomHorizonsFour}{2, 2, 2, 3}
\newcommand{\PlanAntBudget}{350}
\newcommand{\PlanAntFlatLR}{0.3}
\newcommand{\PlanAntHierLR}{0.06, 0.1}
\newcommand{\PlanAntSamplesOne}{100}
\newcommand{\PlanAntSamplesTwo}{25, 200}
\newcommand{\PlanAntSamplesThree}{100, 100, 400}
\newcommand{\PlanAntSamplesFour}{50, 200, 25, 50}
\newcommand{\PlanAntHorizonsOne}{23}
\newcommand{\PlanAntHorizonsTwo}{4, 9}
\newcommand{\PlanAntHorizonsThree}{4, 2, 5}
\newcommand{\PlanAntHorizonsFour}{3, 2, 2, 3}
\newcommand{\PlanPushBudget}{90}
\newcommand{\PlanPushFlatLR}{0.05}
\newcommand{\PlanPushHierLR}{0.1, 0.075}
\newcommand{\PlanPushSamplesOne}{800}
\newcommand{\PlanPushSamplesTwo}{50, 400}
\newcommand{\PlanPushSamplesThree}{25, 50, 200}
\newcommand{\PlanPushSamplesFour}{200, 800, 800, 800}
\newcommand{\PlanPushHorizonsOne}{15}
\newcommand{\PlanPushHorizonsTwo}{2, 8}
\newcommand{\PlanPushHorizonsThree}{2, 2, 4}
\newcommand{\PlanPushHorizonsFour}{2, 2, 2, 2}
\newcommand{\PlanCubeBudget}{30}
\newcommand{\PlanCubeFlatLR}{0.05}
\newcommand{\PlanCubeHierLR}{0.1, 0.075}
\newcommand{\PlanCubeSamplesOne}{800}
\newcommand{\PlanCubeSamplesTwo}{400, 100}
\newcommand{\PlanCubeSamplesThree}{25, 800, 800}
\newcommand{\PlanCubeSamplesFour}{200, 50, 800, 800}
\newcommand{\PlanCubeHorizonsOne}{4}
\newcommand{\PlanCubeHorizonsTwo}{2, 2}
\newcommand{\PlanCubeHorizonsThree}{2, 2, 1}
\newcommand{\PlanCubeHorizonsFour}{2, 2, 1, 1}

\title{H-JEPA: End-to-End Learning of Hierarchical \\ World Models for Visual Planning}

\author{%
\textbf{Wancong Zhang}\textsuperscript{1,2,}\thanks{Equal contribution; order determined by coin toss.}\qquad
\textbf{Basile Terver}\textsuperscript{2,3,}\footnotemark[1]\qquad
\textbf{Michael Rabbat}\textsuperscript{2} \\
\textbf{Yann LeCun}\textsuperscript{1,2,}\thanks{Equal advising.}\qquad
\textbf{Randall Balestriero}\textsuperscript{2,4,}\footnotemark[2] \\
\textsuperscript{1}NYU\quad
\textsuperscript{2}Advanced Machine Intelligence\quad
\textsuperscript{3}INRIA Paris\quad
\textsuperscript{4}Brown University
}
\iclrfinalcopy

\begin{document}
\maketitle
\fancyhead{}
\renewcommand{\headrulewidth}{0pt}
\vspace{-1em}
\centerline{\small
  \href{https://h-jepa.com/}{\faGlobe\enspace Website}\qquad
  \href{https://github.com/kevinghst/H-JEPA}{\faGithub\enspace Code}}
\vspace{1.25em}

\begin{abstract}
Long-horizon planning with latent world models requires reasoning across
timescales and levels of abstraction. Existing task-agnostic JEPA world models
predict and plan in a single latent space, often at a single timescale.
We introduce \textbf{H-JEPA}, an end-to-end recipe for training a
hierarchy of action-conditioned JEPAs in which each level predicts farther ahead in its
own learned latent space. Planning proceeds top-down: the top level
optimizes progress toward the goal, and each level's predictions become
subgoals for the planner below it. When factors in the data evolve at separated
timescales, higher levels discard fast, unpredictable detail and retain slower
task-relevant state. Across four simulated navigation and manipulation
environments, hierarchical planning improves over a flat JEPA;
on Visual AntMaze, a three-level hierarchy raises success from 18\% to 73\% using less planner compute. Ablations attribute these gains to both temporal decomposition
and higher-level goal representations. With inverse-dynamics supervision, the
approach extends to diverse real-robot videos from DROID, where hierarchy
improves offline planning fidelity at lower planner compute.

\end{abstract}

\begin{figure}[H]
\centering
\includegraphics[width=\linewidth,trim={0 14bp 0 0},clip]{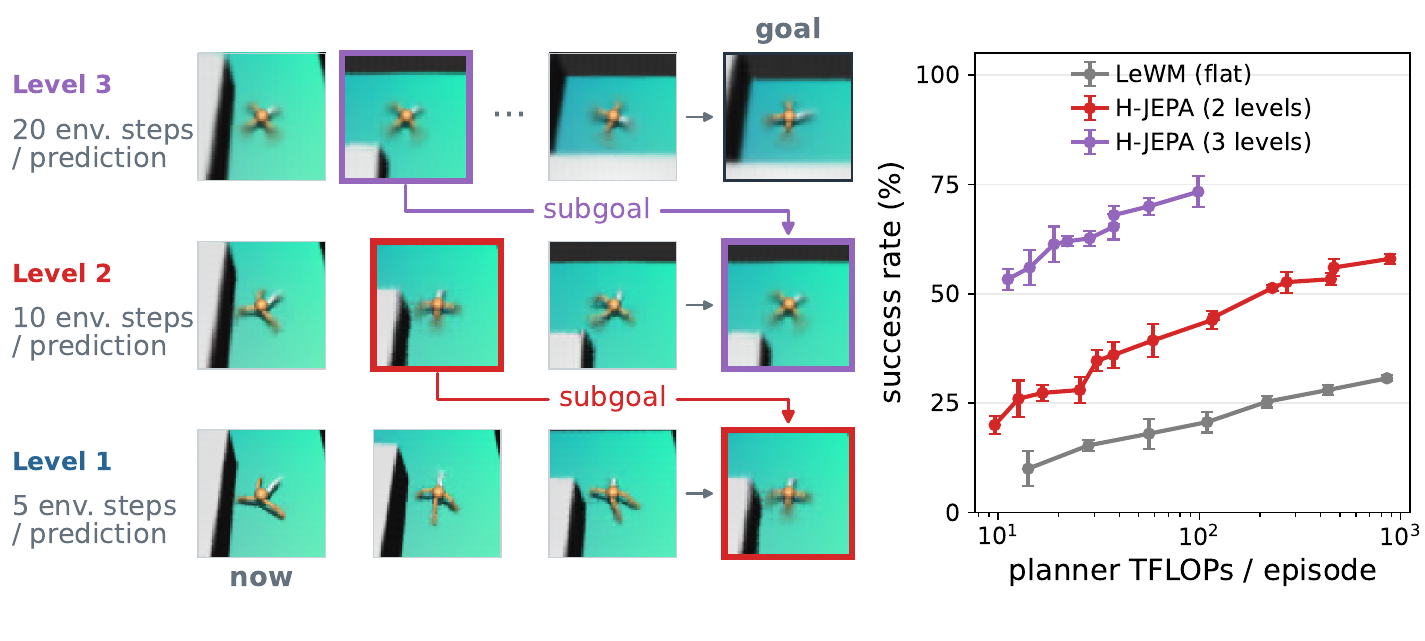}
\caption{\textbf{Hierarchical abstractions support planning at multiple timescales.}
\emph{Left:} visualization of hierarchical planning in Visual AntMaze. Level 3 plans toward
the final goal in the most abstract latent space. Its first predicted state becomes
a subgoal for the level-2 planner; this decomposition recurses to the level-1 planner, which
outputs raw actions.
The first and last columns decode each level's start and goal latents, respectively;
the middle columns decode its planned predictions. Higher-level latents retain
the ant's position and maze layout while losing leg-pose detail. See \S\ref{sec:method-planning} for the
formal planning procedure.
\emph{Right:} planning success versus planner compute on AntMaze. Curves show
Pareto fronts for flat LeWM~\citep{maes2026lewm} and two-/three-level H-JEPA; error bars are standard error (SE) over three paired
model and planner seeds (\S\ref{sec:planning-performance}).}
\label{fig:hierarchy-teaser}
\end{figure}

\section{Introduction}

World models learn the dynamics of an environment from experience so that an agent
can predict, understand, and plan~\citep{haschmidhuber2018worldmodels,hafner2019planet,hafner2020dreamer}.
Joint-Embedding Predictive Architectures (JEPAs) learn such models by predicting
masked \emph{latent} states rather than
pixels~\citep{lecun2022path,assran2023ijepa,bardes2024vjepa}, avoiding both reward
and pixel
reconstruction~\citep{balestriero2024reconstruction,nilaksh2026reconstructionsemantics}.
Recent action-conditioned, task-agnostic JEPA world models can plan zero-shot in
learned latent spaces~\citep{zhou2025dinowm,assran2025vjepa2,sobal2025pldm,maes2026lewm,terver2026whatdrives};
among these models, PLDM~\citep{sobal2025pldm} and LeWM~\citep{maes2026lewm}
learn the encoder and predictive dynamics jointly, end-to-end from pixels.
These models, however, predict and plan within a single latent space, often at a
single timescale.
This has two limitations. First, at a single timescale, long-horizon prediction
rolls out many fine-grained steps: prediction error compounds and the action search
space grows. Second, a
single shared latent space must support both low-level dynamics and goal
matching. A latent that carries every fast-varying detail may model low-level dynamics
well, but gives a poor goal-matching cost when the goal is more abstract
than the state (a location to reach rather than a pose to match;
\S\ref{sec:cost-surface}).

A hierarchical world model divides this work. Each higher level predicts over a
longer horizon and keeps in its latent space what remains predictable. Higher levels
operate on slow features over long strides and serve long-range prediction and
high-level planning; lower levels operate on fast features over short strides and
serve short-range prediction and low-level control. The brain has such hierarchical structure: cortical areas
form a hierarchy of intrinsic timescales, with slower-varying representations
higher up~\citep{hasson2008temporal,kiebel2008hierarchy,murray2014timescales,baldassano2017event}.
Hierarchical world models have been studied in video prediction and reward-driven
control (\S\ref{app:related-work}, Table~\ref{tab:hier-control}). Hierarchical latent
video models reconstruct pixels and are not used for
control~\citep{saxena2021clockwork,kim2019vta,mounir2023streamer}, while hierarchical
model-based reinforcement learning learns task-specific policies from
reward~\citep{hafner2022director,hansen2025puppeteer,gumbsch2024thick}.
Closest to our setting, HWM~\citep{zhang2026hierarchical} is a task-agnostic JEPA
world model that plans hierarchically, but its predictors over different horizons
share one latent space. It decomposes the horizon but cannot match objectives in
more abstract spaces.

We introduce \textbf{H-JEPA}, a hierarchical JEPA architecture in which each level
predicts the future within its own latent space, at a coarser temporal stride and a
higher degree of abstraction than the level below, without reward or reconstruction
objectives. Fig.~\ref{fig:hierarchy-teaser} shows the learned hierarchy and its
tradeoff between planner compute and success. We make four contributions:

\begin{enumerate}
  \item We introduce an end-to-end training method for hierarchical JEPA world models (\S\ref{sec:method}).

  \item We show that when factors in the data evolve at separated
  timescales, higher levels discard fast detail they cannot predict
  over their horizons and retain slower, predictable state
  (\S\ref{sec:learning-abstractions}).

  \item We show that H-JEPA's hierarchical planning outperforms
  single-level planning at a fraction of the test-time compute
  (\S\ref{sec:planning-performance}). We identify two complementary mechanisms:
  hierarchical latent spaces let higher-level planners score progress at
  the goal's own level of abstraction, while temporal decomposition breaks
  long-horizon tasks into easier subproblems
  (\S\ref{sec:cost-surface}).

  \item With an inverse-dynamics term, the method extends to DROID, a real-robot
  manipulation corpus whose scene, lighting and objects change every
  episode (\S\ref{sec:droid}).
\end{enumerate}

\section{Method}
\label{sec:method}

\subsection{Hierarchical predictive learning}
\label{sec:objective}

\begin{figure}[t]
\centering
\resizebox{\textwidth}{!}{%
\begin{tikzpicture}[
    x=1cm, y=1cm,
    font=\fontsize{15.20}{18.24}\selectfont,
    >={Stealth[length=4.5pt,width=3.5pt]},
    panel/.style={draw=hjepaborder, line width=0.6pt, rounded corners=5pt},
    objective/.style={draw=hjepacost!65, fill=hjepacostbg,
        text=hjepacost, rounded corners=3pt, line width=0.6pt,
        inner xsep=6pt, inner ysep=4pt},
    costflow/.style={flow,draw=hjepacost},
    flow/.style={->, draw=black!65, line width=0.65pt},
    encoder/.style={and gate US,logic gate inputs=nn,
        rotate=90,transform shape,xscale=0.90,yscale=1.17,
        draw=black!65,fill=black!3,line width=0.65pt,
        minimum width=1.08cm,minimum height=0.72cm,
        inner xsep=5pt,inner ysep=2pt},
    predictor/.style={and gate US,logic gate inputs=nn,
        draw=black!65,fill=black!3,line width=0.65pt,
        minimum width=1.13cm,minimum height=0.78cm,
        inner xsep=1pt,inner ysep=0pt},
    state/.style={text=hjepastate, inner sep=3pt},
    action/.style={text=hjepaaction, inner sep=3pt},
    heading/.style={text=hjepaheading,font=\sffamily\bfseries\fontsize{15.20}{18.24}\selectfont, anchor=west},
    note/.style={font=\sffamily\fontsize{13.85}{16.62}\selectfont, text=black!65, inner sep=2pt},
    subheading/.style={text=hjepaheading,font=\sffamily\fontsize{13.85}{16.62}\selectfont, text=black!80},
    level/.style={font=\sffamily\fontsize{12.50}{15.00}\selectfont, text=black!60, anchor=east},
]
\definecolor{hjepastate}{HTML}{286494}
\definecolor{hjepaaction}{HTML}{B76622}
\definecolor{hjepaborder}{HTML}{B5C6D2}
\definecolor{hjepaheader}{HTML}{EAF2F6}
\definecolor{hjepaheading}{HTML}{29495D}
\definecolor{hjepacost}{HTML}{77529C}
\definecolor{hjepacostbg}{HTML}{F5F1FA}

\def\encoderGap{0.38cm}
\path[use as bounding box] (-0.65,0.35) rectangle (21.55,9.10);
\begin{scope}[on background layer]
    \begin{scope}
        \clip[rounded corners=5pt] (-0.65,0.35) rectangle (11.25,9.10);
        \fill[hjepaheader] (-0.65,8.20) rectangle (11.25,9.10);
    \end{scope}
    \begin{scope}
        \clip[rounded corners=5pt] (11.70,0.35) rectangle (21.55,9.10);
        \fill[hjepaheader] (11.70,8.20) rectangle (21.55,9.10);
    \end{scope}
    \draw[panel] (-0.65,0.35) rectangle (11.25,9.10);
    \draw[panel] (11.70,0.35) rectangle (21.55,9.10);
    \draw[hjepaborder,line width=0.5pt] (-0.65,8.20) -- (11.25,8.20);
    \draw[hjepaborder,line width=0.5pt] (11.70,8.20) -- (21.55,8.20);
\end{scope}
\node[heading] (hierarchyheading) at (-0.25,8.65)
    {(a) Build hierarchy ($w_2=1$, $s_2=2$)};
\node[subheading] (stateheading) at (3.75,7.27) {State encoding};
\node[subheading] at (9.05,7.27) {Action encoding};
\node[heading] at (12.10,8.65) {(b) Train each level};
\node[note,anchor=west] at (12.10,7.47) {For each level $\ell=1,2$};

\node[predictor,anchor=north,yshift=-0.50cm] (predictor) at (15.20,0 |- stateheading.north)
    {$F^{(\ell)}$};

\foreach \x/\t/\columnopacity in {1.55/0/1,2.95/1/0.45,4.55/2/1,6.15/3/0.45} {
    \begin{scope}[opacity=\columnopacity]
    \node[state] (o\t) at (\x,0.90) {$o_{\t}$};
    \node[encoder,anchor=west] (e\t) at ([yshift=\encoderGap]o\t.north) {};
    \node at (e\t.center) {$E^{(1)}$};
    \node[state,anchor=south] (z\t) at ([yshift=\encoderGap]e\t.output) {$z^{(1)}_{\t}$};
    \draw[flow] (o\t.north) -- (e\t.west);
    \draw[flow] (e\t.output) -- (z\t.south);
    \end{scope}
}
\node[level] at (1.05,0 |- z0) {Level 1};
\node[encoder,anchor=west] (upper0) at ([yshift=\encoderGap+0.25cm]z0.north) {};
\node at (upper0.center) {$E^{(2)}$};
\node[state,anchor=south] (u0) at ([yshift=\encoderGap]upper0.output) {$z^{(2)}_0$};
\node[encoder,anchor=west] (upper1) at ([yshift=\encoderGap+0.25cm]z2.north) {};
\node at (upper1.center) {$E^{(2)}$};
\node[state,anchor=south] (u1) at ([yshift=\encoderGap]upper1.output) {$z^{(2)}_1$};
\node[level] at (1.05,0 |- u0) {Level 2};
\draw[flow] (z0.north) -- (upper0.west);
\draw[flow] (z2.north) -- (upper1.west);
\draw[flow] (upper0.output) -- (u0.south);
\draw[flow] (upper1.output) -- (u1.south);

\foreach \x/\t in {8.05/0,10.05/1} {
    \node[action] (rawa\t) at (\x,0.90) {$a_{\t}$};
    \node[encoder,anchor=west] (ae\t) at ([yshift=\encoderGap]rawa\t.north) {};
    \node at (ae\t.center) {$A^{(1)}$};
    \node[action,anchor=south] (az\t) at ([yshift=\encoderGap]ae\t.output) {$a^{(1)}_{\t}$};
    \draw[flow] (rawa\t.north) -- (ae\t.west);
    \draw[flow] (ae\t.output) -- (az\t.south);
}
\node[encoder] (actionupper) at (9.05,0 |- upper0.center) {};
\node at (actionupper.center) {$A^{(2)}$};
\node[action,anchor=south] (au) at ([yshift=\encoderGap]actionupper.output) {$a^{(2)}_0$};
\draw[flow] (az0.north) -- (actionupper.input 1);
\draw[flow] (az1.north) -- (actionupper.input 2);
\draw[flow] (actionupper.output) -- (au.south);

\node[state] (context) at (12.95,0 |- predictor.center) {$z^{(\ell)}_t$};
\node[state] (prediction) at (17.65,0 |- predictor.center) {$\hat z^{(\ell)}_{t+1}$};
\node[objective] (loss) at (20.20,0 |- predictor.center) {$\mathcal{L}^{(\ell)}_{\mathrm{pred}}$};
\node[state,below=0.35cm of loss] (target) {$z^{(\ell)}_{t+1}$};
\node[action,anchor=base] (conditioning) at (predictor.center |- target.base)
    {$a^{(\ell)}_t$};
\draw[flow] (context.east) -- (predictor.west);
\draw[flow] (conditioning.north) -- (predictor.south);
\draw[flow] (predictor.output) -- (prediction.west);
\draw[costflow] (prediction) -- (loss);
\draw[costflow] (target) -- (loss);
\node[note,align=center] at ([yshift=-0.90cm]context.center) {Context};
\node[note] at ([yshift=-0.90cm]prediction.center) {Prediction};
\node[note,align=center,below=0.12cm of target] (targetnote) {Target};
\node[note,anchor=west,text=hjepacost] (objectiveheading) at (12.10,2.30) {Per-level objective};
\node[objective,anchor=south,text=black] at (16.65,0.75)
    {$\displaystyle \textcolor{hjepacost}{\mathcal{L}^{(\ell)}}
      = \textcolor{hjepacost}{\mathcal{L}^{(\ell)}_{\mathrm{pred}}}
      + \lambda_\ell\,\mathrm{SIGReg}\!\left(
          \textcolor{hjepastate}{Z^{(\ell)}}\right)$};

\coordinate (regularizationcenter) at (17.10,3.75);
\node[note,anchor=west,align=left,inner ysep=0pt]
    at (12.10,0 |- regularizationcenter)
    {Regularize encoded states\\[0.28cm]
     {\fontsize{15.20}{18.24}\selectfont\rmfamily\color{black}$\mathrm{SIGReg}\!\left(
        \textcolor{hjepastate}{Z^{(\ell)}}\right),
      \qquad \textcolor{hjepastate}{Z^{(\ell)}}
      = \{\textcolor{hjepastate}{z^{(\ell)}_t}\}_{\text{batch, time}}$}};
\end{tikzpicture}%
}
\caption{\textbf{H-JEPA training across temporal scales.}
(a) With $w_2{=}1$ and $s_2{=}2$, level~2 encodes every other level-1 state
pointwise and aggregates the intervening action embeddings.
Faded columns are retained at level~1 but skipped at level~2.
(b) Each level predicts the next latent state conditioned on preceding latent
states and actions,
with SIGReg regularizing the states against collapse.
All levels are trained end-to-end,
with gradients from upper-level objectives propagating through lower-level encoders.}
\label{fig:hjepa_training}
\end{figure}
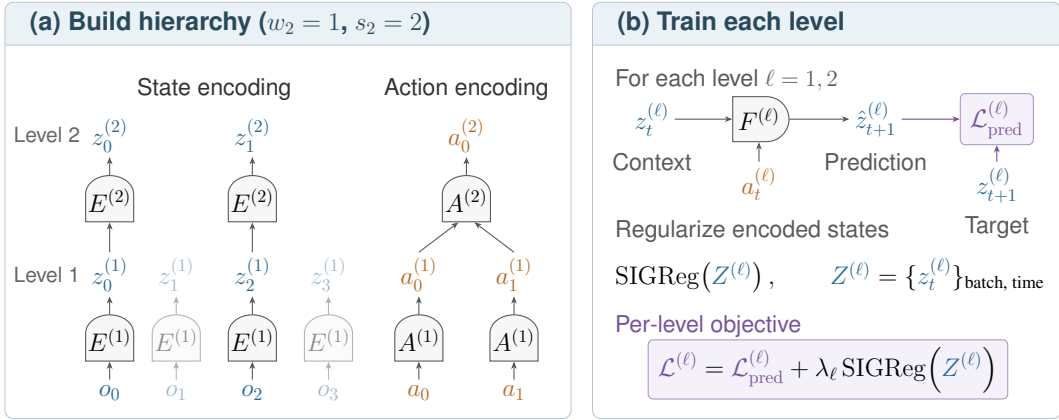


H-JEPA learns a hierarchy of latent predictive models from trajectories of
observations and actions.  The hierarchy is temporal: level 1 operates on the
finest observation stream, while each higher level consumes the latent states
produced by the level below over a coarser stride.  This gives a stack of JEPA
world models in which each level has its own observation encoder, action encoder, and
latent predictor. All levels predict future
latent states from past latent states and actions.

Let $o_t$ denote an observation and $a_t$ the action block for the transition
from $o_t$ to $o_{t+1}$.  The first level encodes the observation, optionally with
the proprioceptive state, into a latent state $z^{(1)}_t = E^{(1)}(o_t)$, and the
action block into $a^{(1)}_t = A^{(1)}(a_t)$.  Higher levels are built
compositionally: both state and action encoders pool temporal windows of latents from the level below.

Each level $\ell>1$ has two temporal hyperparameters.  The \emph{stride}
$s_\ell$ is the subsampling factor: upper-level time $t$ maps to lower-level time
$t\cdot s_\ell$, so consecutive upper-level states are $s_\ell$ lower-level steps
apart.  The \emph{window size} $w_\ell$ is the number of lower-level steps
each upper-level state summarizes. The two are independent: $s_\ell$ sets how
far each window advances and $w_\ell$ how far it spans.

With $a{:}b$ denoting the $b{-}a$ indices $a,\ldots,b{-}1$, the
level-$\ell$ state and action embeddings are
\begin{equation}
\label{eq:level-state}
z^{(\ell)}_t = E^{(\ell)}\!\left(z^{(\ell-1)}_{t\cdot s_\ell\,:\,t\cdot s_\ell+w_\ell}\right),
\qquad
a^{(\ell)}_t = A^{(\ell)}\!\left(a^{(\ell-1)}_{t\cdot s_\ell+w_\ell-1\,:\,(t+1)\cdot s_\ell+w_\ell-1}\right).
\end{equation}
The state encoder $E^{(\ell)}$ pools the $w_\ell$-step lower-level window into one
abstract state.  The action encoder $A^{(\ell)}$ instead aggregates
$s_\ell$ lower-level action embeddings, independent
of $w_\ell$, so $a^{(\ell)}_t$ covers the full transition from the state
anchoring $z^{(\ell)}_t$ toward the next upper-level state $z^{(\ell)}_{t+1}$.
All experiments use $w_\ell{=}1$, so upper-level state encoders are pointwise
(Fig.~\ref{fig:hjepa_training}).

Each level is trained as a JEPA in its own latent
space. Given a context of $c_\ell$ latent states
$(z^{(\ell)}_{t-c_\ell+1}, \ldots, z^{(\ell)}_t)$ and the associated action embeddings,
a predictor $F^{(\ell)}$ predicts future latent states.  One-step
training uses the teacher-forced latent prediction loss
\begin{equation}
\label{eq:pred-loss}
\mathcal{L}^{(\ell)}_{\mathrm{pred}}
= \frac{1}{c_\ell}
\sum_{\tau=1}^{c_\ell}
\left\|
F^{(\ell)}(z^{(\ell)}_{<\tau}, a^{(\ell)}_{<\tau}) - z^{(\ell)}_{\tau+1}
\right\|_2^2 .
\end{equation}

The per-level loss also includes SIGReg~\citep{balestriero2025lejepa}, a sketched
normality regularizer that prevents collapse by encouraging an isotropic
Gaussian embedding distribution (\S\ref{app:epps-pulley}). Let $Z^{(\ell)} \in \mathbb{R}^{n \times D}$
collect $n$ level-$\ell$ embeddings of width $D$, flattened over batch and time. The overall level-$\ell$ objective is
\begin{equation}
\label{eq:level-obj}
\mathcal{L}^{(\ell)}
= \mathcal{L}^{(\ell)}_{\mathrm{pred}}
+ \lambda_\ell \mathrm{SIGReg}\!\left(Z^{(\ell)}\right).
\end{equation}

\subsection{Hierarchical planning over learned abstractions}
\label{sec:method-planning}

\begin{figure}[t]
\centering
\resizebox{\textwidth}{!}{%
\begin{tikzpicture}[
    x=1cm,y=1cm,font=\fontsize{12.20}{14.64}\selectfont,
    >={Stealth[length=3.62pt,width=2.81pt]},
    panel/.style={draw=hjepaborder, line width=0.482pt, rounded corners=4.02pt},
    objective/.style={draw=hjepacost!65, fill=hjepacostbg,
        text=hjepacost, rounded corners=2.41pt, line width=0.482pt,
        inner xsep=4.82pt, inner ysep=3.22pt},
    costflow/.style={flow,draw=hjepacost},
    flow/.style={->,draw=black!65,line width=0.523pt},
    encoder/.style={and gate US,logic gate inputs=nn,
        rotate=90,transform shape,xscale=0.90,yscale=1.17,
        draw=black!65,fill=black!3,line width=0.523pt,
        minimum width=0.868cm,minimum height=0.579cm,
        inner xsep=4.02pt,inner ysep=1.61pt},
    predictor/.style={and gate US,logic gate inputs=nn,
        draw=black!65,fill=black!3,line width=0.523pt,
        minimum width=0.942cm,minimum height=0.650cm,
        inner xsep=0.804pt,inner ysep=0pt},
    state/.style={text=hjepastate,inner sep=3pt},
    action/.style={text=hjepaaction,inner sep=3pt},
    heading/.style={text=hjepaheading,font=\sffamily\bfseries\fontsize{12.20}{14.64}\selectfont,anchor=west},
    note/.style={font=\sffamily\fontsize{11.10}{13.32}\selectfont,text=black!65,inner sep=2pt},
    level/.style={font=\sffamily\fontsize{10.00}{12.00}\selectfont,text=black!60,anchor=east},
    omitted/.style={text=black!55,inner sep=3pt},
]
\definecolor{hjepastate}{HTML}{286494}
\definecolor{hjepaaction}{HTML}{B76622}
\definecolor{hjepaborder}{HTML}{B5C6D2}
\definecolor{hjepaheader}{HTML}{EAF2F6}
\definecolor{hjepaheading}{HTML}{29495D}
\definecolor{hjepacost}{HTML}{77529C}
\definecolor{hjepacostbg}{HTML}{F5F1FA}

\path[use as bounding box] (-0.75,-0.657) rectangle (17.10,7.064);
\begin{scope}[on background layer]
    \begin{scope}
        \clip[rounded corners=4.02pt] (-0.75,-0.657) rectangle (17.10,7.064);
        \fill[hjepaheader] (-0.75,6.264) rectangle (17.10,7.064);
    \end{scope}
    \draw[panel] (-0.75,-0.657) rectangle (17.10,7.064);
    \draw[hjepaborder,line width=0.402pt] (-0.75,6.264) -- (17.10,6.264);
\end{scope}
\node[heading] (title) at (-0.40,6.664) {Hierarchical planning};
\node[note,anchor=base west,text=hjepaheading] at ([xshift=0.35cm]title.base east)
    {Two-level example: subgoals $K_2=1$, stride $s_2=2$, window $w_2=1$};

\def\lowerEncoderGap{0.42cm}
\node[state] (zinitialone) at (1.45,2.10) {$z^{(1)}_0$};
\node[encoder,anchor=output] (einitialone)
    at ([yshift=-\lowerEncoderGap]zinitialone.south) {};
\node[state,anchor=north] (o0)
    at ([yshift=-\lowerEncoderGap]einitialone.west) {$o_0$};
\node at (einitialone.center) {$E^{(1)}$};
\coordinate (einitialonein) at ($(einitialone.input 1)!0.5!(einitialone.input 2)$);
\coordinate (einitialoneout) at (einitialone.output);
\node[encoder] (einitialtwo) at (1.45,3.85) {};
\node at (einitialtwo.center) {$E^{(2)}$};
\coordinate (einitialtwoin) at ($(einitialtwo.input 1)!0.5!(einitialtwo.input 2)$);
\coordinate (einitialtwoout) at (einitialtwo.output);
\node[state] (zinitialtwo) at (1.45,5.65) {$z^{(2)}_0$};
\draw[flow] (o0.north) -- (einitialone.west);
\draw[flow] (einitialone.output) -- (zinitialone.south);
\draw[flow] (zinitialone) -- (einitialtwoin);
\draw[flow] (einitialtwoout) -- (zinitialtwo);
\node[note,anchor=base east] (initiallabel) at ([xshift=-0.12cm]o0.base west) {Initial};

\node[state] (gone) at (16.25,2.10) {$g^{(1)}$};
\node[encoder,anchor=output] (egoalone)
    at ([yshift=-\lowerEncoderGap]gone.south) {};
\node[state,anchor=north] (og)
    at ([yshift=-\lowerEncoderGap-0.058cm]egoalone.west) {$o_g$};
\node at (egoalone.center) {$E^{(1)}$};
\coordinate (egoalonein) at ($(egoalone.input 1)!0.5!(egoalone.input 2)$);
\coordinate (egoaloneout) at (egoalone.output);
\node[encoder] (egoaltwo) at (16.25,3.85) {};
\node at (egoaltwo.center) {$E^{(2)}$};
\coordinate (egoaltwoin) at ($(egoaltwo.input 1)!0.5!(egoaltwo.input 2)$);
\coordinate (egoaltwoout) at (egoaltwo.output);
\node[state] (gtwo) at (16.25,5.65) {$g^{(2)}$};
\draw[flow] (og.north) -- (egoalone.west);
\draw[flow] (egoalone.output) -- (gone.south);
\draw[flow] (gone) -- (egoaltwoin);
\draw[flow] (egoaltwoout) -- (gtwo);
\node[note,anchor=base east] (goallabel) at ([xshift=-0.12cm]og.base west) {Goal};
\node[level] at (0.60,5.65) {Level 2};
\node[level] at (0.60,2.10) {Level 1};

\node[predictor] (ftwofirst) at (4.20,5.65) {};
\node at (ftwofirst.center) {$F^{(2)}$};
\node[state] (subgoal) at (7.00,5.65) {$\hat z^{(2)}_1$};
\node[omitted] (skipcoarse) at (8.45,5.65) {$\cdots$};
\node[predictor] (ftwolast) at (9.90,5.65) {};
\node at (ftwolast.center) {$F^{(2)}$};
\node[state] (terminal) at (11.95,5.65) {$\hat z^{(2)}_{H_2}$};
\node[objective] (costtwo) at (14.00,5.65) {$C^{(2)}$};
\draw[flow] (zinitialtwo) -- (ftwofirst.west);
\draw[flow] (ftwofirst.output) -- (subgoal);
\draw[flow] (subgoal) -- (skipcoarse);
\draw[flow] (skipcoarse) -- (ftwolast.west);
\draw[flow] (ftwolast.output) -- (terminal);
\draw[costflow] (terminal) -- (costtwo);
\draw[costflow] (gtwo) -- (costtwo);
\node[action,below=0.38cm of ftwofirst] (atwofirst) {$a^{(2)}_0$};
\node[action,below=0.38cm of ftwolast] (atwolast) {$a^{(2)}_{H_2-1}$};
\draw[flow] (atwofirst.north) -- (ftwofirst.south);
\draw[flow] (atwolast.north) -- (ftwolast.south);

\node[predictor] (fonefirst) at (2.85,2.10) {};
\node at (fonefirst.center) {$F^{(1)}$};
\node[state] (fineintermediate) at (4.225,2.10) {$\hat z^{(1)}_1$};
\node[predictor] (fonelast) at (5.60,2.10) {};
\node at (fonelast.center) {$F^{(1)}$};
\node[state] (fineterminal) at (7.00,2.10) {$\hat z^{(1)}_2$};
\draw[flow] (zinitialone) -- (fonefirst.west);
\draw[flow] (fonefirst.output) -- (fineintermediate);
\draw[flow] (fineintermediate) -- (fonelast.west);
\draw[flow] (fonelast.output) -- (fineterminal);
\foreach \suffix/\t in {first/0,last/1} {
    \node[action,below=0.38cm of fone\suffix] (aone\suffix) {$a^{(1)}_{\t}$};
    \draw[flow] (aone\suffix.north) -- (fone\suffix.south);
}

\node[encoder] (liftencoder) at (7.00,3.25) {};
\node at (liftencoder.center) {$E^{(2)}$};
\coordinate (liftencoderin) at ($(liftencoder.input 1)!0.5!(liftencoder.input 2)$);
\coordinate (liftencoderout) at (liftencoder.output);
\node[objective] (costone) at (7.00,4.45) {$C^{(1)}$};
\draw[costflow] (fineterminal) -- (liftencoderin);
\draw[costflow] (liftencoderout) -- (costone);
\draw[costflow,line width=0.804pt] (subgoal) -- (costone)
    node[midway,left=3pt,note,font=\sffamily\fontsize{11.10}{13.32}\selectfont\bfseries,text=hjepacost] {Subgoal};

\def\descriptionGap{0.30cm}
\node[anchor=south west,inner ysep=0pt] (finecosttext)
    at (8.55,-0.251)
    {$\textcolor{hjepacost}{C^{(1)}}=\bigl\|E^{(2)}(\textcolor{hjepastate}{\hat z^{(1)}_2})
        -\textcolor{hjepastate}{\hat z^{(2)}_1}\bigr\|_2^2$};
\node[note,anchor=south west,align=left,inner ysep=0pt] (finedescription)
    at ([yshift=0.12cm]finecosttext.north west)
    {Level 1: optimize $\textcolor{hjepaaction}{a^{(1)}_0,\,a^{(1)}_1}$\\first subgoal term shown};
\node[anchor=south west,inner ysep=0pt] (coarsecosttext)
    at ([yshift=\descriptionGap]finedescription.north west)
    {$\textcolor{hjepacost}{C^{(2)}}=\bigl\|\textcolor{hjepastate}{\hat z^{(2)}_{H_2}}
        -\textcolor{hjepastate}{g^{(2)}}\bigr\|_2^2$};
\node[note,anchor=south west,align=left,inner ysep=0pt]
    at ([yshift=0.12cm]coarsecosttext.north west)
    {Level 2: optimize $\textcolor{hjepaaction}{a^{(2)}_{0:H_2-1}}$\\to minimize goal cost};
\end{tikzpicture}%
}
\caption{\textbf{Planning runs top-down through the hierarchy, and the lower level
is steered by subgoals the upper level predicted.}  Two-level instance ($L{=}2$) of
\S\ref{sec:method-planning}.  Current and goal observations $o_0$ and $o_g$
are encoded bottom-up into initial states $z^{(1)}_0, z^{(2)}_0$ and goal states
$g^{(1)}, g^{(2)}$.  Level~2 optimizes macro-actions by unrolling $F^{(2)}$ and matching
its final prediction to $g^{(2)}$.  Its predicted states $\hat z^{(2)}_i$ become subgoals
for level~1, which optimizes primitive actions and lifts its predictions
through $E^{(2)}$ for comparison in level-2 latent space
(Eq.~\eqref{eq:lower-plan}).}
\label{fig:hierarchical_planning}
\end{figure}
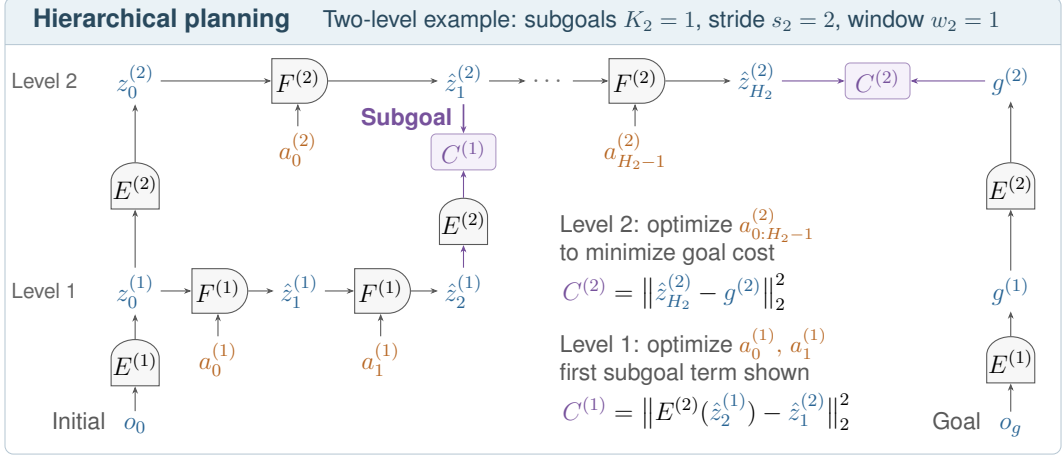


H-JEPA plans top-down through the learned hierarchy
(Fig.~\ref{fig:hierarchical_planning}). The current and goal observations
are encoded at every level $\ell$, yielding initial latent states
$\{z^{(\ell)}_0\}_{\ell=1}^L$ and goal latent states
$\{g^{(\ell)}\}_{\ell=1}^L$, where $L$ is the number of levels.
The top level proposes a coarse plan to the goal. Each lower level plans
toward subgoals supplied by the predicted trajectory of the level above,
refining the plan into finer-scale transitions until level 1
produces primitive actions.
The top-level planner optimizes macro-actions to minimize the distance
between its final predicted state and the goal:
\begin{equation}
\label{eq:top-plan}
a^{(L),*}_{0:H_L-1}
= \arg\min_{a^{(L)}_{0:H_L-1}}
\left\|\hat{z}^{(L)}_{H_L}-g^{(L)}\right\|_2^2.
\end{equation}
At each level $\ell$, the predicted states
$(\hat{z}^{(\ell)}_1,\ldots,\hat{z}^{(\ell)}_{H_\ell})
= \mathrm{Roll}^{(\ell)}(z^{(\ell)}_0,a^{(\ell)}_{0:H_\ell-1})$
are obtained by autoregressively applying predictor $F^{(\ell)}$ from
$z^{(\ell)}_0$ under the candidate action sequence, with horizon $H_\ell$
measured in level-$\ell$ steps. A superscript $*$ marks the optimized
actions and their predicted states.

For each lower level $\ell<L$, the optimized rollout from the level above,
$(\hat{z}^{(\ell+1),*}_1,\ldots,\hat{z}^{(\ell+1),*}_{H_{\ell+1}})$,
supplies subgoals. The lower level optimizes its actions to match these
subgoals with its own predicted states, encoded into the upper-level latent
space by $E^{(\ell+1)}$.

We first choose the number of subgoals to match,
$1\le K_{\ell+1}\le H_{\ell+1}$, then set the lower-level horizon to cover
the corresponding encoder windows. With upper-level stride $s_{\ell+1}$
and window size $w_{\ell+1}$, this requires
$H_\ell=K_{\ell+1}s_{\ell+1}+w_{\ell+1}-1$. The encoded predictions are
\begin{equation}
\label{eq:lift}
\tilde{z}^{(\ell+1)}_i
= E^{(\ell+1)}\!\left(
\hat{z}^{(\ell)}_{i s_{\ell+1}\,:\,i s_{\ell+1}+w_{\ell+1}}\right),
\qquad i=1,\ldots,K_{\ell+1}.
\end{equation}
Each $\tilde{z}^{(\ell+1)}_i$ is aligned with its upper-level subgoal in time and latent space.
The lower level solves
\begin{equation}
\label{eq:lower-plan}
a^{(\ell),*}_{0:H_\ell-1}
= \arg\min_{a^{(\ell)}_{0:H_\ell-1}}
\left[
\left\|\tilde{z}^{(\ell+1)}_{K_{\ell+1}}-\hat{z}^{(\ell+1),*}_{K_{\ell+1}}\right\|_2^2
+ \beta\sum_{i=1}^{K_{\ell+1}-1}
\left\|\tilde{z}^{(\ell+1)}_i-\hat{z}^{(\ell+1),*}_i\right\|_2^2
\right],
\end{equation}
where $\beta\ge0$ weights each intermediate subgoal relative to the terminal
target. Closed-loop planning uses $K_{\ell+1}=1$, matching only the first
predicted subgoal; open-loop planning uses $K_{\ell+1}=H_{\ell+1}$, with
$\beta=0$ matching only the terminal prediction and $\beta>0$ also matching
intermediate subgoals.

We optimize action sequences at every level by gradient descent. Level 1 returns
primitive actions for execution in the environment. In closed-loop control, we execute an action prefix, encode the
new observation, and repeat hierarchical planning until the evaluation budget
of environment steps is exhausted.
Execution schedules and solver settings are in \S\ref{app:planning-eval}.

\section{Learning Hierarchical Abstractions}
\label{sec:learning-abstractions}

\begin{figure}[t]
\centering
\includegraphics[width=\linewidth]{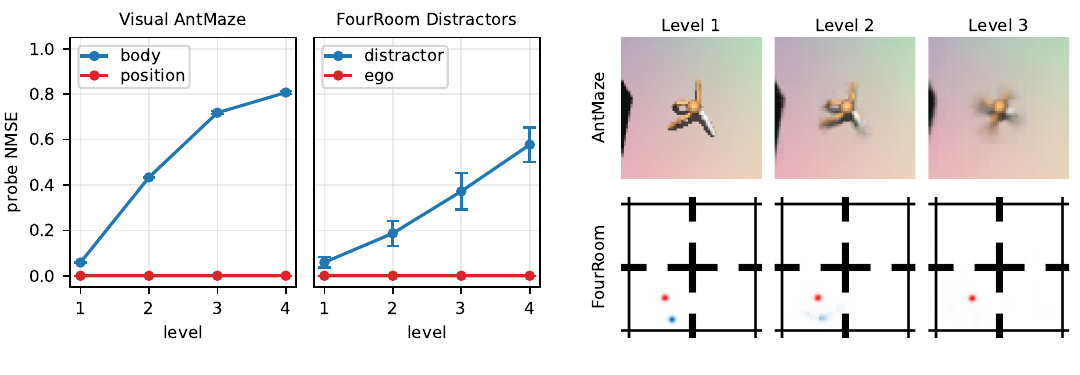}
\captionsetup{skip=-2pt}
\caption{\textbf{Selective abstraction across hierarchy levels.}
\emph{Left}: MLP-probe NMSE (mean $\pm$ SE over three training seeds) at
successive levels of four-level H-JEPAs on Visual AntMaze and
FourRoomDistractors. Body state or distractor position becomes less
recoverable at higher levels while agent position remains recoverable.
\emph{Right}: one observation decoded from
levels $1$--$3$; decoders serve only for visualization.
AntMaze's leg pose becomes less distinct; FourRoom's blue distractor fades
while the red agent remains visible.}
\label{fig:depth}
\end{figure}

\textbf{Environments.} We use FourRoomDistractors and Visual AntMaze for navigation,
and Push-T and OGBench Cube for manipulation.
FourRoomDistractors extends PLDM's two-room navigation
environment~\citep{sobal2025pldm} to four rooms. It pairs a controlled
agent with a distractor that moves independently and smoothly between occasional
random teleports, while Visual AntMaze requires a quadruped to navigate a maze.
Push-T involves pushing a T-shaped block, and OGBench
Cube involves picking up and placing a cube. In \S\ref{sec:droid}, we extend the
evaluation to DROID, a dataset of real-robot manipulation trajectories. Additional
details are in \S\ref{app:envs}.

\textbf{Setup.} We train H-JEPAs up to four levels end-to-end.
The level-1 JEPA is LeWM~\citep{maes2026lewm}: a ViT-Tiny encoder that
summarizes each image in a global representation through its \texttt{[CLS]} token,
and a causal transformer predictor. Trained on its own, this single-level model
is the flat baseline we compare against throughout.
Higher-level JEPAs use two-layer MLP encoders and causal transformer predictors,
with stride $s_\ell{=}2$ and window $w_\ell{=}1$, so a single prediction step at each
higher level spans twice as many environment timesteps as a step at the level below.
Within each environment, all levels share the same SIGReg coefficient. Additional
details in \S\ref{app:implementation}.

\textbf{Evaluation protocol.} We train two-layer MLP
probes on frozen representations and report variance-normalized MSE (NMSE),
averaged with equal weight over an entity's dimensions. A value near $1$ is the
error of predicting the marginal mean, while a value near $0$ indicates accurate
recovery. Quantitative results in this section are mean $\pm$ SE over three
training seeds; pixel decodings are qualitative examples from post-hoc decoders.
\S\ref{app:probing} details probe and decoder training.

\begin{wrapfigure}{r}{0.5\textwidth}
\centering
\includegraphics[width=\linewidth]{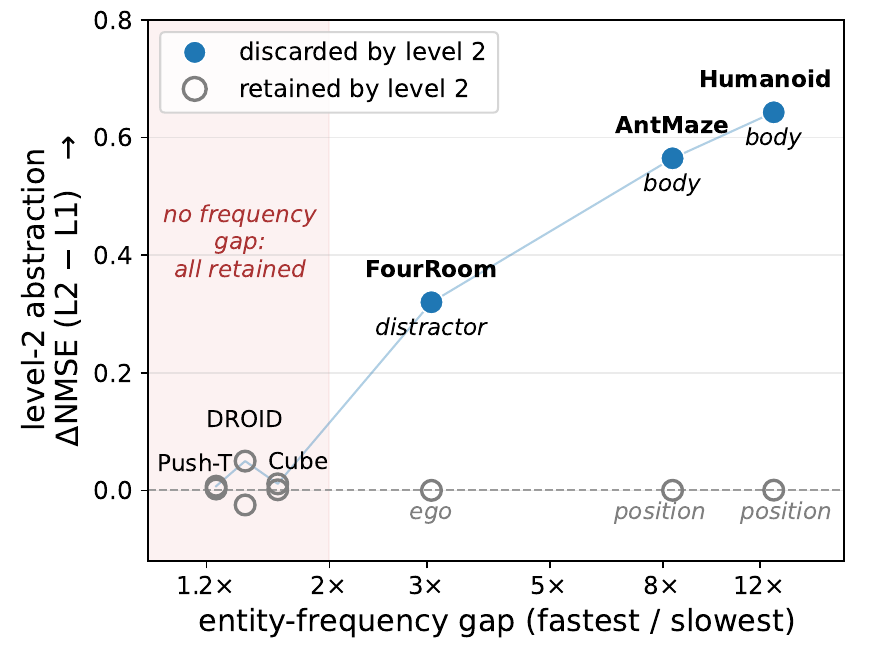}
\vspace{-2em}
\caption{\textbf{Larger temporal-frequency gaps correlate with stronger selective abstraction.}
The x-axis measures temporal separation in each training dataset:
the ratio of the characteristic frequencies (spectral centroids) of the fastest-
and slowest-varying state components, shown on a log scale.
The y-axis shows the change in per-entity
probing NMSE from level 1 to level 2;
positive values indicate poorer recoverability at level 2.
Across environments, larger frequency gaps coincide with greater increases
in fast-entity probing error (filled markers), while slow-entity error remains
near its level-1 value (open markers).
Environments with gaps below $2\times$ (shaded region) show little selective abstraction.}
\label{fig:selective}
\end{wrapfigure}

\textbf{Higher levels discard fast-varying features while retaining slow ones.} In
Fig.~\ref{fig:depth}, AntMaze's body state and FourRoom's distractor position become
less recoverable with depth, while the agent's position remains
accurately recoverable throughout both hierarchies. The discarded detail
is consistent with each level's prediction horizon: the FourRoom distractor's
motion is locally predictable between teleports, but longer horizons are more likely
to cross a random teleport, making its future position harder to predict.
In AntMaze, the ant's joint configurations oscillate many times faster than its
global position drifts (\S\ref{app:entity-traces}), so they are harder to
predict over long horizons. These trends are consistent with JEPA's joint
optimization of the encoder and predictor: prediction-error gradients flow through
the encoder, encouraging each level to retain features it can predict at its own
timescale and abstract away those it cannot. In these environments, H-JEPA thus
learns a hierarchy in which lower levels retain both fast- and slow-varying features,
while higher levels discard detail and preserve the
slow features.

Selective abstraction does not
emerge in every environment. We relate the change in per-entity probing error
from level 1 to level 2 to the training dataset's \emph{entity frequency gap}: the ratio of the
frequencies (spectral centroids) of the fastest- and slowest-varying state components, which define the fast and slow entities, respectively.
\S\ref{app:entity-traces} illustrates this separation in an AntMaze trajectory.
Across environments, larger gaps accompany greater loss of fast-entity
information at level 2, while slow-entity recoverability remains near its level-1
value (Fig.~\ref{fig:selective}). This separation appears in AntMaze, Humanoid,
and FourRoom; the manipulation datasets have smaller gaps and show little
selective abstraction. Gap measurements in \S\ref{app:selective-bars}.

\section{Planning with Hierarchical Abstractions}
\label{sec:long-horizon}

\begin{figure}[t]
\centering
\includegraphics[width=\linewidth]{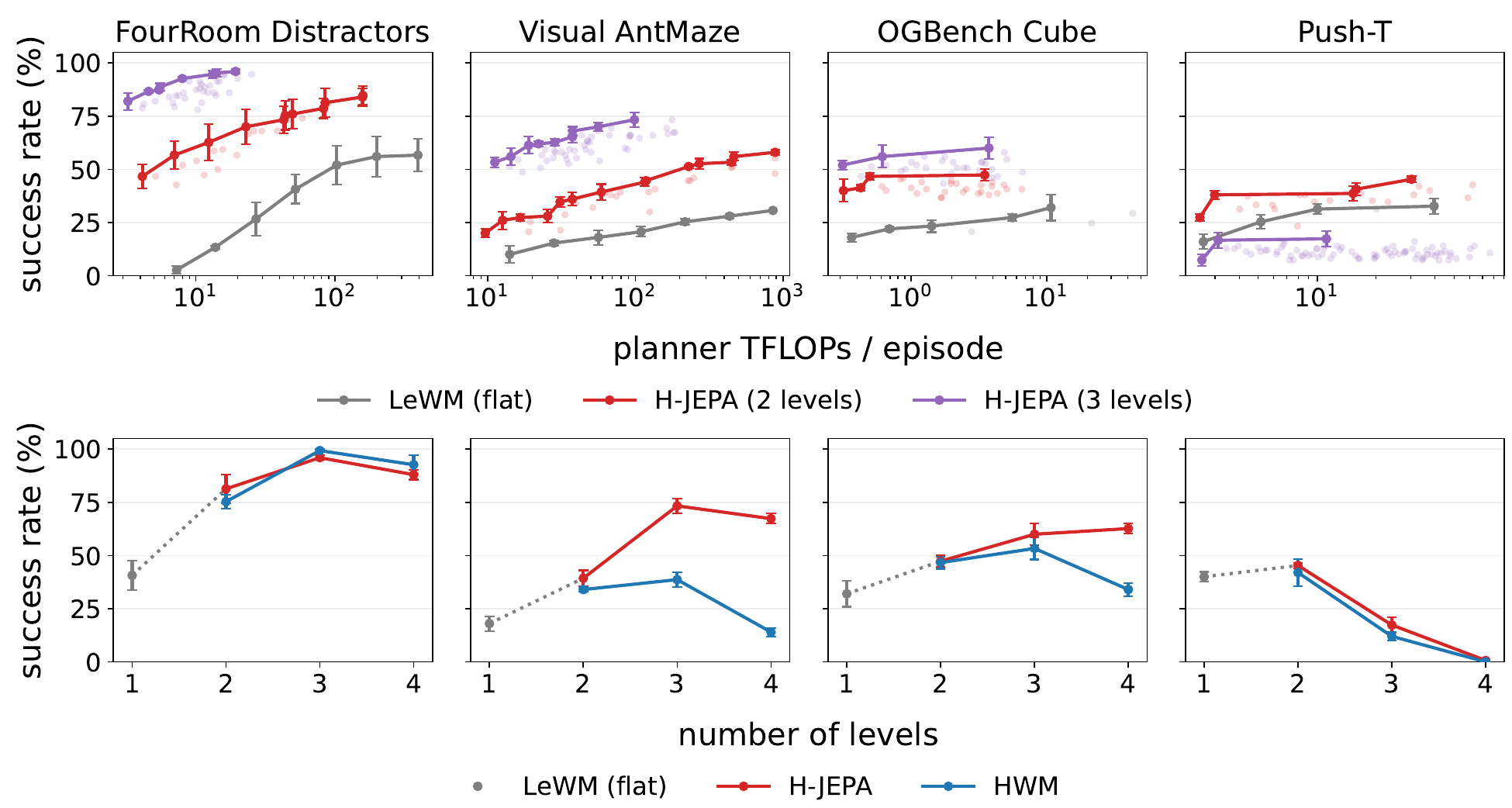}
\caption{\textbf{Gradient-based planning across compute budgets and hierarchy depths},
comparing LeWM~\citep{maes2026lewm}, H-JEPA, and the hierarchical world model baseline
HWM~\citep{zhang2026hierarchical}.
Error bars in both panels show one SE over three seeds.
\emph{Top:} success versus planner FLOPs per episode, sweeping action-sample counts
with other planner settings fixed. FLOPs include forward and backward passes. Bold curves are Pareto fronts.
\emph{Bottom:} success versus hierarchy depth on a separate set of $50$ tasks per
environment. Planner settings are selected under a $100$-TFLOP planning budget.
See \S\ref{app:planning-eval} for protocols and
Fig.~\ref{fig:compute-pareto-real} for the DROID results.}
\label{fig:compute-pareto}
\label{fig:depth-planning}
\end{figure}

Planning top-down through the learned hierarchy improves long-horizon control.
We first compare end-to-end H-JEPA with flat LeWM~\citep{maes2026lewm} across test-time compute
budgets and hierarchy depths, and with HWM~\citep{zhang2026hierarchical}, whose temporal
hierarchy is confined to one latent space. We then examine how learned abstractions
reshape the planning cost surface and how temporal decomposition breaks long-horizon
control into easier subproblems. Both evaluations use simulated environments
with a fixed camera and background. Finally, we scale the method to diverse scenes
with real teleoperated robot video from DROID. Task
construction and planner settings for all evaluations are in \S\ref{app:planning-eval}.

\subsection{Planning performance across compute budgets and depths}
\label{sec:planning-performance}

H-JEPA scales LeWM in depth:
\textbf{LeWM} is the single-level planner, and \textbf{H-JEPA} adds levels
above it, planning across distinct latent spaces with the top-down procedure of
\S\ref{sec:method-planning}. We test up to four levels, first examining
how planning success scales with test-time compute
(Fig.~\ref{fig:compute-pareto}, top). The sweep
varies the number of trajectories optimized by gradient descent at each
planning level, with other planner settings and tasks fixed within the sweep.
The depth comparisons and cost ablations use a separate evaluation task set
(\S\ref{app:planning-eval}). On FourRoom, AntMaze and
Cube, each additional level up to three shifts the Pareto frontier up and to the
left: higher success with less planner compute.
Push-T is the exception: two-level H-JEPA matches or exceeds
LeWM at every budget, but the three-level model performs poorly even with maximum compute. The likely cause is training
data. Clips must lie within one episode, and Push-T episodes
are short relative to the minimum number of clips required to train a three level model, so it sees only
$\ImplPushClipsThreePct\%$ of LeWM's transitions per epoch
(Table~\ref{tab:impl-data}); the other environments retain
$\ImplCubeClipsThreePct$--$\ImplRoomClipsThreePct\%$ at three levels.

\textbf{Distinct representations within a temporal hierarchy.}
We compare H-JEPA depth for depth with \textbf{HWM}~\citep{zhang2026hierarchical},
a hierarchical planner confined to one latent space, asking whether
distinct latent spaces add to the benefit of temporal
hierarchy alone. For a controlled comparison, HWM (blue curves in
Fig.~\ref{fig:depth-planning}, bottom) uses the same implementation, data and
training settings as H-JEPA at each depth, except that encoders above level~1
are identity maps. All levels therefore predict and plan in the same
latent space.
The improvement of H-JEPA over HWM is most pronounced on AntMaze,
where three-level H-JEPA reaches nearly twice
HWM's success.
On Cube, H-JEPA also scores higher at higher depths,
although more moderately relative to AntMaze.
On FourRoom and Push-T, both hierarchies perform similarly. Consistent with earlier findings,
both methods deteriorate on Push-T at three levels and beyond, likely because
short episodes limit usable training data: three- and four-level H-JEPA see only
$\ImplPushClipsThreePct\%$ and $\ImplPushClipsFourPct\%$ of LeWM's training
transitions per epoch, respectively (Table~\ref{tab:impl-data}).

HWM achieves higher mean success than LeWM in all four environments,
supporting the benefit of temporal decomposition even without a representational
hierarchy. H-JEPA provides further gains in some environments, particularly
substantial on AntMaze, suggesting that hierarchical representations offer
benefits beyond temporal decomposition alone. We next examine the mechanisms
through which temporal decomposition and hierarchical representations each
improve planning.

\Needspace{24\baselineskip}
\subsection{How hierarchy improves planning}
\label{sec:cost-surface}

\begin{figure}[t]
\begin{minipage}[t]{0.49\linewidth}
\vspace{0pt}
\centering
\footnotesize
\setlength{\tabcolsep}{1pt}
\begin{tabular}{@{}lccc@{}}
\toprule
& \multicolumn{3}{c}{Model depth} \\
\cmidrule(lr){2-4}
Cost / planner & 2 levels $\uparrow$ & 3 levels $\uparrow$ & 4 levels $\uparrow$ \\
\midrule
Native L1 & \AntDLTwoNative & \AntDLThreeNative & \AntDLFourNative \\
L2 projection & {\boldmath\AntDLTwoLtwo} & \AntDLThreeLtwo & \AntDLFourLtwo \\
L3 projection & -- & {\boldmath\AntDLThreeLthree} & \AntDLFourLthree \\
L4 projection & -- & -- & {\boldmath\AntDLFourLfour} \\
\midrule
Hierarchical & \AntDLTwoHier & \AntDLThreeHier & \AntDLFourHier \\
\bottomrule
\end{tabular}
\captionof{table}{\textbf{Changing only the cost on AntMaze.}
Success (\%, mean $\pm$ SE, three seeds).
Within each column, the first four rows share level-1 dynamics and planner
settings; only the cost space changes. Hierarchical planning is a separate reference.}
\label{tab:cost-ladder}
\end{minipage}\hfill
\begin{minipage}[t]{0.49\linewidth}
\vspace{0pt}
\centering
\includegraphics[width=\linewidth]{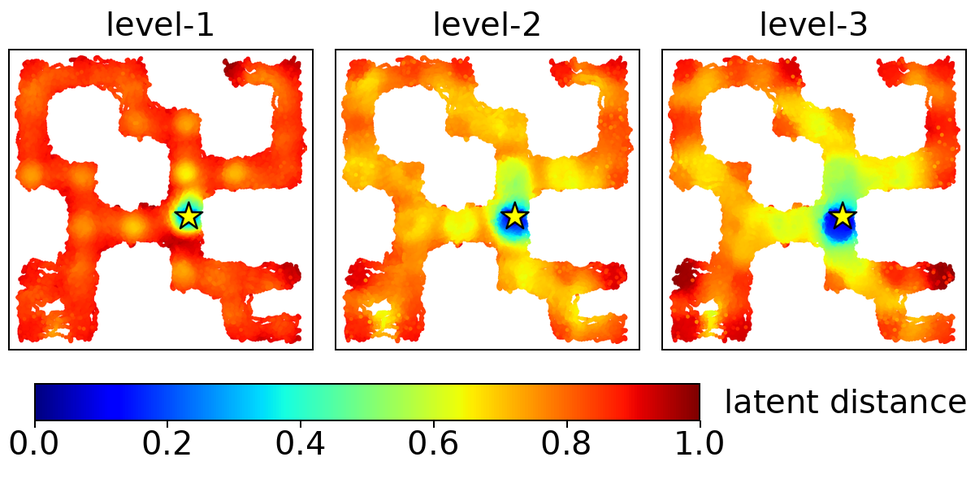}
\captionsetup{skip=6.06pt}
\caption{\textbf{Planning cost geometry on AntMaze.}
Distance to anchor $\star$ in each latent space of a three-level H-JEPA
(blue near, red far; one training seed).
Higher levels show more graded costs along corridors and a broader low-cost
region around the anchor.}
\label{fig:heatmaps-levels}
\end{minipage}
\end{figure}

Hierarchical planning changes two things at once: it decomposes the search in
time, and it scores candidate futures in a more abstract latent space. We
isolate the second factor, asking how much an objective in a higher, more
abstract level helps on its own, independent of temporal decomposition. We
focus on Visual AntMaze, where H-JEPA's substantial gains over HWM suggest
benefits from hierarchical representations beyond temporal decomposition alone
(Fig.~\ref{fig:depth-planning}, bottom).

\textbf{Planning costs across the hierarchy.} We visualize the cost geometry in
each level's latent space by mapping the latent distance from every state to a fixed
anchor (Fig.~\ref{fig:heatmaps-levels}). In a maze, a useful cost should correlate
with the length of the shortest path to the anchor, around the walls. In the
three-level model shown, level-1 distances are nearly constant away from the
anchor, providing little signal to guide a gradient planner until it is
close. At higher levels, costs rise smoothly along corridors from the low-cost
anchor toward distant states.
These results suggest that predicting over longer horizons both promotes
abstraction and gives higher levels a more informative cost toward distant goals.

\textbf{Changing only the cost.} To attribute planning gains to the cost geometry, we
perform flat planning with each model's level-1 world model, keeping rollout dynamics
fixed while varying only the latent space used to compute the goal cost.
The native L1 cost scores a rollout by the level-1 latent
distance $\lVert \hat{z}^{(1)} - g^{(1)}\rVert$ between the prediction and the goal.
The \emph{level-$k$ projection} maps both the prediction and the goal through the
upper encoders $E^{(2)}\circ\cdots\circ E^{(k)}$ of the same H-JEPA model before measuring
the distance; with window $w_\ell{=}1$ these encoders are pointwise, so a single
level-1 latent projects to every level. Each column of Table~\ref{tab:cost-ladder}
corresponds to a separately trained H-JEPA model with two, three or four levels,
evaluated on the tasks of Fig.~\ref{fig:depth-planning} (bottom). Within each
column, the cost-space rows share the same level-1 world model and planner
settings, varying only the latent space used to measure the goal cost. Each
model's full hierarchical planner serves as a reference.

For every tested multilevel H-JEPA, at least one
upper-level cost improves planning success over the native level-1 cost
(Table~\ref{tab:cost-ladder}). Gains are largest where native-space planning
performs poorly. The best projected level-1 planner also exceeds the
separately trained LeWM baseline in mean success at every tested depth on AntMaze
(LeWM: \AntDLLewm\%; Fig.~\ref{fig:depth-planning}, bottom).
These gains require only level-1 planning: the upper-level encoders define the
cost, but no upper-level planner generates subgoals. Learning a hierarchy of
representations can thus help planning simply by providing a more abstract space
in which to measure distance to the goal, before any temporal decomposition.
We extend this ablation to FourRoom, Push-T and Cube in
\S\ref{sec:negative-control} (Table~\ref{tab:cost-ladder-appendix}).
On FourRoom, abstract goal costs improve flat planning even though H-JEPA and
HWM perform similarly. On Push-T and Cube, where we observe no selective
abstraction, projected costs provide no clear benefit.

\begin{figure}[t]
\centering
\includegraphics[width=\linewidth]{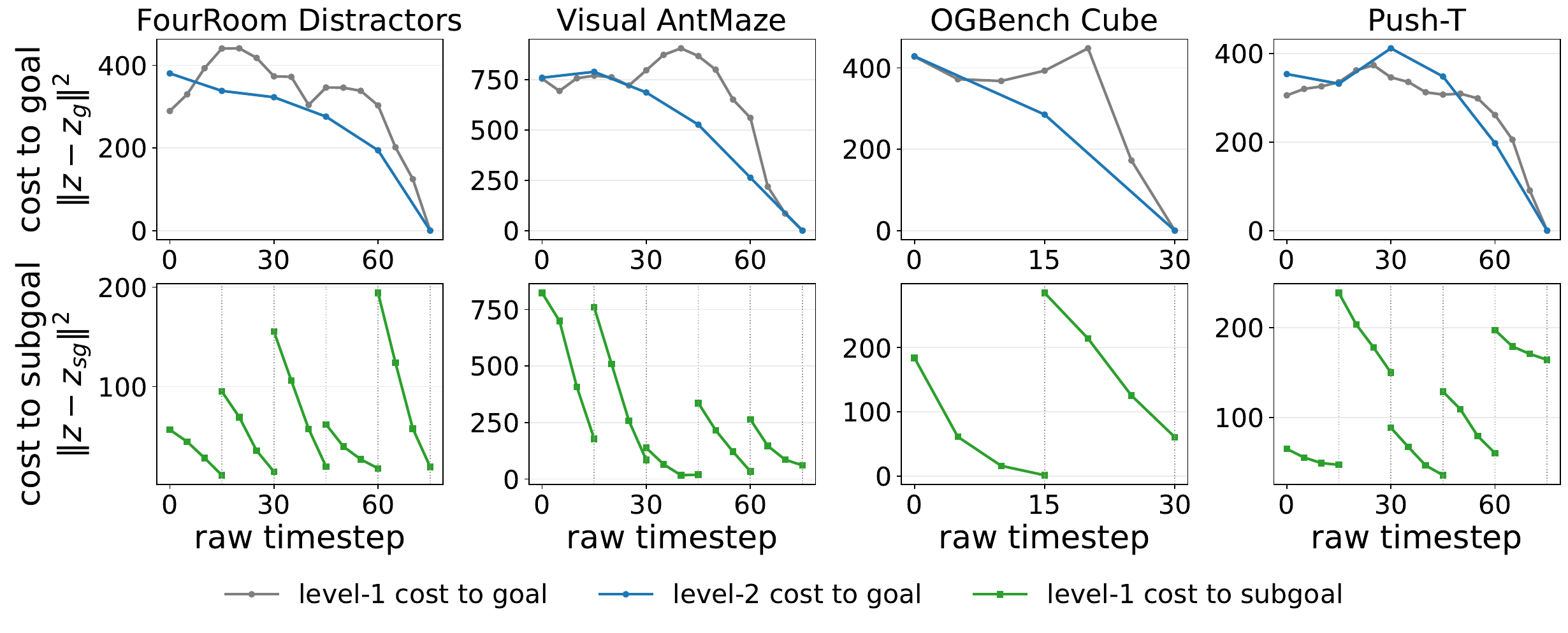}
\caption{\textbf{Temporal decomposition of planning costs.} We analyze a two-level H-JEPA with $s_2=3$. \emph{Top:} \textbf{Goal cost in level-1 and level-2 space.} For a held-out expert trajectory we encode each frame at both levels,
and compute the distance between each timestep's latent $z^{(\ell)}_t$ and the
goal (final-timestep) latent, $\|z^{(\ell)}_t-g^{(\ell)}\|_2^2$. At level 2 we subsample
at the training stride. \emph{Bottom:} \textbf{Subgoal
tracking cost.} The level-2 subgoals are the level-2 latents at the subsampled
timesteps. Within each subgoal segment we re-anchor the level-1 rollout from the encoded latent at the segment start, unroll it to the next
subgoal, project it into level-2 space, and measure its distance to that subgoal. Within each segment, the cost decreases monotonically toward the next subgoal.}
\label{fig:tmono-instance}
\end{figure}

\textbf{Temporal decomposition.}
Independently of shaping the planning objective through higher-level
representations, hierarchy helps by decomposing a long-horizon task into
subgoal problems. When each lower-level planner tracks only a prefix of the
upper-level plan ($K_{\ell+1}<H_{\ell+1}$), these subproblems require fewer sequential rollout steps per candidate
and limit the action search to shorter horizons. This reduces compute,
consistent with the improved success--compute tradeoffs in
Fig.~\ref{fig:compute-pareto} (top).
Second, a flat planner's full-horizon goal cost can stall or increase even along
expert trajectories, providing an inconsistent measure of progress toward the
goal (Fig.~\ref{fig:tmono-instance}).
Across all four simulated environments, the level-1 planner's costs for tracking
level-2 subgoals decrease more consistently than its full-horizon goal costs:
for H-JEPA, monotonicity (negative Spearman correlation between cost and time)
ranges from $\PushtTmonoHjepaProj$--$\DistrTmonoHjepaProj$ for subgoal tracking,
compared with $\AntTmonoLone$--$\PushtTmonoLone$ for the full-horizon goal cost
(Table~\ref{tab:tmono} in \S\ref{app:subgoal-monotonicity}).
Both HWM and H-JEPA exhibit this improvement, consistent with temporal
decomposition contributing to HWM's planning gains over LeWM even without
distinct abstract representations (Fig.~\ref{fig:depth-planning}, bottom).

\subsection{Scaling to diverse scenes: DROID}
\label{sec:droid}
We now extend H-JEPA to natural scenes that change across episodes. Once the scene varies across episodes, the background becomes the easily predictable part of the signal compared to the robot, so a purely predictive objective is pushed to encode the scene and discard the agent. To address this, we add an Inverse-Dynamics Modeling (IDM) loss term to the objective.

\paragraph{Environments and evaluation.}
In this section we use
DROID~\citep{khazatsky2024droid}, a dataset of real teleoperated manipulation videos whose scene, lighting, and manipulated objects all change between episodes. Details about data are in \Cref{app:envs}.
DROID has no simulator, so we score plans offline with the
\emph{Fr\'echet fidelity}, which measures how closely a plan's three-dimensional end-effector
path follows the expert's. Let $p$ and $\tilde p$ be the cumulative $xyz$ paths of the
planned and the expert delta-pose sequences, let $d_F$ be
the discrete Fr\'echet distance in meters, and let $\mathbf{0}$ be the zero-action path:
\begin{equation}
\label{eq:skill}
\text{fidelity} = \frac{d_F(\mathbf{0}, \tilde{p}) - d_F(p, \tilde{p})}{d_F(\mathbf{0}, \tilde{p})}.
\end{equation}
Thus, $0\%$ is satisfied by the arm not moving, $100\%$ is the expert path, and negative values are worse than
doing nothing. Scoring the whole path, rather than its endpoint, is why it is relevant for the non-greedy pick-and-place tasks we evaluate. Planners and budgets are in
\S\ref{app:planning-eval}, DROID clips and protocol in \S\ref{app:protocol}, and
calibration of this metric against success in simulation, where it ranks graded task
progress, in \S\ref{app:frechet}.

\paragraph{Inverse dynamics loss avoids slow-feature collapse.}
Against a non-stationary background a purely predictive objective admits a slow-feature
solution in which the encoder keeps the scene identity, which is predictable over the horizon, and discards the moving agent, which is only predictable with actions. A simple construction in \Cref{app:slowfeature} shows that this solution is always feasible when the regularizer sees only the marginal law of the embeddings, a failure mode related to slow-feature analysis~\citep{wiskott2002slow}.
To avoid such collapse, a loss has to act on the joint distribution of representations
across time. Following the use of inverse dynamics in
PLDM~\citep{sobal2025pldm} and related JEPA world models~\citep{terver2026ebjepa,sobal2022slowfeatures},
we add a third term
to the level objective of \S\ref{sec:objective}: an inverse-dynamics model that must regress, from a pair of consecutive state representations, the action relating them.
With $d_a$ the action dimension and $\operatorname{sg}[\cdot]$ the stop-gradient, the
level-$\ell$ IDM head $I^{(\ell)}$ gives the loss term
\begin{equation}
\label{eq:idm-loss}
\mathcal{L}^{(\ell)}_{\mathrm{idm}}
= \frac{1}{T-1}\sum_{t=1}^{T-1}
\frac{1}{d_a}\left\|
I^{(\ell)}\!\left(z^{(\ell)}_t,\, z^{(\ell)}_{t+1}\right)
- \operatorname{sg}\!\left[a^{(\ell)}_t\right]
\right\|_2^2 ,
\end{equation}
gets added to the level objective of Eq.~\eqref{eq:level-obj}.
The term follows that prior work but, to our knowledge, has not been paired with SIGReg, so we cross-validate both weights jointly (level-1 panel of Fig.~\ref{fig:crossval-l1grid-cls}).
At $\ell>1$ the regressed target is the
pooled macro-action $a^{(\ell)}_t$ of Eq.~\eqref{eq:level-state}, not a primitive action, so it is detached and the gradient reaches the encoder only through the two states.

\begin{figure}[t]
\centering
\newsavebox{\cornerbox}%
\newcommand{\cornerlabel}[2]{\sbox{\cornerbox}{#2}%
  \usebox{\cornerbox}\hspace{-\wd\cornerbox}%
  \rlap{\raisebox{\dimexpr\ht\cornerbox-1.6ex\relax}[0pt][0pt]{\small(#1)}}%
  \hspace{\wd\cornerbox}}%
\begin{subfigure}[c]{0.56\linewidth}
  \centering
  \phantomsubcaption\label{fig:scene-diversity}%
  \cornerlabel{a}{\includegraphics[width=\linewidth]{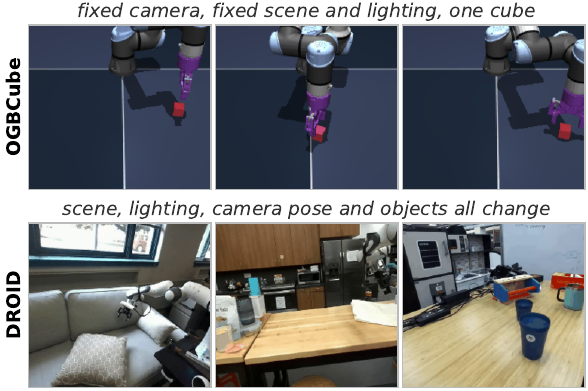}}
\end{subfigure}\hspace{0.012\linewidth}%
\begin{minipage}[c]{0.41\linewidth}
  \makebox[0.11905\linewidth][c]{%
    \raisebox{-0.5\height}{\rotatebox{90}{\fontsize{8}{9.6}\selectfont
      Fr\'echet fidelity (\%) $\rightarrow$}}}%
  \begin{minipage}[c]{0.88095\linewidth}
    \begin{subfigure}[t]{\linewidth}
      \raggedright
      \phantomsubcaption\label{fig:cls-ladder-droid}%
      \hspace*{0.01652\linewidth}%
      \cornerlabel{b}{\includegraphics[trim=34 0 0 0, clip,
        width=0.95764\linewidth]{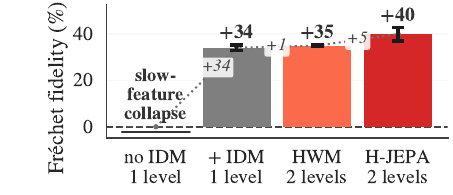}}
    \end{subfigure}

    \vspace{0.6em}

    \begin{subfigure}[t]{\linewidth}
      \raggedright
      \phantomsubcaption\label{fig:compute-pareto-real}%
      \cornerlabel{c}{\includegraphics[width=\linewidth]{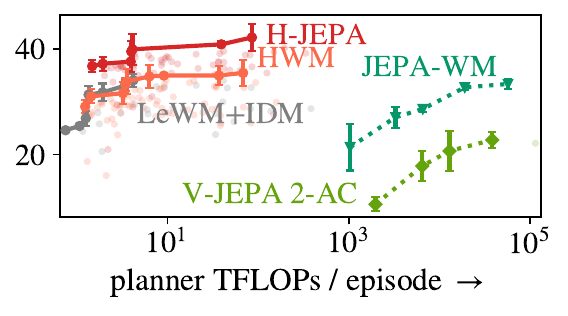}}
    \end{subfigure}
  \end{minipage}
\end{minipage}
\caption{\textbf{Scene diversity makes the inverse-dynamics loss necessary, and a second
level adds performance at a lower planner budget.} \emph{(a)} First frames of three Cube
and three DROID clips, showing scene variation across episodes.
\emph{(b)} DROID at $5$\,fps, Fr\'echet fidelity (\%), showing the slow-feature-collapsed
LeWM, LeWM+IDM, then two hierarchies: HWM, whose level~2 is an identity encoder, and
H-JEPA, whose level~2 is learned. \emph{(c)} The one-level and two-level models of
\emph{(b)} swept over planner budgets, Fr\'echet fidelity against measured planner TFLOPs
per episode on a log axis, the bold line being the Pareto front. Two frozen-encoder world
models are shown for reference: they are public checkpoints, evaluated at their native
$1$\,fps on the $15$ clips whose window fits at that rate, so their intervals span three
planner seeds. All other intervals are SE over three train seeds.}
\label{fig:scene-diversity-recipe}
\end{figure}

\paragraph{Hierarchical planning results.}
Without the loss term of
Eq.~\eqref{eq:idm-loss} the model collapses and performance is zero, as in Fig.~\ref{fig:cls-ladder-droid} and Fig.~\ref{fig:anticollapse}.
OGBench Cube escapes it because its scene barely varies: latent variance is $75\%$ \emph{between} episodes on DROID but only $28\%$ on Cube (\S\ref{app:varcomp}), so diverse scenes pull the encoder's capacity toward scene identity.
Fig.~\ref{fig:cls-ladder-droid} exhibits that adding IDM makes LeWM a strong baseline, and that a second level adds performance on top of it. Both
the representational lever and the temporal lever of \Cref{sec:cost-surface} are active on DROID, as shown by the clear advantage of H-JEPA over HWM. Finally, Fig.~\ref{fig:compute-pareto-real} shows, like in the single-scene datasets of
\Cref{sec:planning-performance}, that H-JEPA reaches higher performance at a lower
planning budget than the flat model. The frozen-encoder baselines, shown for reference, need orders of magnitude more compute for a lower fidelity.

\section{Related Work}
\label{sec:related}

Action-conditioned JEPA world models plan in learned latent spaces without
pixel reconstruction or reward supervision. DINO-WM~\citep{zhou2025dinowm}
learns dynamics over frozen pretrained features, V-JEPA\,2~\citep{assran2025vjepa2}
adds action-conditioned prediction after video pretraining, and
PLDM~\citep{sobal2025pldm} and LeWM~\citep{maes2026lewm} train the encoder and
predictor jointly from pixels. H-JEPA extends this end-to-end setting to a
hierarchy of latent spaces and prediction timescales, using LeJEPA's SIGReg
regularizer~\citep{balestriero2025lejepa} at every level.

Hierarchical video models learn representations at several timescales but
reconstruct observations and do not plan~\citep{saxena2021clockwork,kim2019vta,mounir2023streamer},
while hierarchical model-based control methods are reward-driven and
task-specific~\citep{hafner2022director,hansen2025puppeteer,gumbsch2024thick}.
Closest to our setting, HWM~\citep{zhang2026hierarchical} trains latent world
models at multiple temporal horizons from reward-free offline data and plans
hierarchically, but within a shared latent space.
H-JEPA additionally learns a distinct representation at each level, so that
higher-level objectives can discard detail that lower-level dynamics retain.
\S\ref{app:related-work} gives the extended survey and
Table~\ref{tab:hier-control} compares hierarchical control methods.

\section{Conclusion}
\label{sec:conclusion}

We presented H-JEPA, an end-to-end method for training a hierarchy of
action-conditioned JEPAs, each predicting in its own latent space over
progressively longer horizons. Higher levels discard fast features that are
unpredictable over their horizon and retain slow, predictable ones. Hierarchical planning
with H-JEPA outperforms flat planning at lower planner compute by scoring goals in a more
abstract latent space and decomposing long tasks into shorter subgoal problems.
With inverse-dynamics supervision, H-JEPA also improves offline, open-loop planning
fidelity on diverse real-robot videos from DROID. A next step is to test whether
these gains translate to closed-loop control on a physical robot.

Broader directions include learning hierarchies from action-free video and other
high-dimensional temporal streams; aligning upper-level latents with language so
that goals can be specified in words rather than as target observations; and training
upper levels on variable-duration segments instead of fixed strides, which may yield better
representations or subgoals for hierarchical planning.

\section*{Acknowledgments}

We thank Artem Zholus for his advice on robotic data and for setting up custom
RoboCasa tasks such as ObstacleReach; Raktim Goswami for generating alternative
RoboCasa data used in exploratory experiments; Yilun Kuang for providing access to
additional computational resources and helping run experiments; and Jean Ponce
for his valuable advice throughout the project.

\newpage

\setlength{\bibsep}{0.4ex plus 0.2ex}

\clearpage
\beginappendix

\startcontents[appendix]
\printcontents[appendix]{}{1}{\setcounter{tocdepth}{1}}
\clearpage

\section{Notation}
\label{app:notation}

\begingroup
\small
\setlength{\tabcolsep}{6pt}
\renewcommand{\arraystretch}{0.83}%
\begin{table}[H]
\centering
\caption{\textbf{Notation reference.} Symbols are introduced in \S\ref{sec:method} unless the row names another section.
Level superscripts are always parenthesized, $(\ell)$; a bare subscript is a time index and
$a{:}b$ is a half-open range. Appendix-local symbols are defined where they are used.}\label{tab:notation}
\begin{tabular}{@{}l p{0.78\textwidth}@{}}
\toprule
\textbf{Symbol} & \textbf{Description} \\
\midrule
\multicolumn{2}{@{}l}{\emph{Levels and indices}} \\
$\ell$, $L$ & Hierarchy level index, $\ell = 1,\ldots,L$; level $1$ is the finest \\
$t$ & Discrete time index, at the level's own rate \\
$a{:}b$ & Half-open slice: the $b{-}a$ indices $a,\ldots,b{-}1$ \\
\addlinespace[2pt]
\multicolumn{2}{@{}l}{\emph{Observations, latents and actions}} \\
$o_t$ & Image observation at time $t$ \\
$a_t$ & Primitive action block for the transition $o_t \to o_{t+1}$ \\
$z^{(\ell)}_t$ & Level-$\ell$ latent state; $z^{(1)}_t = E^{(1)}(o_t)$ \\
$\hat{z}^{(\ell)}_{1:H_\ell}$ & Action-conditioned autoregressive rollout at level $\ell$ \\
$\hat{z}^{(\ell+1),*}$ & Optimized rollout produced by the level above; $*$ marks an optimized quantity \\
$a^{(\ell)}_t$ & Level-$\ell$ action embedding (a macro-action for $\ell > 1$) \\
$a^{(\ell),*}_{0:H_\ell-1}$ & Optimized action sequence returned by the level-$\ell$ planner \\
$g^{(\ell)}$ & Goal state at level $\ell$, obtained by encoding the goal observation \\
$Z^{(\ell)}$ & Minibatch of level-$\ell$ embeddings, $Z^{(\ell)} \in \mathbb{R}^{n \times D}$: $n$ rows of width $D$ \\
$N$, $T$ & Sequences and timesteps of a minibatch of embeddings \\
$D$, $W$ & Embedding width, and the number of data-parallel ranks (\S\ref{app:sigreg-ddp}) \\
\addlinespace[2pt]
\multicolumn{2}{@{}l}{\emph{Learned components}} \\
$E^{(\ell)}$ & State encoder at level $\ell$; pools a $w_\ell$-step lower-level window \\
$A^{(\ell)}$ & Action encoder at level $\ell$; pools an $s_\ell$-step lower-level window for $\ell > 1$ \\
$F^{(\ell)}$, $I^{(\ell)}$ & Latent predictor and inverse-dynamics head at level $\ell$ \\
$\mathrm{Roll}^{(\ell)}$ & Autoregressive rollout operator induced by $F^{(\ell)}$ \\
\addlinespace[2pt]
\multicolumn{2}{@{}l}{\emph{Training objective}} \\
$\mathcal{L}^{(\ell)}$, $\mathcal{L}^{(\ell)}_{\mathrm{pred}}$, $\mathcal{L}^{(\ell)}_{\mathrm{idm}}$ & Total level-$\ell$ objective and its prediction and inverse-dynamics terms \\
$\mathrm{SIGReg}(Z^{(\ell)})$ & Sketched normality (anti-collapse) regularizer \\
$u_m$, $h_m$ & Unit direction $u_m \in \mathbb{S}^{D-1}$ and its projection $h_m = Z^{(\ell)} u_m$ (\S\ref{app:epps-pulley}) \\
$M$ & Number of random directions drawn by SIGReg (\S\ref{app:epps-pulley}) \\
$\mathcal{T}_{\mathrm{EP}}$ & Epps--Pulley normality statistic (\S\ref{app:epps-pulley}) \\
\addlinespace[2pt]
\multicolumn{2}{@{}l}{\emph{Planning}} \\
$H_\ell$ & Planning horizon at level $\ell$, in level-$\ell$ steps \\
$C^{(\ell)}(z,z')$ & Pairwise cost paid by the level-$\ell$ planner; squared Euclidean distance $\|z-z'\|_2^2$ within one latent space \\
$\tilde{z}^{(\ell+1)}_i$ & Level-$\ell$ rollout lifted into level-$(\ell{+}1)$ space by $E^{(\ell+1)}$ \\
$K_{\ell+1}$ & Number of subgoal terms in the level-$\ell$ cost; pinned by the lift, $K_{\ell+1} = \lfloor H_\ell/s_{\ell+1}\rfloor \le H_{\ell+1}$ \\
$g^{(\ell+1)}_i$ & The $i$-th subgoal handed to level $\ell$; the last is anchored, $g^{(\ell+1)}_{K_{\ell+1}} = g^{(\ell+1)}$ \\
$\beta$ & Weight of one subgoal term relative to the terminal term, Eq.~\eqref{eq:lower-plan} \\
$t_j$, $\xi_j$ & Replanning times, and primitive commands committed at phase $j$ (the commit schedule, \S\ref{app:planning-eval}) \\
$S$, $q$, $R$ & Planner candidates per level, scored final steps of a rollout, and level-1 blocks executed per replan (\S\ref{app:planning-eval}) \\
\addlinespace[2pt]
\multicolumn{2}{@{}l}{\emph{Architecture and loss hyperparameters}} \\
$s_\ell$, $w_\ell$ & Temporal stride and window size from level $\ell{-}1$ to level $\ell$: each level-$\ell$ state spans $w_\ell$ lower-level states, taken every $s_\ell$ \\
$c_\ell$ & Context length, in level-$\ell$ latent states \\
$\lambda_\ell$, $\gamma_\ell$ & SIGReg and inverse-dynamics weights at level $\ell$ \\
\bottomrule
\end{tabular}
\end{table}
\endgroup
\clearpage

\section{Environments and datasets}
\label{app:envs}

We evaluate on two navigation environments (FourRoomDistractors and Visual AntMaze), two
simulated manipulation environments (Push-T and OGBench Cube) and the real-robot DROID
corpus.
Push-T and Cube follow the LeWM setup~\citep{maes2026lewm} in environment and
dataset; FourRoomDistractors, Visual AntMaze and their datasets are ours; DROID is a
public release. Observations are RGB, $224{\times}224$ except
Visual AntMaze at $64{\times}64$ and DROID at $256{\times}256$;
actions are continuous.

\begin{enumerate}[label=\alph*)]
\item \textbf{FourRoomDistractors} is a continuous 2D navigation environment derived
from the two-room environment of \citet{sobal2025pldm}. The world is divided into four
equal rooms by two walls, each with one door; every door is closed independently with
probability $0.5$ at reset, so the route between two rooms changes across episodes.
The ego agent is a dot moving $2.5$\,px per step under a 2D velocity command. A single
distractor dot of the same speed moves under a smooth random policy that ignores the
ego, and teleports to a random free location at Poisson-distributed
intervals of mean $35$ steps; it neither collides with the ego nor affects the task.
We collect $3{,}750$ episodes of $400$ steps ($1.5$M transitions) with a per-episode
mixture of two ego policies: a heavily noised shortest-path expert that picks a new
goal each time it reaches the current one, and a smooth random-exploration policy.
Level 1 is trained on $1.0$M of these transitions (Table~\ref{tab:impl-training}), and
$150$ episodes of this collection serve as validation data.

\item \textbf{Visual AntMaze} is the OGBench \texttt{visual-antmaze-medium}
environment~\citep{park2025ogbench}: an 8-DoF quadruped navigating a maze from a
third-person camera over a color-coded floor, so that the agent's location can be read
from the image, with an 8D torque command in $[-1,1]$. The world model additionally
receives the $27$-dimensional proprioceptive state (joint positions without the global
$xy$, and joint velocities). We collect data with the OGBench protocol, a SAC
low-level locomotion expert driven by a commanded direction, in two of its variants.
\texttt{explore} resamples a random direction every $10$ steps under heavy action
noise, giving low-quality but high-coverage trajectories. \texttt{stitch} drives the
lightly noised expert along the maze's shortest path toward a goal four cells away, giving
short goal-directed segments that must be stitched to solve longer tasks. Each variant has
$12{,}500$ episodes of $400$ steps ($5.0$M transitions), and the training set combines both
($10.05$M transitions). Validation uses $1{,}250$ held-out \texttt{stitch} episodes.

\item \textbf{Push-T} is the continuous 2D manipulation task of
\citet{chi2023diffusionpolicy} in which a circular agent pushes a T-shaped block to a
target configuration; the 2D command is the agent's target position. We use the expert
dataset of DINO-WM~\citep{zhou2025dinowm} as released with LeWM: $18{,}685$ training
episodes ($2.34$M frames, $125$ steps on average) and $21$ validation episodes.

\item \textbf{OGBench Cube} is the single-cube variant of the OGBench
\texttt{cube} manipulation environment~\citep{park2025ogbench}, in which a robot arm
must pick up a cube and place it at a target location; the 5D command sets the
end-effector displacement, yaw, and gripper. As in LeWM, the data are $10{,}000$
episodes of $200$ steps from the benchmark's scripted data-collection policy, split
into $9{,}800$ training and $200$ validation episodes.

\item \textbf{DROID}~\citep{khazatsky2024droid} is a corpus of real teleoperated
manipulation video with no simulator, in which scene, lighting, camera calibration and
manipulated objects all change between episodes. The end-effector command is
$7$-dimensional: Cartesian $xyz$, roll-pitch-yaw and gripper. Actions are normalized per
coordinate against dataset statistics before training, so the gripper dimension, one to two
orders of magnitude larger than the Cartesian deltas, does not dominate the mean in
Eq.~\eqref{eq:idm-loss}. We train on a re-encoded copy ($74{,}530$ episodes, $220{,}084$
views, three cameras at $256{\times}256$), described in \S\ref{app:droid-pipeline}.
Plans are scored offline against logged expert
trajectories on a frozen $16$-clip evaluation set, specified in \S\ref{app:protocol}.
\end{enumerate}

\section{Architecture and training}
\label{app:implementation}

This appendix collects the architecture and training settings of every model in the
paper: the LeWM baseline and the two-, three-, and four-level H-JEPA models of the
simulated environments (\S\ref{app:implementation-sim}), then the DROID models
(\S\ref{app:hparams}). Probing and decoding are
described in \S\ref{app:probing}, and planning evaluation in
\S\ref{app:planning-eval}.

\subsection{Single-scene environments}
\label{app:implementation-sim}

These are the models of Fig.~\ref{fig:depth-planning} (bottom). LeWM is a single-level JEPA using
the architecture and objective of \citet{maes2026lewm}. Each H-JEPA model trains all
levels jointly from scratch in one end-to-end run, with upper-level losses
backpropagating through the encoders below. We describe shared settings once and
list environment-specific choices in Table~\ref{tab:impl-training}.

\paragraph{Data.} Within each environment, all four models use the same training
data: $1.0$M transitions from the one-distractor FourRoom dataset; a $10.05$M-transition,
equal mixture of AntMaze \texttt{explore} and \texttt{stitch}; the Push-T expert set;
or the Cube scripted set (\S\ref{app:envs}). Frames use ImageNet
normalization without augmentation. A level-1 step spans $5$ environment steps,
with the intervening actions concatenated. Every upper level has stride
$s_\ell=2$ and window $w_\ell=1$. Each sample is one clip long enough to supply
$T=4$ states at the top level; each lower level takes a random four-state crop of
its own stream. The causal predictor receives $c_\ell=3$ states at every level and learns their
teacher-forced one-step successors. If $n_\ell$ is the number of level-$\ell$
frames needed in the clip, $n_L=4$ and
$n_{\ell-1}=(n_\ell-1)s_\ell+w_\ell$, giving the minimum clip lengths in
Table~\ref{tab:impl-data}. Batches contain $128$ clips. A clip must lie within one
episode, so the number of clips an epoch yields decreases with the minimum clip length
(Table~\ref{tab:impl-data}): slightly on FourRoom and AntMaze, whose episodes are
long, and sharply on Push-T, whose episodes are short. H-JEPA training is
epoch-matched across depths, so deeper models take fewer gradient steps.

\begin{table}[ht]
\centering
\small
\caption{\textbf{Temporal layout of one training sample and the clips it leaves per
epoch.} The minimum clip length is the shortest episode segment that yields one
training sample, in environment steps, including both endpoint observations. All H-JEPA
models use stride two and window one at every level above the first. A clip must
lie within one episode (the loader reserves $5\,n_1$ environment steps per clip),
so the number of clips an epoch yields falls as the minimum clip length grows; parentheses give each count as a percentage of the LeWM count in the same
environment. Each clip supplies three teacher-forced transitions per level.}
\label{tab:impl-data}
\setlength{\tabcolsep}{3.0pt}   %
\begin{tabular}{lcccc}
\toprule
 & LeWM & Two-level H-JEPA & Three-level H-JEPA & Four-level H-JEPA \\
\midrule
Level-1 frames per clip & $\ImplFramesOne$ & $\ImplFramesTwo$ & $\ImplFramesThree$ & $\ImplFramesFour$ \\
Minimum clip length (env.\ steps) & $\ImplSpanOne$ & $\ImplSpanTwo$ & $\ImplSpanThree$ & $\ImplSpanFour$ \\
\midrule
\multicolumn{5}{l}{\emph{Training clips per epoch}} \\
FourRoom & $\ImplRoomClipsOne$ & $\ImplRoomClipsTwo$ ($\ImplRoomClipsTwoPct\%$) & $\ImplRoomClipsThree$ ($\ImplRoomClipsThreePct\%$) & $\ImplRoomClipsFour$ ($\ImplRoomClipsFourPct\%$) \\
AntMaze & $\ImplAntClipsOne$ & $\ImplAntClipsTwo$ ($\ImplAntClipsTwoPct\%$) & $\ImplAntClipsThree$ ($\ImplAntClipsThreePct\%$) & $\ImplAntClipsFour$ ($\ImplAntClipsFourPct\%$) \\
Push-T & $\ImplPushClipsOne$ & $\ImplPushClipsTwo$ ($\ImplPushClipsTwoPct\%$) & $\ImplPushClipsThree$ ($\ImplPushClipsThreePct\%$) & $\ImplPushClipsFour$ ($\ImplPushClipsFourPct\%$) \\
Cube & $\ImplCubeClipsOne$ & $\ImplCubeClipsTwo$ ($\ImplCubeClipsTwoPct\%$) & $\ImplCubeClipsThree$ ($\ImplCubeClipsThreePct\%$) & $\ImplCubeClipsFour$ ($\ImplCubeClipsFourPct\%$) \\
\bottomrule
\end{tabular}
\end{table}

\paragraph{Training dynamics.} Fig.~\ref{fig:training-dynamics} tracks the
unweighted prediction and SIGReg terms of every level of the four-level models
over training. In a healthy run, the SIGReg statistic drops sharply early in
training, then settles within a small factor of its value on exactly Gaussian
embeddings, about $1$ (\S\ref{app:sigreg-ddp}). The prediction
loss often rises over the first epochs, while SIGReg spreads out the initially
concentrated embeddings, then decreases to a nonzero plateau. A prediction loss
falling toward zero is the signature of collapse, since constant embeddings are
trivially predictable. Neither term degenerates at any level, so
end-to-end training keeps all levels non-collapsed at the deepest setting we
train. Both terms are ordered by level in every environment, each higher level
converging to a larger prediction loss and a larger SIGReg value, consistent with
the fewer clips and coarser targets an upper level sees
(Table~\ref{tab:impl-data}). Upper levels also start improving later than
level~1, most visibly on AntMaze, where the abstract predictors hold a plateau
over the first epochs before descending. On Push-T and Cube the level-3 and
level-4 prediction curves are the noisiest and flatten earliest, in line with
the short-episode clip deficit of Table~\ref{tab:impl-data}. The total
training objective, summed over levels, decreases throughout training in every
environment and is flat by the final epochs.

On DROID (bottom row of Fig.~\ref{fig:training-dynamics}), no term degenerates
either. The prediction loss follows the same ordering as in the other
environments, with level~2 converging well above level~1. The SIGReg ordering is
reversed, level~2 ending below level~1, consistent with its twice larger SIGReg
weight. Level~1 of H-JEPA also reaches lower prediction and inverse-dynamics
losses than the flat LeWM+IDM trained with the same weights, so training the
upper level jointly does not degrade the lower one. The total training objective, almost
entirely the weighted inverse-dynamics term, decreases monotonically up to
epoch-to-epoch noise.

\begin{figure}[htbp]
\centering
\includegraphics[width=\linewidth]{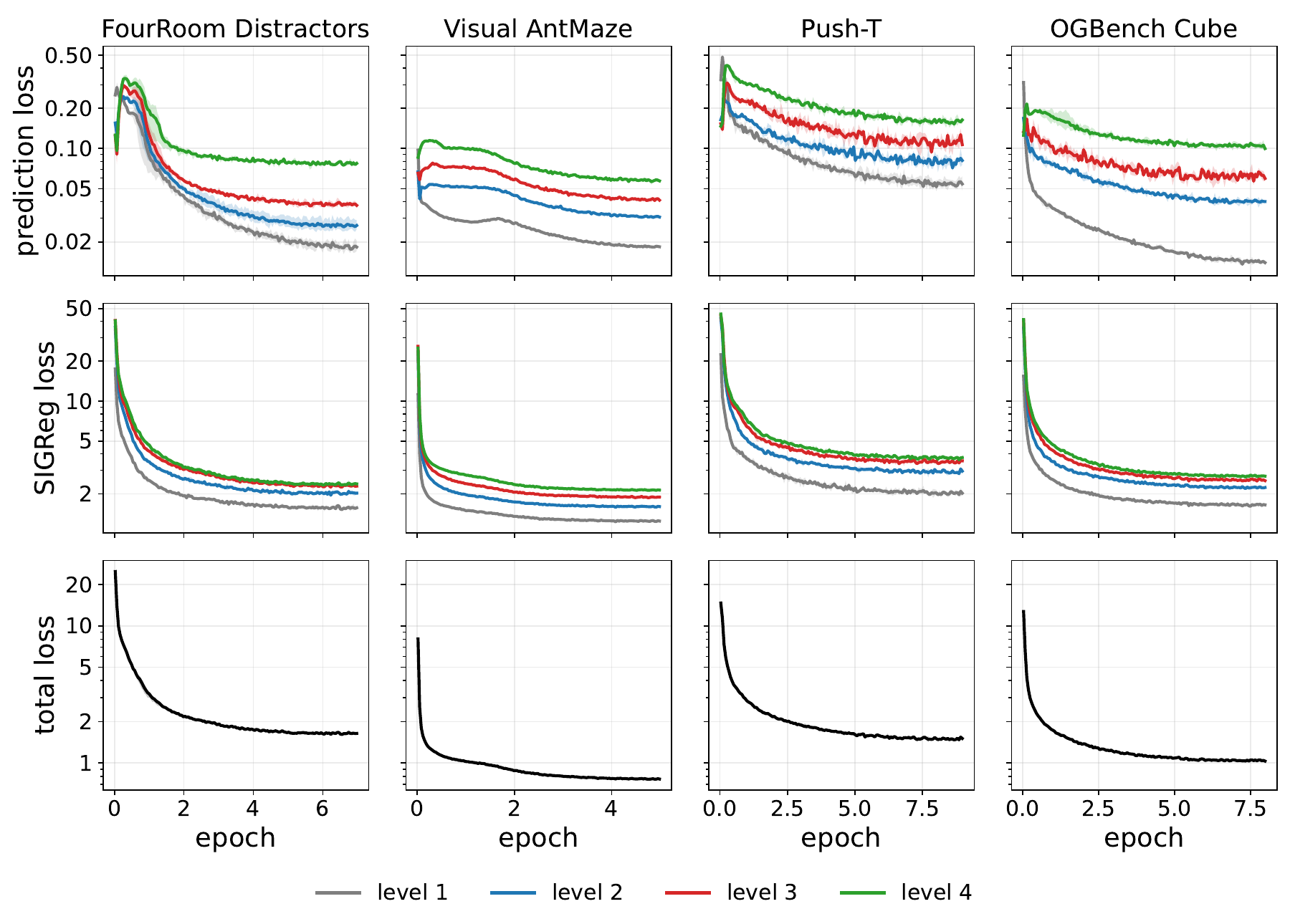}
\vspace{0.5em}
\includegraphics[width=\linewidth]{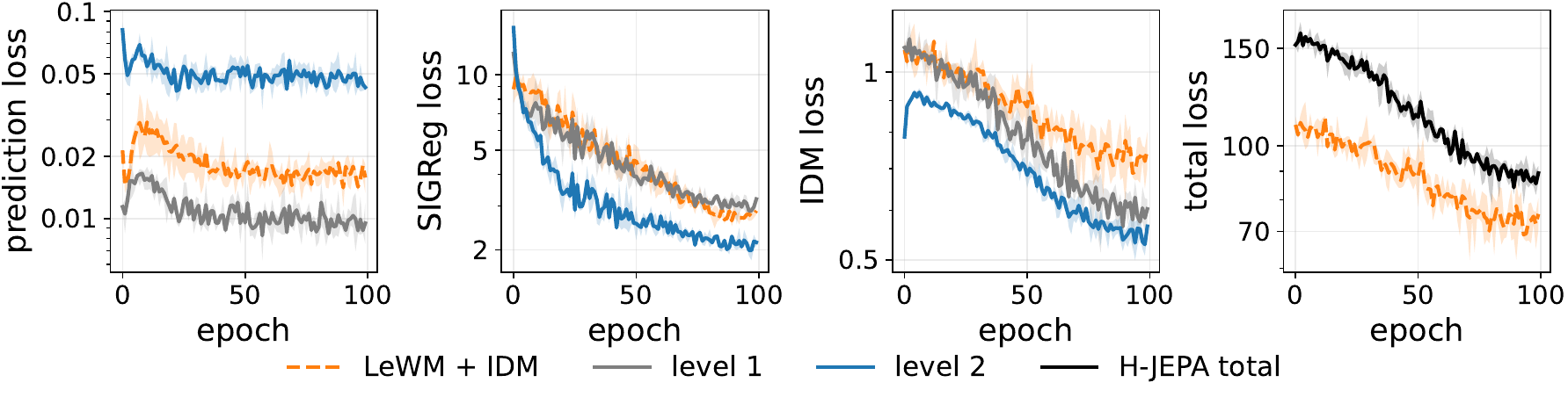}
\caption{\textbf{Per-level training losses of the four-level H-JEPA models and
of the DROID models.} Top three rows: unweighted prediction loss and SIGReg
statistic of each level of the four-level models over training, and their
total training objective, the sum over levels of the prediction loss plus the
weighted SIGReg term. Every level of a model uses the same SIGReg weight
($0.18$ on FourRoom, $0.09$ elsewhere), so the weighted curves differ from the
unweighted ones by a constant factor per environment. The seed band of the
total objective is narrower than its line.
Bottom row: unweighted prediction, SIGReg and inverse-dynamics losses of the
DROID flat LeWM+IDM (orange, dashed) and of both levels of the two-level H-JEPA
(solid) over their $100$ training epochs, and the total training objective, the sum
of all terms weighted as in Table~\ref{tab:hparams} (level~1: SIGReg $0.08$,
IDM $100$; level~2: SIGReg $0.16$, IDM $50$), summed over both levels for
H-JEPA. All curves are the mean over three training seeds with a min/max band,
on a log scale.}
\label{fig:training-dynamics}
\end{figure}

\paragraph{Architecture.} The level-1 encoder is a ViT-Tiny whose \texttt{[CLS]}
token passes through a one-hidden-layer MLP projector with Batch Normalization.
On AntMaze, its $256$-dimensional pixel projection is concatenated with a
$128$-dimensional embedding of the proprioceptive state from a two-layer MLP.
The level-1 predictor is a causal transformer with $4$ layers, $8$ heads, width
$192$, feed-forward width $1024$, action conditioning through adaptive LayerNorm,
and an MLP projector of the same form as the encoder's. Each upper level uses a
two-layer MLP encoder with hidden width $384$, applied to the lower-level latent;
a one-layer transformer pools the two intervening lower-level action embeddings
into a $4$-dimensional macro-action. The upper-level predictor has the same structure
as at level~1, with dimensions given in Table~\ref{tab:impl-training}. Within an
environment, latent dimensions and upper-level architectures are shared across
all H-JEPA depths.

\paragraph{Training.} We minimize the sum of
$\mathcal{L}_{\mathrm{pred}}+\lambda_\ell\,\mathrm{SIGReg}$ over levels
(\S\ref{sec:objective}), with $M=1024$ random projections and no stop-gradient,
moving-average targets, or pretrained weights. All models use AdamW with weight
decay $10^{-3}$, linear warm-up (1\%) followed by cosine decay, gradient clipping at
$1.0$, and \texttt{bfloat16}. Peak learning rates are $5\times10^{-5}$ for level~1
and $10^{-5}$ for every upper level. The SIGReg coefficient is uniform across
levels within each H-JEPA model. Training uses three seeds ($42$, $43$, $44$).
The two-, three-, and four-level H-JEPA runs share an epoch budget within each
environment; longer clips change the number of valid samples, so optimizer-step
counts and training FLOPs are not matched across depths.

\begin{table}[ht]
\centering
\small
\setlength{\tabcolsep}{5pt}
\caption{\textbf{Per-environment architecture and training settings.} Upper-level
predictor dimensions and H-JEPA epoch counts apply to every depth. H-JEPA models
are initialized independently of the LeWM baselines.}
\label{tab:impl-training}
\begin{tabular}{lcccc}
\toprule
 & FourRoom & AntMaze & Push-T & Cube \\
\midrule
Frame size / patch & $\ImplRoomImage$ / $\ImplRoomPatch$ & $\ImplAntImage$ / $\ImplAntPatch$ & $\ImplPushImage$ / $\ImplPushPatch$ & $\ImplCubeImage$ / $\ImplCubePatch$ \\
Command / proprio dim. & $\ImplRoomAction$ / n/a & $\ImplAntAction$ / $\ImplProprio$ & $\ImplPushAction$ / n/a & $\ImplCubeAction$ / n/a \\
Latent dim.\ $d$ & $\ImplRoomLatent$ & $\ImplAntLatent$ & $\ImplPushLatent$ & $\ImplCubeLatent$ \\
Upper predictor layers / heads & $\ImplRoomLayers$ / $\ImplRoomHeads$ & $\ImplAntLayers$ / $\ImplAntHeads$ & $\ImplPushLayers$ / $\ImplPushHeads$ & $\ImplCubeLayers$ / $\ImplCubeHeads$ \\
Upper predictor width / FF width & $\ImplRoomWidth$ / $\ImplRoomMlp$ & $\ImplAntWidth$ / $\ImplAntMlp$ & $\ImplPushWidth$ / $\ImplPushMlp$ & $\ImplCubeWidth$ / $\ImplCubeMlp$ \\
\midrule
SIGReg $\lambda$, LeWM / H-JEPA & $\ImplRoomLeWMLambda$ / $\ImplRoomLambda$ & $\ImplAntLeWMLambda$ / $\ImplAntLambda$ & $\ImplPushLeWMLambda$ / $\ImplPushLambda$ & $\ImplCubeLeWMLambda$ / $\ImplCubeLambda$ \\
H-JEPA epochs & $\ImplRoomEpochs$ & $\ImplAntEpochs$ & $\ImplPushEpochs$ & $\ImplCubeEpochs$ \\
\bottomrule
\end{tabular}
\end{table}

\paragraph{Implementation and hardware.} We use the \texttt{stable-worldmodel} and
\texttt{stable-pretraining} libraries of LeWM~\citep{maes2026lewm}, and train
on H200 GPUs.

\subsection{DROID}
\label{app:hparams}

At level $\ell>1$, the macro-actions are embeddings produced by $A^{(\ell)}$ rather than raw
actions. To avoid slow-feature collapse, we also apply SIGReg to these action embeddings, with
weight $0.005$; level~1 carries no such term. Our DROID models see pixels and actions only, no proprioception.

The data pipeline and learning rates also differ from the single-scene environments
(\S\ref{app:implementation-sim}). There, the levels are balanced: each lower level takes a
random four-state crop of its own stream, so every level learns from three teacher-forced
transitions per clip, and every upper level trains at a fifth of the level-1 learning rate.
On DROID, a training clip supplies eight level-1 states, and level~2 reads three of them at
stride $3$, with no crop of its own. With one-step teacher forcing, level~2 thus receives two
prediction targets per clip against seven at level~1. Both levels share the peak learning
rate of Table~\ref{tab:hparams}.

Table~\ref{tab:hparams} gives the training settings of the headline two-level run.
A stagewise variant instead trains level~1 alone, then freezes it and trains level~2
on its latents. We use this variant only for the data-specialization analysis
(\S\ref{sec:dataset-per-level}), the level-2 coefficient search
(Fig.~\ref{fig:crossval-l1grid-cls}), and the decoded plans
(Fig.~\ref{fig:decoded-plans}); all headline results use end-to-end training.
The architecture and loss settings describe the model rather than how
it is scheduled, so they apply equally to the end-to-end and stagewise procedures;
the optimization settings are those of the end-to-end headline runs.
Level 1 consumes raw actions, which is why the action regularizer appears only in the level-2 block.

The level-1 and optimization blocks of Table~\ref{tab:hparams} also apply to one-level
LeWM+IDM. Frames are center-cropped to a square, resized to
$256{\times}256$ and normalized with ImageNet statistics, with no augmentation beyond the
random choice of clip. Level~2 has stride $s_2=3$ and window $w_2=1$: its encoder is
applied to every third level-1 state on its own. Its macro-action encoder concatenates
the raw $7$-dimensional commands of the three level-1 steps it spans and maps them to an
$8$-dimensional macro-action through an MLP with one hidden layer of width $64$ and a
final LayerNorm. At both levels the predictor is a
causal transformer with $6$ heads of width $64$, feed-forward width $1536$ and dropout
$0.1$, conditioned on actions through adaptive LayerNorm. All SIGReg terms, including the action term,
use $M=1024$ random projections.

\begin{table}[ht]
\centering
\small
\setlength{\tabcolsep}{6pt}
\caption{\textbf{Training hyperparameters on DROID.} Predictors are causal transformers, given as
width/depth. The grouping row
gives the shape of the tensor SIGReg tests: $(T,N,D)$ tests normality within each timestep,
over the $N$ sequences of the batch. The architecture and loss settings do not depend on
the training schedule, whether end-to-end or stagewise (\S\ref{app:hparams}); the
optimization block does, and reports the end-to-end headline runs. The peak learning rate is
the same at both levels, AdamW runs at $\beta = (0.9, 0.999)$, and the dataloader length caps the
iterations per epoch, which gives $292$ at this batch size.}
\label{tab:hparams}
\begin{tabular}{lc}
\toprule
 & DROID \\
\midrule
\multicolumn{2}{l}{\emph{Global}} \\
Stop-gradient on latent targets & None \\
Level-1 to level-2 stride       & $3$ \\
\midrule
\multicolumn{2}{l}{\emph{Level 1}} \\
Encoder                         & ViT-S/16, \texttt{CLS} \\
Predictor                       & $384$/$12$ \\
Predictor context $c_1$         & $7$ states \\
SIGReg grouping                 & $(T,N,D)$ \\
SIGReg weight $\lambda_1$       & $0.08$ \\
IDM weight $\gamma_1$           & $100$ \\
\midrule
\multicolumn{2}{l}{\emph{Level 2}} \\
Encoder                         & latent MLP $768$ \\
Predictor                       & $384$/$12$ \\
Predictor context $c_2$         & $2$ states \\
SIGReg grouping                 & $(T,N,D)$ \\
SIGReg weight $\lambda_2$       & $0.16$ \\
IDM weight $\gamma_2$           & $50$ \\
Action SIGReg weight            & $0.005$ \\
\midrule
\multicolumn{2}{l}{\emph{Optimization}} \\
Optimizer                       & AdamW \\
Total batch size                & $256$ \\
Peak learning rate              & $5\times10^{-4}$ \\
Warm-up, fraction of steps      & $0.1$, linear \\
Decay                           & cosine to $0$ \\
Weight decay                    & $10^{-4}$ \\
Gradient clipping, per level    & $1.0$ \\
Precision                       & \texttt{bfloat16} \\
Epochs                          & $100$ \\
Iterations per epoch            & $292$ \\
Training seeds                  & $1$, $1000$, $10000$ \\
\bottomrule
\end{tabular}
\end{table}

\section{Probing and decoding}
\label{app:probing}

All probe and decoder results are obtained \emph{post hoc}: after training, the world
model is frozen and every frame of a held-out probing set is encoded through the
hierarchy without striding, giving one detached latent per level for every
frame. Probes and decoders are then fitted on these latents.

\paragraph{Probes.} Each probed entity has a two-layer MLP (hidden width $512$,
LayerNorm, GELU) per level, mapping the latent to the $z$-scored entity vector under a
mean-squared-error loss. We report the variance-normalized error (NMSE) on the
evaluation split: the per-dimension de-normalized squared error divided by
the entity's variance and averaged equally over dimensions
(\S\ref{sec:learning-abstractions}).

\paragraph{Decoder.} Following LeWM, a lightweight transformer decoder reconstructs the
frame from a single latent, for visualization only. The latent is projected to width
$256$ and serves as key and value for a set of learnable query tokens, one per
encoder patch ($16{\times}16$ queries at both resolutions); four
cross-attention blocks with $8$ heads and residual MLPs update the queries, and a linear
layer maps each to its pixel patch. The loss is the mean-squared error to the
ImageNet-normalized frame.

\paragraph{Training.} Probes and decoders are trained jointly with AdamW (weight
decay $10^{-4}$), batch size $128$, for $25$ epochs ($60$ on FourRoomDistractors);
probes use learning rate $10^{-3}$ and decoders $10^{-4}$ ($10^{-3}$ on Visual
AntMaze). Table~\ref{tab:impl-probing} lists the probing data and the probed entities
per environment. Entities are grouped by their temporal frequency into the slow entity
that level~2 retains and the fast entity it discards (\S\ref{sec:learning-abstractions}); Push-T
and Cube form a single frequency band, so all of their entities are retained and none is
marked fast.

\begin{table}[ht]
\centering
\small
\setlength{\tabcolsep}{5pt}
\caption{\textbf{Probing data and probed entities per environment.} Transition counts
are for fitting (train) and scoring (eval). Visual AntMaze is scored on one evaluation
set with equal numbers of \texttt{explore} and \texttt{stitch} episodes. On FourRoomDistractors and Visual
AntMaze, whose state separates into slow and fast entities
(Fig.~\ref{fig:selective-bars}), the entities are labeled accordingly. Entity
dimensions in parentheses.}
\label{tab:impl-probing}
\begin{tabular}{lll}
\toprule
Environment & transitions (train / eval) & Probed entities \\
\midrule
FourRoomDistractors & $160$k / $40$k & \makecell[l]{\textbf{slow:} ego $xy$ (2)\\ \textbf{fast:} distractor $xy$ (2)} \\
\addlinespace
Visual AntMaze & $500$k / $50$k & \makecell[l]{\textbf{slow:} global $xy$ (2)\\ \textbf{fast:} body state (27)} \\
\addlinespace
Push-T & $250$k / $2.5$k & \makecell[l]{agent $xy$ (2), block $xy$ (2),\\ block orientation (2), agent velocity (2)} \\
\addlinespace
OGBench Cube & $250$k / $40$k & \makecell[l]{cube position (3), cube orientation (6),\\ effector position (3), effector yaw (2), gripper (2),\\ arm joint positions (6), arm joint velocities (6)} \\
\bottomrule
\end{tabular}
\end{table}

\section{Planning evaluation}
\label{app:planning-eval}

\paragraph{Evaluation tasks.} FourRoomDistractors and
Push-T use pairs of held-out expert frames $75$ environment steps apart; the
FourRoom start and goal lie two rooms apart. Cube uses held-out $20$-step segments
centered on a grasp, so the task requires manipulating the object. Visual AntMaze
uses expert rollouts between maze cells at grid distance three ($82$ steps on
average), with neutral-pose start and goal states. The planner receives goal
images and, on AntMaze, proprioception; success is determined by the environment's
goal test within the episode budget in Table~\ref{tab:impl-planning} (on AntMaze,
within $1.0$ of the goal $xy$). Fig.~\ref{fig:eval-tasks} shows one successful
H-JEPA execution per environment.
We use separate task sets for the compute sweeps
in Fig.~\ref{fig:depth-planning} (top) and the final evaluations in
Fig.~\ref{fig:depth-planning} (bottom) and
Tables~\ref{tab:cost-ladder} and~\ref{tab:cost-ladder-appendix}.
The final evaluation set contains $50$ start--goal tasks per environment,
generated by the same process and with the same difficulty settings as the
compute-sweep tasks, but using different task-generation seeds. Planner settings
are fixed before evaluation on these new tasks. All methods and model depths
share the same tasks within each set.

\begin{figure}[t]
\centering
\includegraphics[width=\linewidth]{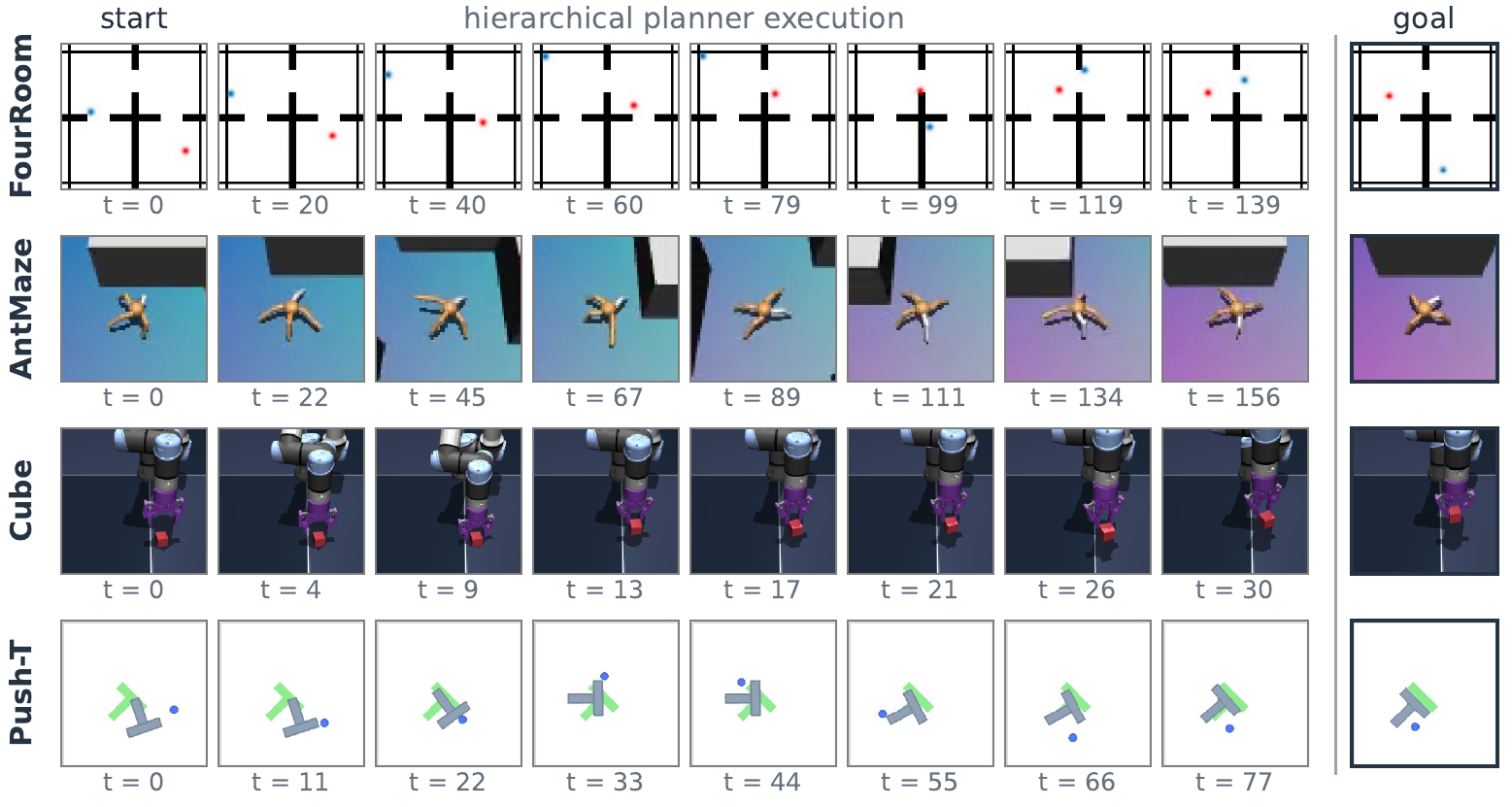}
\caption{\textbf{Successful hierarchical planner executions.} Selected H-JEPA
executions on one evaluation task per environment, using three-level models for
FourRoom and AntMaze and two-level models for Cube and Push-T. These checkpoints
also appear in Fig.~\ref{fig:depth-planning} (bottom); the recordings use their original
evaluation settings rather than the compute-selected sample counts. Each row
shows eight observations sampled uniformly in environment steps, followed by the
goal observation supplied to the planner. Rows end at first recorded success,
except Cube, which shows the full $30$-step execution to include the completed
lift. Labels give elapsed environment steps.}
\label{fig:eval-tasks}
\end{figure}

\paragraph{Shared solver.} Every active level optimizes candidate action sequences
by differentiating through its latent rollout, using AdamW for at most $90$
iterations with early stopping when the best cost stalls, and returns the lowest-cost
candidate. Level-1 candidates are initialized from a standard Gaussian in normalized
action space. Upper levels have no native action bounds, so their search is anchored
to the macro-actions seen in training: each upper-level action encoder keeps its last
$4000$ training macro-actions, candidates are initialized from a per-dimension Gaussian
fitted to them (one candidate at its mean), and after every optimizer step the
candidates are clipped to that sample's per-dimension 2nd--98th percentiles. The latent-distance cost scores the last
two predicted states at level 1 and the last state at upper levels. H-JEPA plans
top-down, passing the first predicted latent as a subgoal for the level below,
whose rollout is encoded into the upper latent space for comparison
(\S\ref{sec:method-planning}). On Push-T and Cube, level-1 actions are also pooled
and matched to the first level-2 macro-action with weight $0.2$; upper-level
optimization adds Gaussian action noise of standard deviation $0.01$. Both terms
are disabled on FourRoom, AntMaze and DROID. On DROID, the initial Gaussians are
$1.5$ times wider, and candidates at every level, level 1 included, are clipped to
the prior mean $\pm 2$ standard deviations instead of the percentiles.

\paragraph{Replanning and horizons.} All models execute two level-1 action blocks
($10$ environment steps) before replanning from scratch, truncating the final
commit at the episode budget. Level 1 advances $5$ environment steps per predicted
step, and each upper level doubles that interval. For LeWM and the top H-JEPA
level, the horizon decreases from the configured $H$ in
Table~\ref{tab:impl-planning} to one across the episode's replanning calls. Each
lower level holds horizon two until the level above reaches one, then decreases
from its own configured $H$ to one over the remaining calls. An upper level at
horizon one is skipped. FourRoom and AntMaze retain the goal cost in the highest
skipped level's latent space; Push-T and Cube use the active level's own goal
embedding.

\paragraph{Reported settings.} Table~\ref{tab:impl-planning} gives the operating
points in the depth comparison. On the compute-sweep tasks
(Fig.~\ref{fig:depth-planning}, top), LeWM and H-JEPA select the candidate
allocation with the highest observed mean success rate under $100$ mean planner
TFLOPs per episode, breaking ties by lower compute. Learning rates and horizons
were fixed from earlier experiments, before the compute sweeps and final
evaluation. The compute sweeps vary only candidate trajectory counts, with all
other planner settings held fixed. These settings are then evaluated without further tuning on the
separate $50$-task set used for Fig.~\ref{fig:depth-planning} (bottom) and
Tables~\ref{tab:cost-ladder} and~\ref{tab:cost-ladder-appendix}, which report
success on this final evaluation set. Results report mean and SE across
three paired model and planner seeds: $42$, $43$, and $44$.

\begin{table}[ht]
\centering
\small
\setlength{\tabcolsep}{3.5pt}   %
\caption{\textbf{Planning settings for the simulated depth comparison.}
$S$: candidate sequences; $\eta$: AdamW learning rate; $H$: configured horizon
in predicted steps of that level, used by the schedule described above. Tuples
list levels from lowest to highest, except the H-JEPA learning-rate row, which
lists level one and the rate shared by all upper levels. Episode budgets are in
environment steps. Cube's horizon-one upper solvers are skipped, so their listed
candidate counts are unused.}
\label{tab:impl-planning}
\begin{tabular}{lcccc}
\toprule
 & FourRoom & AntMaze & Push-T & Cube \\
\midrule
Episode budget & $\PlanRoomBudget$ & $\PlanAntBudget$ & $\PlanPushBudget$ & $\PlanCubeBudget$ \\
$\eta$, LeWM & $\PlanRoomFlatLR$ & $\PlanAntFlatLR$ & $\PlanPushFlatLR$ & $\PlanCubeFlatLR$ \\
$\eta$, H-JEPA (L1, upper) & $(\PlanRoomHierLR)$ & $(\PlanAntHierLR)$ & $(\PlanPushHierLR)$ & $(\PlanCubeHierLR)$ \\
\midrule
$H$, LeWM & $\PlanRoomHorizonsOne$ & $\PlanAntHorizonsOne$ & $\PlanPushHorizonsOne$ & $\PlanCubeHorizonsOne$ \\
$H$, two-level H-JEPA & $(\PlanRoomHorizonsTwo)$ & $(\PlanAntHorizonsTwo)$ & $(\PlanPushHorizonsTwo)$ & $(\PlanCubeHorizonsTwo)$ \\
$H$, three-level H-JEPA & $(\PlanRoomHorizonsThree)$ & $(\PlanAntHorizonsThree)$ & $(\PlanPushHorizonsThree)$ & $(\PlanCubeHorizonsThree)$ \\
$H$, four-level H-JEPA & $(\PlanRoomHorizonsFour)$ & $(\PlanAntHorizonsFour)$ & $(\PlanPushHorizonsFour)$ & $(\PlanCubeHorizonsFour)$ \\
\midrule
$S$, LeWM & $\PlanRoomSamplesOne$ & $\PlanAntSamplesOne$ & $\PlanPushSamplesOne$ & $\PlanCubeSamplesOne$ \\
$S$, two-level H-JEPA & $(\PlanRoomSamplesTwo)$ & $(\PlanAntSamplesTwo)$ & $(\PlanPushSamplesTwo)$ & $(\PlanCubeSamplesTwo)$ \\
$S$, three-level H-JEPA & $(\PlanRoomSamplesThree)$ & $(\PlanAntSamplesThree)$ & $(\PlanPushSamplesThree)$ & $(\PlanCubeSamplesThree)$ \\
$S$, four-level H-JEPA & $(\PlanRoomSamplesFour)$ & $(\PlanAntSamplesFour)$ & $(\PlanPushSamplesFour)$ & $(\PlanCubeSamplesFour)$ \\
\bottomrule
\end{tabular}
\end{table}

Plans are executed under a \emph{commit schedule}: the agent replans
at environment times $t_0 = 0 < t_1 < t_2 < \cdots$; at phase $j$ it encodes the
observation at $t_j$, runs the top-down solve of
Eqs.~\eqref{eq:top-plan}--\eqref{eq:lower-plan} from that state, executes the first
$\xi_j = t_{j+1} - t_j \le H_1$ actions of $a^{(1),*}_{0:H_1-1}$, and discards the rest.
At planning time every level encodes only the current observation and the goal, one frame each,
and $\mathrm{Roll}^{(\ell)}$ conditions each predictor step on the latest state alone, observed or
predicted, so the planning context is $c_\ell=1$ at every level and in every environment. Training
covers this regime: under the causal mask, the first position of each training context attends
only to itself.
The schedule need not be uniform; the remaining episode budget truncates the last phase.
Committing one command per phase, $\xi_j \equiv 1$, is fully closed-loop control;
committing an entire plan in one phase, $\xi_0 = H_1$, is open-loop single-shot execution.

DROID is offline and open-loop: one plan covers the whole window, so the commit width equals
the horizon and nothing is replanned. It is scored on the $16$ clips of
\S\ref{app:protocol} with three planner seeds and three training seeds.
\Cref{tab:impl-planning-real} gives its settings.

\begin{table}[ht]
\centering
\small
\setlength{\tabcolsep}{3.5pt}
\caption{\textbf{Planning settings on DROID.}
$S$: candidate sequences; $\eta$: AdamW
learning rate; $H$: configured horizon in predicted steps of that level. Tuples list levels from lowest to highest. Every level runs
AdamW for at most $90$ iterations with early stopping when the best cost stalls; the
cross-level action cost is $0$ and the subgoal weight $\beta$ of
Eq.~\eqref{eq:lower-plan} is $0.5$, for both two-level arms. Planning is open-loop: the whole level-1 plan is executed, so the commit
width equals the level-1 horizon. Each model keeps one $\eta$ and these settings at every
planner budget: the compute-Pareto of Fig.~\ref{fig:compute-pareto-real} varies only the number
of candidates, chosen independently per level for the two-level arms. The listed $S$ is the
operating point of each bar in Fig.~\ref{fig:cls-ladder-droid}.}
\label{tab:impl-planning-real}
\begin{tabular}{lc}
\toprule
 & DROID, $5$\,fps \\
\midrule
Episode budget & $36$ \\
$\eta$, LeWM+IDM & $0.01$ \\
$\eta$, HWM (L1, L2) & $(0.01, 0.3)$ \\
$\eta$, H-JEPA (L1, L2) & $(0.01, 0.3)$ \\
\midrule
$H$, LeWM+IDM & $36$ \\
$H$, HWM & $(36, 12)$ \\
$H$, H-JEPA & $(36, 12)$ \\
\midrule
$S$, LeWM+IDM & $32$ \\
$S$, HWM & $(16, 16)$ \\
$S$, H-JEPA & $(16, 16)$ \\
\bottomrule
\end{tabular}
\end{table}

\section{Temporal-frequency separation and selective abstraction}

\subsection{Frequency separation within a trajectory}
\label{app:entity-traces}

The entity frequency gap of \S\ref{sec:learning-abstractions} is a dataset-level statistic: one
spectral centroid per entity, pooled over episodes, and their fast/slow ratio.
Fig.~\ref{fig:entity-traces} shows the raw behaviour that statistic summarizes on a
single AntMaze training episode, so the reader can verify the gap by eye rather than
trust the aggregate. The episode is drawn from the same dataset and decimated to the
same level-1 frame rate (frameskip $5$) that the centroids in
Fig.~\ref{fig:selective-bars} are computed from; it is the episode among the first $64$
whose own fast/slow centroid ratio is closest to the dataset value, so it is
representative rather than the most striking case (per-episode ratios have median
$6.4\times$ and interquartile range $4.9$--$7.8\times$).

The entities are AntMaze's probe targets: the slow entity is the global position
($xy$, two dimensions) and the fast entity is the body state. For legibility we plot
the body's $13$ position-like dimensions (torso height, root orientation, and the
eight hip and ankle joint angles) and omit the $14$ joint- and root-velocity
dimensions, which oscillate at least as fast and would fill the panel; the episode's
gap is $8.3\times$ with either choice. Because the entities have different units, every
dimension is normalized to zero mean and unit variance over the episode,
making the spectral centroid amplitude-invariant.

\begin{figure}[t]
\centering
\begin{subfigure}{\linewidth}
\centering
\phantomsubcaption\label{fig:entity-traces}
\begin{tikzpicture}
\node[inner sep=0] (panel) {\includegraphics[width=\linewidth]{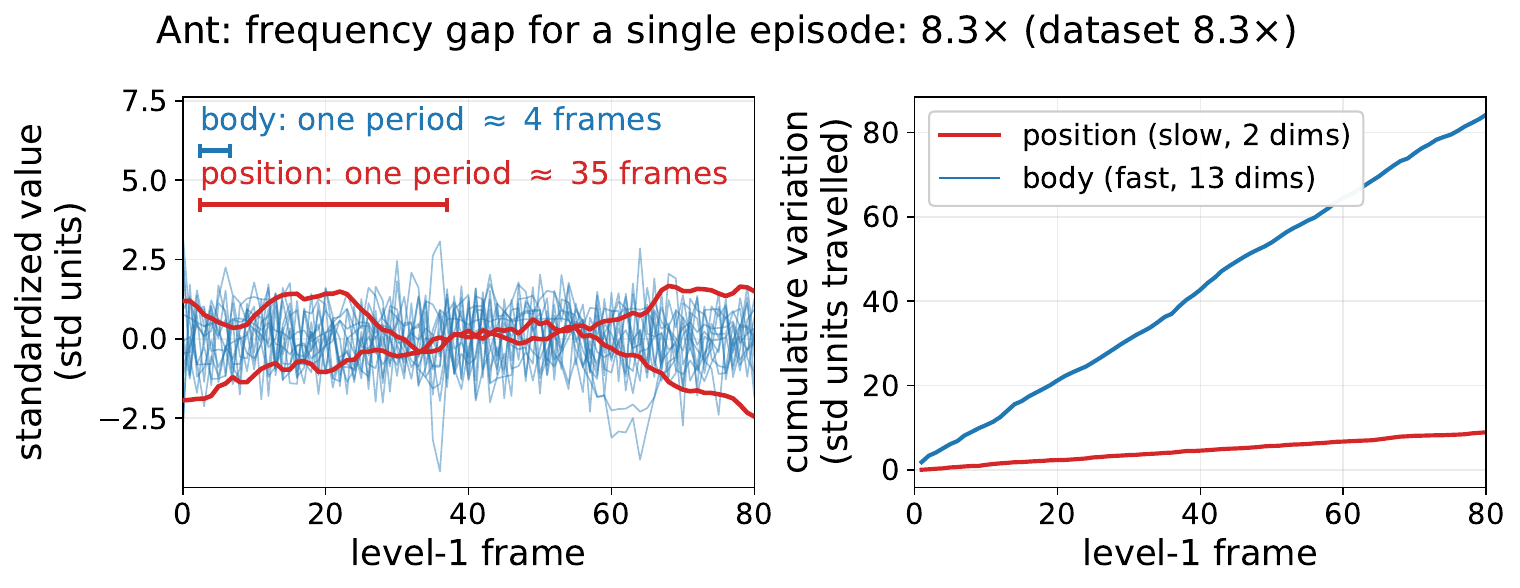}};
\node[anchor=north west, font=\small] at (panel.north west) {(a)};
\end{tikzpicture}
\end{subfigure}

\vspace{0.5em}

\begin{subfigure}{\linewidth}
\centering
\phantomsubcaption\label{fig:selective-bars}
\begin{tikzpicture}
\node[inner sep=0] (panel) {\includegraphics[width=\linewidth]{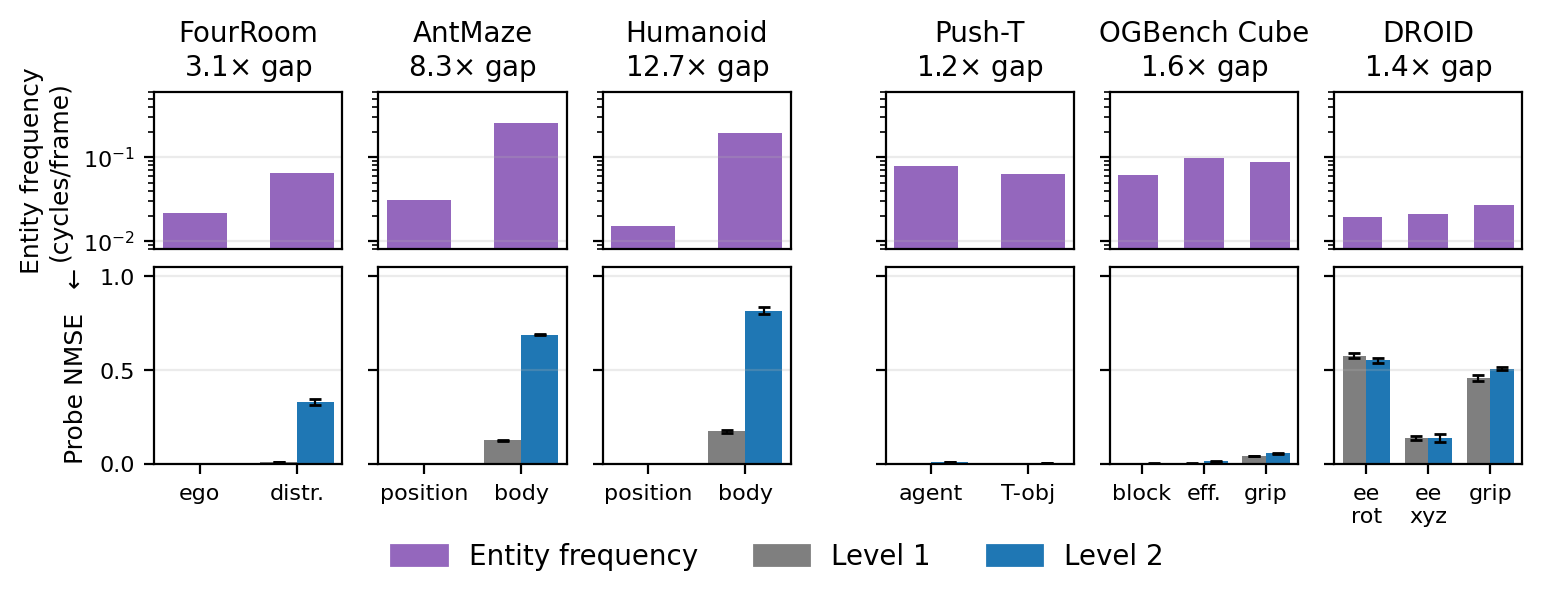}};
\node[anchor=north west, font=\small] at (panel.north west) {(b)};
\end{tikzpicture}
\end{subfigure}
\caption{\textbf{Temporal-frequency separation and selective abstraction.}
\emph{(a)} One AntMaze episode: normalized global position (red) and $13$
body-state dimensions (blue) over $81$ level-1 frames (left), and cumulative
absolute normalized increments averaged across dimensions (right). Inverse
spectral centroids are about $35$ frames for position and $4$ for body state
($8.3\times$ gap); body state accumulates roughly ten times more variation.
Including the omitted velocity dimensions leaves the frequency gap unchanged.
\emph{(b)} Per-entity spectral centroids in cycles per level-1 frame (top, log
scale) and paired level-1 and level-2 probe macro NMSE (bottom). The left group
shows FourRoom, AntMaze and Humanoid; the right shows Push-T, OGBench Cube and
DROID. Probe heads are matched across levels within each environment and differ
across environments. Error bars are one SE over three training seeds. All entities
in an environment share the same level-2 configuration.}
\label{fig:frequency-abstraction}
\end{figure}

The trajectory illustrates the slow position changes and fast
body-state variation summarized by the dataset-level frequency gap. Frequency
separation alone does not establish a difference in predictability, since fast
periodic motion can remain predictable. The per-entity probes below show how this
separation relates empirically to content retained at level 2 across environments.

\subsection{Per-entity probes across environments}
\label{app:selective-bars}

For each state entity, we compute the spectral centroid of its amplitude-normalized
trajectory stream at the level-1 sampling rate, pooling over episodes. The
entity-frequency gap is the ratio of the largest to smallest entity centroid.
Fig.~\ref{fig:selective-bars} expands Fig.~\ref{fig:selective} with each entity's
centroid and paired level-1 and level-2 probe NMSE, computed from two-level
H-JEPA models. Humanoid, AntMaze, and
FourRoom have gaps of $3.1$--$12.7\times$: body-state or distractor decoding
worsens at level 2 while agent position remains near its level-1 floor.
Push-T, OGBench Cube and DROID have gaps of $1.2$--$1.6\times$ and
show little selective abstraction, with comparable entity recoverability at
both levels. The relevant comparison is
the within-entity change between levels: absolute NMSE differs across environments,
particularly for DROID's natural video.

\section{Data specialization across levels}
\label{sec:dataset-per-level}

We test whether training each hierarchy level on a different type of data
improves planning. We use stagewise training, which freezes each lower level
before training the next, so each level can have its own dataset.

\begin{figure}[t]
\centering
\includegraphics[height=0.262\linewidth]{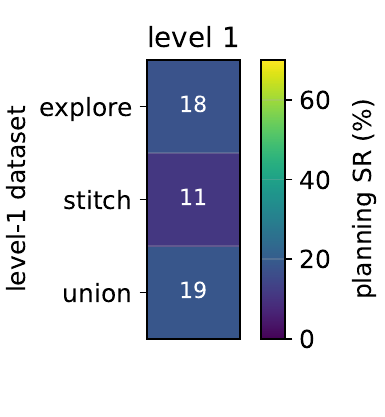}\hfill
\includegraphics[height=0.262\linewidth]{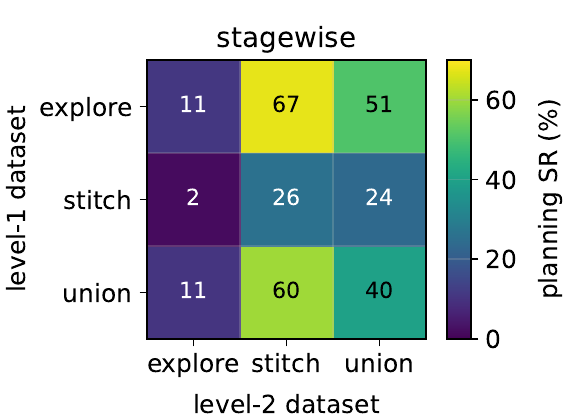}\hfill
\includegraphics[height=0.262\linewidth]{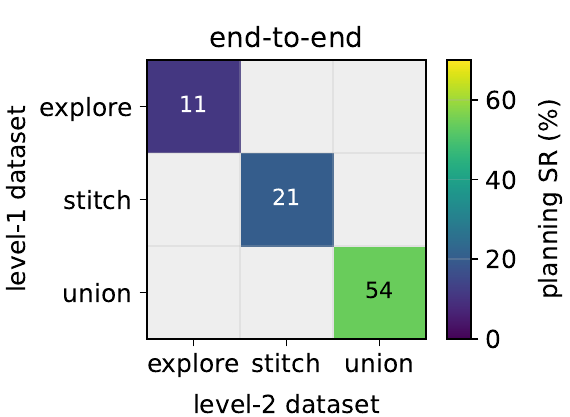}
\caption{\textbf{Planning with different training datasets on Visual AntMaze.}
Planning success rate (GD online, mean over three seeds), using
\texttt{explore}, \texttt{stitch}, or their \texttt{union} for training.
\emph{Left:} flat level-1 planning, with one model trained on each of the three
datasets. \emph{Middle:} two-level hierarchical planning after stagewise
training, covering all nine combinations of level-1 (rows) and level-2
(columns) datasets. \emph{Right:} two-level hierarchical planning after
end-to-end training, with both levels trained on the same dataset in
each run, giving only diagonal entries. Grey cells were not run.
The end-to-end runs are FLOP-matched to the stagewise ones.}
\label{fig:dataset-grid}
\end{figure}

\paragraph{Experimental setup.} We use a two-level hierarchy for ease of analysis,
with Visual AntMaze's two OGBench dataset types, \texttt{explore} and
\texttt{stitch}, and their \texttt{union}. For stagewise training, we vary the
dataset independently at each level; for end-to-end training, both
levels use the same dataset. Fig.~\ref{fig:dataset-grid} compares the
resulting two-level hierarchical planners with flat level-1 planners trained
on each of the three datasets.

\paragraph{Matched data: a second level helps only when the data is temporally
structured.} Along the diagonal, adding an abstract level improves long-horizon
planning on \texttt{stitch} and \texttt{union} (end-to-end hierarchical planning SR
rises to $21$ from the single-level $11$ on \texttt{stitch} and to $54$ from $19$ on
\texttt{union}) but not on \texttt{explore}, where neither two-level model
(\texttt{stagewise} $11$, end-to-end $11$) beats the single-level baseline ($18$).
We read this as the level-2 model being unable to form a useful temporal
abstraction from purely exploratory rollouts, whose trajectories lack the temporal
correlation the abstract level relies on; hierarchy then brings no performance gain.

\textbf{Mismatched data: the two levels prefer different data.} The pattern in the
\texttt{stagewise} panel is that the low-level model benefits from broad,
weakly-correlated exploration while the abstract level benefits from more
temporally coherent, goal-directed trajectories. The best configuration anywhere in
Fig.~\ref{fig:dataset-grid} is exactly this pairing (\texttt{explore} at
level 1 and \texttt{stitch} at level 2) reaching $67$, well above every
matched-dataset cell of either procedure, including end-to-end on
\texttt{union}/\texttt{union} ($54$) and stagewise on \texttt{union}/\texttt{union}
($40$), while the reverse assignment (\texttt{stitch} at level 1, \texttt{explore}
at level 2) collapses to $2$. It is the \emph{pairing} (exploratory data low,
coherent data high) not either dataset on its own, that drives the gain. This is
a practical case for stagewise training: freezing the lower-level model lets us
specialize the upper level's data without changing the lower level's learned
dynamics, a choice the default joint procedure does not offer.

\section{Planning-cost ablations across environments}
\label{sec:negative-control}

Table~\ref{tab:cost-ladder-appendix} extends the AntMaze cost ablation of
Table~\ref{tab:cost-ladder} to FourRoomDistractors, OGBench Cube and Push-T,
on the tasks of Fig.~\ref{fig:depth-planning} (bottom). Within each model,
we hold the level-1 rollout dynamics and planner settings fixed and vary only
the latent space used to score the goal cost.

\paragraph{FourRoom: gains from projected costs.}
As on AntMaze, at least one upper-level cost improves mean success over the
native level-1 cost at every tested depth. The best projection also exceeds
the separately trained LeWM baseline at three and four levels.
The deepest cost is not always best: the four-level model has its highest
mean success with the level-2 projection, although the projected-cost error
bars overlap.

\paragraph{Cube and Push-T: no clear gain from projected costs.}
These environments do not exhibit the selective abstraction observed on
AntMaze and FourRoom. On Cube, the level-2 projection lowers success in the
two-level model. At three and four levels it gives similar success to the native
cost, while the level-3 and level-4 projections roughly halve success.
On Push-T, projected and native costs give similar success at two and three
levels; the four-level model has near-zero success under every cost and with
hierarchical planning. These results show that an upper-level cost need not
improve planning; they do not establish that selective abstraction alone
causes the gains in the other environments.

\begin{table}[ht]
\centering
\caption{\textbf{Planning-cost ablations across environments.}
Success (\%, mean $\pm$ SE, three seeds), using the same protocol as
Table~\ref{tab:cost-ladder}. Within each H-JEPA row, only the cost space varies
across the four level-1 planning columns; LeWM and hierarchical H-JEPA are
references. For FourRoom, bold marks the highest mean among cost spaces.
FourRoom gains from projected costs, while Cube and Push-T show no clear benefit.}
\label{tab:cost-ladder-appendix}
\footnotesize
\setlength{\tabcolsep}{3pt}
\begin{tabular}{llccccc}
\toprule
& & \multicolumn{4}{c}{\shortstack{Level-1 planning:\\vary cost space}} & \shortstack{Hierarchical\\planner} \\
\cmidrule(lr){3-6}\cmidrule(lr){7-7}
Environment & Depth & Native L1 $\uparrow$ & L2 proj. $\uparrow$ & L3 proj. $\uparrow$ & L4 proj. $\uparrow$ & H-JEPA $\uparrow$ \\
\midrule
\multirow{4}{*}{\shortstack[l]{FourRoom\\Distractors}}
  & LeWM (flat)     & \DistrDLLewm & & & & \\
  & 2 levels & \DistrDLTwoNative   & {\boldmath\DistrDLTwoLtwo}   & & & \DistrDLTwoHier \\
  & 3 levels & \DistrDLThreeNative & \DistrDLThreeLtwo & {\boldmath\DistrDLThreeLthree} & & \DistrDLThreeHier \\
  & 4 levels & \DistrDLFourNative  & {\boldmath\DistrDLFourLtwo}  & \DistrDLFourLthree & \DistrDLFourLfour & \DistrDLFourHier \\
\midrule
\multirow{4}{*}{OGBench Cube}
  & LeWM (flat)     & \CubeDLLewm & & & & \\
  & 2 levels & \CubeDLTwoNative   & \CubeDLTwoLtwo   & & & \CubeDLTwoHier \\
  & 3 levels & \CubeDLThreeNative & \CubeDLThreeLtwo & \CubeDLThreeLthree & & \CubeDLThreeHier \\
  & 4 levels & \CubeDLFourNative  & \CubeDLFourLtwo  & \CubeDLFourLthree & \CubeDLFourLfour & \CubeDLFourHier \\
\midrule
\multirow{4}{*}{Push-T}
  & LeWM (flat)     & \PushtDLLewm & & & & \\
  & 2 levels & \PushtDLTwoNative   & \PushtDLTwoLtwo   & & & \PushtDLTwoHier \\
  & 3 levels & \PushtDLThreeNative & \PushtDLThreeLtwo & \PushtDLThreeLthree & & \PushtDLThreeHier \\
  & 4 levels & \PushtDLFourNative  & \PushtDLFourLtwo  & \PushtDLFourLthree & \PushtDLFourLfour & \PushtDLFourHier \\
\bottomrule
\end{tabular}
\end{table}

\section{Prediction-cost monotonicity}
\label{app:subgoal-monotonicity}

Table~\ref{tab:tmono} expands the comparison in
\S\ref{sec:cost-surface}. It reports goal-cost monotonicity over full
held-out trajectories and subgoal-cost monotonicity within shorter segments,
using the two-level HWM and H-JEPA models in the four simulated benchmarks. These measurements characterize
costs along expert trajectories rather than the success of optimized plans.

\begin{table}[H]
\centering
\caption{\textbf{Prediction-cost monotonicity} ($-\mathrm{Spearman}$ of
$\|z^{(\ell)}-g^{(\ell)}\|_2^2$ against time; higher is better; mean over three training
seeds, held-out episodes). \emph{L2} columns score each model's level-2 encoding against the final
goal; \emph{subgoal} scores a level-1 rollout, re-anchored to the true state at
each subgoal (bounded horizon), against the level-2 subgoals (per-segment).}
\label{tab:tmono}
\begin{tabular}{lccccc}
\toprule
& & \multicolumn{2}{c}{HWM} & \multicolumn{2}{c}{H-JEPA} \\
\cmidrule(lr){3-4}\cmidrule(lr){5-6}
Environment & Level-1 $\uparrow$ & Level-2 $\uparrow$ & subgoal $\uparrow$ & Level-2 $\uparrow$ & subgoal $\uparrow$ \\
\midrule
Visual AntMaze & \AntTmonoLone   & \AntTmonoHwmLtwo   & \AntTmonoHwmProj   & \textbf{\AntTmonoHjepaLtwo}   & \textbf{\AntTmonoHjepaProj} \\
Push-T         & \PushtTmonoLone & \PushtTmonoHwmLtwo & \PushtTmonoHwmProj & \textbf{\PushtTmonoHjepaLtwo} & \textbf{\PushtTmonoHjepaProj} \\
OGBench Cube   & \OgbTmonoLone   & \OgbTmonoHwmLtwo   & \OgbTmonoHwmProj   & \textbf{\OgbTmonoHjepaLtwo}   & \textbf{\OgbTmonoHjepaProj} \\
FourRoom Distractors & \DistrTmonoLone & \DistrTmonoHwmLtwo & \DistrTmonoHwmProj & \textbf{\DistrTmonoHjepaLtwo} & \textbf{\DistrTmonoHjepaProj} \\
\bottomrule
\end{tabular}
\end{table}

\section{DROID: corpus, planning evaluation and offline metric}
\label{app:droid}

DROID exposes no simulator, at least none that runs without the ray-tracing cores IsaacSim
requires~\citep{nvidia2025isaaclabgpuacceleratedsimulation}, so every DROID number in this
paper is offline: plans are scored against logged expert trajectories. This section
covers the evaluation protocol (\S\ref{app:protocol}), the metric and what it measures
(\S\ref{app:frechet}), the decoded plans (\S\ref{app:decoded}), and the re-encoded corpus
every DROID model is trained on (\S\ref{app:droid-pipeline}).

\subsection{Planning-evaluation protocol}
\label{app:protocol}

\paragraph{Clips.} The headline set is the \emph{16-clip evaluation set}: 16 pick-and-place
clips, with good lighting, good visibility of the manipulated objects from the left and
right cameras, and big enough manipulated objects. Once chosen, the exact frame indices of
each clip's evaluation slice are frozen so that every run scores the same initial and goal
frame and the same ground-truth actions. For each of the 16 clips, contact waypoints
$[\textit{init},\textit{grasp},\textit{pre-release},\textit{place}]$ are read off the
gripper-closure trace (grasp $=$ first frame with grip $>0.85\times$max), and each
clip is a fixed-length strided window \emph{centered on the grasp}. Episodes whose
window does not fit inside the video are dropped rather than clamped. Centering on
the grasp makes the evaluation a probe of non-greedy behavior: a plan that drives straight
to the deposit matches the ground-truth net displacement but misses the descent to the
object and its transport to the target. The metric of \S\ref{app:frechet} therefore scores
the shape of the whole path. Each window spans
$108$ recorded frames, $\sim$$7.2$\,s, with the goal as its last frame. A clip is read at
the frame rate its model was trained at, so the level-1 horizon is that span divided by the
stride: $H_1 = 36$ at $5$\,fps, the setting this paper trains and evaluates at.

\paragraph{Frame rates.} DROID is recorded at $15$\,Hz~\citep{khazatsky2024droid},
but its video containers are \emph{tagged} $60$\,fps, so a loader that derives its temporal
stride from the tag samples at a quarter of its configured rate. The frozen-encoder
baselines~\citep{terver2026whatdrives,assran2025vjepa2} are configured at $4$\,fps with such
a loader, which corresponds to $1$ frame per second of recorded time, and we evaluate them
at that native rate on the matching manifest (15 clips; one clip's window does not fit
at $1$\,fps).

\paragraph{Leak control.} Every DROID model in this paper is trained on one corpus: the
re-encoded corpus of \S\ref{app:droid-pipeline}, with the $16$ evaluation clips of
Fig.~\ref{fig:decoded-plans} removed. Plan-evaluation fidelity is therefore held out
throughout. Only those $16$ clips are removed, so the auxiliary unroll and light-eval
metrics remain train-set metrics.

\paragraph{Task and action space.} Actions are 7-dimensional end-effector delta-poses
(Cartesian $xyz$, roll-pitch-yaw, gripper). The goal is specified as a single RGB
frame, encoded by the model's own encoder; the planning cost is a Euclidean distance in
that latent space. All metrics use the first three (translation) dimensions; the
7-dimensional variant is dominated by gripper nuisance.

\subsection{The Fr\'echet fidelity offline metric}
\label{app:frechet}

Fr\'echet fidelity is defined in Eq.~\eqref{eq:skill}; the paths it compares are
$p_{1:T+1}$ and $\tilde{p}_{1:T+1}$, the origin-prepended cumulative $xyz$ paths of the
planned and ground-truth delta-pose sequences. Each planner seed in $\{42,43,44\}$ is reduced to
one scalar by averaging over clips; we report the median over
(training seed $\times$ planner seed) cells, because the per-clip mean is in absolute
meters and is therefore dominated by large-motion clips.

We need a path metric. The cumulative $xyz$ action trajectory error used in prior
work~\citep{terver2026whatdrives,NWM} compares net
displacements, $|\sum \text{planned} - \sum \text{ground truth}|$ per dimension, and is
therefore blind to the detour the grasp-centered protocol tests: a
straight-line plan to the deposit scores as well as one that descends to the object and
carries it. Within path metrics the choice was made empirically: Fr\'echet agrees with dynamic time warping
and with an arc-length-split reach/transport score.
We calibrate the metric in two simulated environments where the outcome is observable:
OGBench Cube, whose outcome is \emph{binary} (grasp success), and RoboCasa
CloseDrawer~\citep{nasiriany2024robocasa}, a kitchen task whose scenes vary between episodes
like DROID's and whose outcome is \emph{continuous} (graded drawer-closure progress). On
RoboCasa, we vary the planner on one frozen checkpoint to spread the outcomes, and read how
well the metric tracks them.

\paragraph{Against a binary outcome it is a necessary condition.} On OGBench Cube
(Fig.~\ref{fig:frechet-cube}) no episode succeeds below a floor of Fr\'echet fidelity, and
clearing that floor does not say which episodes will succeed. A grasp either closes or it
does not, and no continuous measure of path agreement separates the episodes that complete
it from those that come close. This is expected of any path metric.

\paragraph{Against a continuous outcome it ranks progress.} Scored against graded drawer
closure on RoboCasa CloseDrawer (Fig.~\ref{fig:robocasa-frechet}), the
within-configuration rank correlation is positive and moderate, and it survives controlling
for executed path length and for per-episode difficulty. The metric measures how far along
the expert path a plan got.

Together, the two calibrations set how we read the metric. The end of the path turns
progress into a completed task, so a plan can follow most of the expert path and still not
finish. On Cube
this shows up as a small offline advantage for the hierarchy alongside a large advantage in
grasp success: the Fr\'echet gap understates the outcome gap. On DROID, where no outcome
exists, we therefore read a gap in Fr\'echet fidelity as evidence that one planner is better,
not as a measure of how much better.

\begin{figure}[t]
\centering
\includegraphics[width=0.95\linewidth]{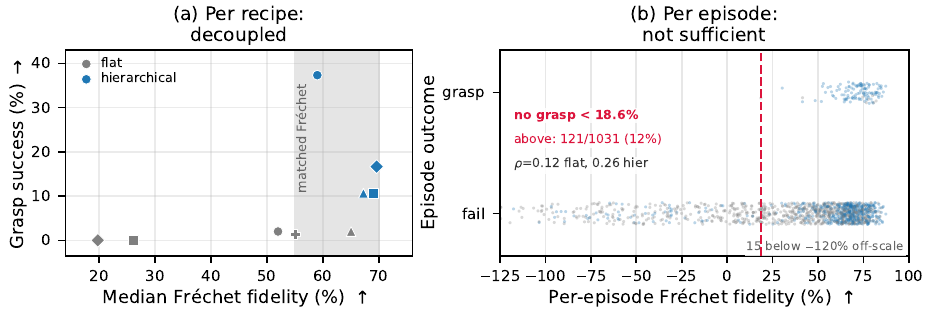}
\caption{\textbf{On OGBench Cube, where the outcome \emph{is} observable, Fr\'echet fidelity
separates the episodes that can succeed from those that cannot, but not the ones that do.} One frozen checkpoint,
$9$ configuration groups $\times$ $3$ planner seeds $\times$ $50$ episodes. \emph{(a)} One
point per planner configuration; in the shaded band the flat and hierarchical planners have
equal offline fidelity and opposite grasp outcomes. \emph{(b)} Every episode: none succeeds
below a fidelity threshold, and most above it still fail.
}
\label{fig:frechet-cube}
\end{figure}

\begin{figure}[t]
\centering
\includegraphics[width=0.80\linewidth]{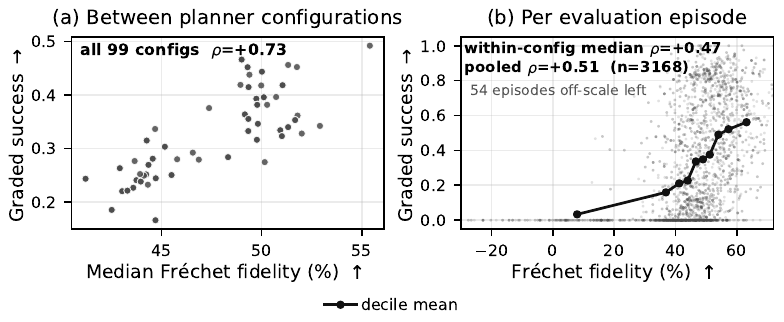}
\caption{\textbf{Against a continuous outcome the same offline metric ranks degree of
progress.} RoboCasa CloseDrawer, one frozen checkpoint, $99$ planner configurations
$\times$ $32$ held-out episodes, budget $200$ steps; the configurations are varied to spread
the outcomes, not tuned for task performance. $x$ is Fr\'echet fidelity, computed
exactly as on DROID; $y$ is graded drawer-closure success, a continuous outcome DROID
cannot report. \emph{(a)} One point per planner configuration, median fidelity against mean
graded success. \emph{(b)} Every episode, with the mean of each fidelity decile joined. The
claim rests on the \emph{within-configuration} rank correlation, which holds episode
difficulty and planner settings fixed; the pooled and between-configuration values are
inflated by a between-configuration mean shift.
}
\label{fig:robocasa-frechet}
\end{figure}

\subsection{Decoded plans}
\label{app:decoded}

The Fr\'echet numbers in \S\ref{sec:droid} summarize a plan by one scalar. Since DROID
plans cannot be executed, we decode them in Fig.~\ref{fig:decoded-plans}: the visual decoder is applied to the
level-1 latents the planner rolls out under its own committed commands, and to the
level-2 macro-plan that supplies its subgoals.

Two things are visible that the scalar does not carry. First, the imagined rollouts move
the gripper \emph{down toward the manipulated object} rather than straight at the goal
frame: the non-greedy behavior the evaluation clips are chosen to elicit. This holds even
for the clip scoring at the zero-action floor: a fidelity of zero means the planned
path did not follow the expert's, not that the plan had no idea where the object was. Second, the
level-2 macro-plan is a coarse but coherent version of the same trajectory, not a
degenerate one, as the coarse cost of Eq.~\eqref{eq:top-plan}, which scores only the final
predicted state, requires.
All 16 clips are shown, ordered best to worst, so both observations can be checked
against the full fidelity distribution rather than against a selection. The level-2
macro-plan advances one frame for every three of level 1, so it shares the level-1
strip rather than occupying a row of its own that would be two-thirds empty.

\begin{figure}[p]
\centering
\includegraphics[width=\linewidth]{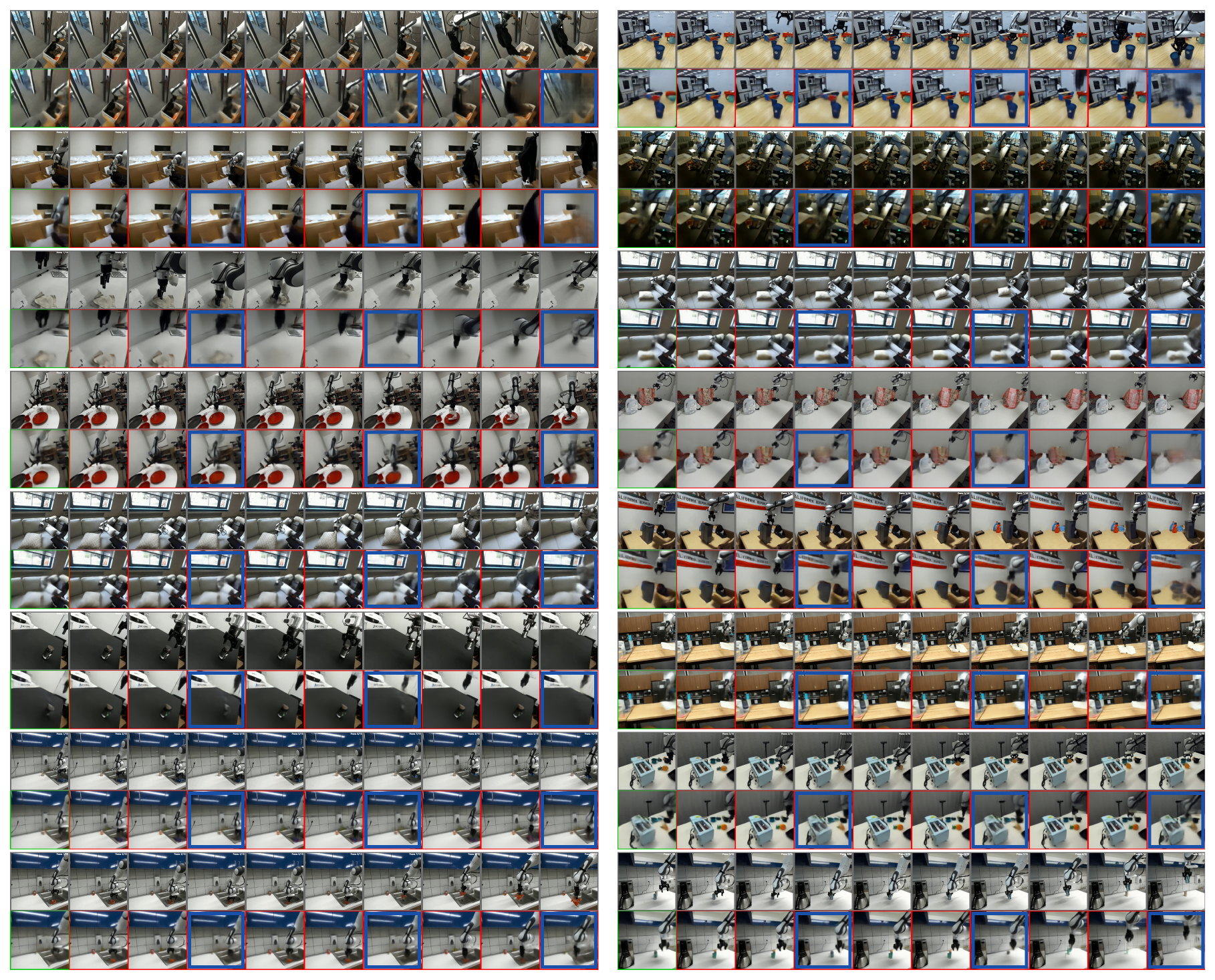}
\caption{\textbf{What the DROID planner imagines, on every clip of the evaluation set.}
Two-level stagewise model, held-out, training seed 1; all $16$ clips ordered best
to worst by Fr\'echet fidelity, down the left column then the right, $+75\%$ down
to $+0\%$ at the zero-action floor. Each clip is two strips: above, the logged expert clip;
below, the decoded plan, which is the level-1 rollout under the planner's committed commands
with the level-2 macro-plan interleaved into it at every third column (temporal stride
$3$), ringed in blue. Columns are timesteps. The leftmost is the context frame the planner
is given (green); the rest are imagined (red), and blur along the rollout, which is the
decoder's fidelity limit and not a planning failure.
}
\label{fig:decoded-plans}
\end{figure}

\subsection{The re-encoded training corpus}
\label{app:droid-pipeline}

Raw DROID is impractical to train on repeatedly: the loader would touch ${\sim}5.5$\,TB
and H.264 decoding, not the GPU, sets the epoch time. Every DROID result in this paper
therefore runs on a re-encoded copy: each of the three camera views downscaled to
$256\times256$ (short-side center crop, then scale) with \texttt{libx264} at CRF $18$, each
per-episode trajectory repacked into a flat \texttt{.npz}, and the stereo files our loader
never opens dropped. We keep every frame and its timestamp rather than decimating to the training rate, which
would collapse all stride phases into one.
The build is $127.6$\,GB for $74{,}530$ episodes, a ${\sim}45\times$ reduction that decodes
$17.9\times$ faster.
We measured the cost of the lossy re-encode: trajectories and assembled loader
samples come back bitwise identical, the video costs a median PSNR of $33.75$\,dB, and a
flat configuration retrained on the re-encoded corpus but scored on the \emph{original}
clips matches raw within noise.
We release the corpus, its manifests and the build script. DROID is
CC-BY 4.0~\citep{khazatsky2024droid}, which permits redistributing an adapted version with
credit and a statement of the change; we mark the re-encode as such.

\subsection{Future work}
\label{app:droid-future}

Every model here encodes one camera at a time: both settings pool the left and right
exterior views and sample one per training window. DROID's wrist camera is never used. Precise
manipulation is where that costs most, since the gripper and the object it closes on are
small from an exterior angle and often self-occluded, while a wrist view sees them
directly. Encoding all views jointly, including the wrist view, is the clearest route to
better performance; we leave it for future work.
Extending the method to videos without action labels would remove the IDM term, our lever
against slow-feature collapse. \Cref{app:slowfeature} suggests that SIGReg alone, applied so
that it repels the representations of different timesteps, cannot fully replace it.

\section{Anti-collapse: SIGReg, inverse dynamics and detachment}
\label{app:epps-pulley}
\label{app:anticollapse}

Given $Z^{(\ell)} \in \mathbb{R}^{n \times D}$, a matrix of $n$ level-$\ell$
embeddings of width $D$, SIGReg samples $M$ unit directions
$u_m \in \mathbb{S}^{D-1}$ and projects the embeddings to one-dimensional
samples $h_m = Z^{(\ell)}u_m$. It averages the Epps--Pulley normality
statistic~\citep{epps1983test,balestriero2025lejepa} across these projections:
\begin{equation}
\label{eq:sigreg}
\mathrm{SIGReg}\!\left(Z^{(\ell)}\right)
= \frac{1}{M}\sum_{m=1}^{M} \mathcal{T}_{\mathrm{EP}}(h_m),
\end{equation}
where $\mathcal{T}_{\mathrm{EP}}$ compares the empirical characteristic
function of $h_m$ to that of a standard Gaussian.
The statistic, its Gaussian weighting,
and the Cram\'er--Wold argument~\citep{cramer1936theorems} that reduces matching $\mathcal{N}(0,I_D)$ to matching every one-dimensional projection are those of \citet{balestriero2025lejepa}; the quadrature is that of
the implementation released with LeWM~\citep{maes2026lewm}, which exploits the even integrand to
work on the half-domain $[0,3]$ with $17$ trapezoidal knots, at $\mathcal{O}(n)$ cost in the $n$
samples entering one test.
With one embedding per frame, the tensor entering the test is $(T, N, D)$: one normality
test per timestep, over the $N$ sequences of the batch. Other terms also prevent
collapse: \S\ref{app:anticollapse-terms} separates what
it, the inverse-dynamics term and prediction-target detachment~\cite{ponce2026dualperspectives} each contribute, and
\S\ref{app:crossval} cross-validates the coefficients of the two mechanisms that have a clear effect, IDM and SIGReg.
\S\ref{app:sigreg-ddp} gives technical guidelines to implement SIGReg across devices.

\subsection{What SIGReg, IDM and detachment each contribute}
\label{app:anticollapse-terms}

In \S\ref{sec:droid}, we show an inverse-dynamics term is required on DROID. We test that with
a full grid over three candidate anti-collapse mechanisms, the SIGReg
weight, the IDM weight and prediction-target detachment (Fig.~\ref{fig:anticollapse}).

\textbf{Representation rank against planning fidelity.} SIGReg is what maintains the visual effective rank~\citep{RankMe} of the representation high enough: without it the
representation collapses at every IDM weight, with or without detachment, to an effective
rank between $1$ and $14$ out of $384$. Our models trained without SIGReg reach a planning performance below the ``do-nothing''
floor. Yet, visual representations with non-collapsed visual effective rank do not necessarily mean the JEPA world model will have high Fr\'echet fidelity after planning. With SIGReg on, every
model keeps a high rank, but models without an IDM term still plan close to the floor. Adding
the IDM term gives high planning fidelity, best at weight $100$.

\textbf{When detachment matters.} On this grid it does not change the outcome. With SIGReg
off, both variants collapse. With SIGReg on and an IDM term, the two variants differ by less
than the spread across planner seeds. We therefore do not use target detachment in our
headline setup (\cref{tab:hparams}).

\begin{figure}[t]
\centering
\includegraphics[width=\linewidth]{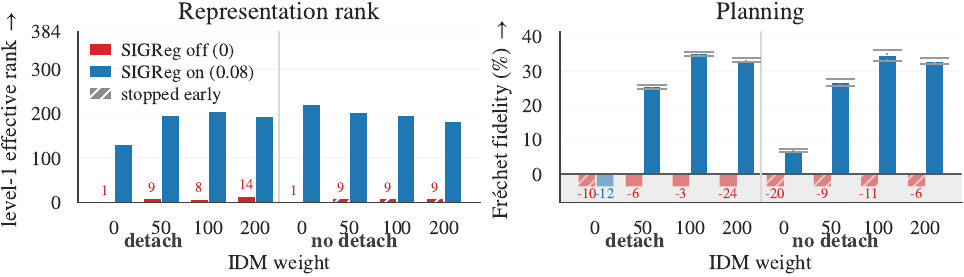}
\caption{\textbf{SIGReg restores representation rank; the inverse-dynamics term is what
turns rank into planning fidelity.} Full grid around the level-1 DROID configuration of
\cref{tab:hparams}: SIGReg $\in\{0,0.08\}$ $\times$ IDM weight $\in\{0,50,100,200\}$
$\times$ prediction-target detachment, $16$ cells, one training seed each, read at epoch
$100$. Five SIGReg-off cells (hatched) were stopped once collapsed and are read at their
last epoch.
\emph{Left:} level-1 effective rank~\citep{RankMe} of the \texttt{CLS} embedding, max $384$.
\emph{Right:} Fr\'echet fidelity on the 16-clip set; the grey band
holds sub-floor cells, drawn as fixed-depth stubs annotated with their true value. Each
planning bar is the mean over planner seeds $1$, $2$ and $3$ of one trained model, and its
error bar is the standard error (SE) over those three planner seeds.
}
\label{fig:anticollapse}
\end{figure}

\subsection{Cross-validating the SIGReg and inverse-dynamics coefficients}
\label{app:crossval}

In \S\ref{app:anticollapse-terms} we studied which terms the objective needs. Here we tune
their coefficients, the SIGReg coefficient and the inverse-dynamics weight, with a full grid
at level 1 (Fig.~\ref{fig:crossval-l1grid-cls}). We then freeze level 1 and run a grid
at level 2, pairing each level-2 training seed with the level-1 seed it builds on.

At level 1, the two coefficients have to be scaled together. A large inverse-dynamics weight
with a small SIGReg coefficient barely beats the do-nothing floor, and a large SIGReg
coefficient with a small inverse-dynamics weight falls below it. Between these two corners, a wide range of
pairs plans well. Level 2 is much less sensitive: with level 1 frozen, all nine combinations
reach a similar fidelity.

\begin{figure}[htbp]
\centering
\includegraphics[width=0.70\linewidth]{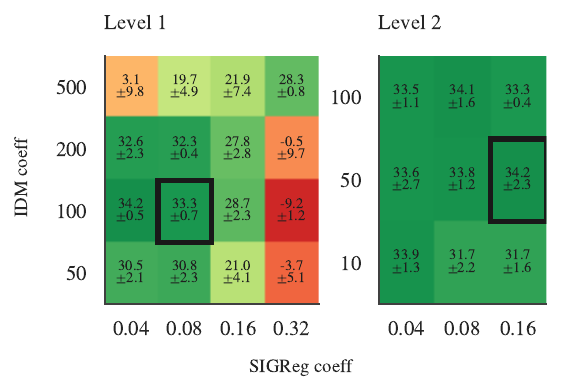}
\caption{\textbf{The SIGReg and inverse-dynamics coefficients are chosen jointly, once per
level.} DROID at $5$\,fps; Fr\'echet fidelity (\%, higher is better), mean and standard
error over three training seeds, each averaged over three planner seeds. \emph{Left}, the
level-1 grid on single-level models; \emph{right}, the level-2 grid on stagewise models with
level 1 frozen. Both use the gradient-based planner of our headline results, and share one
color scale. The heavy outline marks the coefficients we keep; at level 1 they lie within one
standard error of the best cell.
}
\label{fig:crossval-l1grid-cls}
\end{figure}

\subsection{SIGReg under data parallelism}
\label{app:sigreg-ddp}

$\mathcal{T}_{\mathrm{EP}}$ is a statistic of the whole batch, not a sum of per-sample terms. Evaluating it
on each rank's $n_{\mathrm{loc}} = n/W$ samples and averaging the $W$ results gives the mean of $W$ small-sample statistics, which shifts with the number of devices. We synchronize
one level lower. The empirical characteristic function (ECF) underneath the statistic is a mean
over samples, so with equal shards the global ECF is the average of the per-rank ECFs: we
\texttt{all\_reduce} the local real and imaginary ECFs before the quadrature and pass
$n_{\mathrm{loc}}W$ into the sample-count prefactor. Two $M\times17$ tensors cross the wire per
call, one entry per direction and quadrature knot, whatever the batch size, and the statistic
becomes exactly invariant to sharding (\Cref{tab:sigreg-ws}). The remaining $n$ dependence
is deliberate: the prefactor cancels the $1/n$ variance of the ECF, so on target data the
statistic sits at a floor ($1.04$ to $1.37$ over $n = 32$ to $2048$) and on collapsed data it
grows linearly in $n$, holding the anti-collapse gradient per sample constant. Thus $\lambda$
transfers across batch sizes, but not across sample definitions.

\begin{table}[ht]
\centering
\small
\caption{\textbf{Our distributed SIGReg is exactly invariant to how a fixed total batch is
sharded; a single-GPU implementation is not.} One fixed Gaussian matrix ($n=256$, $D=32$,
$M=64$ directions) sharded contiguously across $W$ ranks in \texttt{float64}. The reference is the SIGReg released with LeWM~\citep{maes2026lewm}, formula-identical to ours but with its prefactor set by the local batch and no collective; its code documents itself as single-GPU. The data follow the target, so its $14\%$ drift is launch topology, not signal.}
\label{tab:sigreg-ws}
\begin{tabular}{lccc}
\toprule
& & \multicolumn{2}{c}{$\mathcal{T}_{\mathrm{EP}}$} \\
\cmidrule(l){3-4}
Shards & $\|\nabla\|$ (ours) & ours & single-GPU reference \\
\midrule
$1 \times 256$ & $0.349154178012$ & $1.372619178543$ & $1.372619$ \\
$2 \times 128$ & $0.349154178012$ & $1.372619178543$ & $1.180770$ \\
$4 \times 64$  & $0.349154178012$ & $1.372619178543$ & $1.186233$ \\
\bottomrule
\end{tabular}
\end{table}

Equal shards are assumed throughout, so we drop the last partial batch.

\section{Slow-feature collapse in joint-embedding predictive objectives}
\label{app:slowfeature}

This appendix argues for the inverse-dynamics term of \S\ref{sec:droid}. \citet{sobal2022slowfeatures} exhibit an encoder that copies episode-constant noise and on which a variance--covariance regularizer vanishes. That shows a degenerate encoder attains zero loss, not that the objective \emph{prefers} it; it assumes the identity predictor; and it covers only Gaussian noise. We strengthen all three points; no result assumes a parametric encoder or predictor. We suppress the level index $\ell$ of
\Cref{tab:notation}: $E$, $F$, $I$ are the encoder, predictor and inverse-dynamics head,
$z_t = E(o_t)$, $\mathcal{L}_{\mathrm{pred}}$ and $\mathcal{L}_{\mathrm{idm}}$ the losses of
Eqs.~\eqref{eq:pred-loss} and~\eqref{eq:idm-loss}, $\lambda$, $\gamma$, $\lambda_{\mathrm{sim}}$
the SIGReg, IDM and temporal-similarity weights, and $N$, $T$, $D$ the minibatch shapes of
\Cref{tab:notation}. SIGReg averages $G$ independent $n$-sample normality tests over
$Gn$ minibatch embeddings.

\subsection{Three kinds of content in a sequence of observations}
\label{app:sf-defn}

The observation $o_t = O(\xi_t)$ is a deterministic function of an unobserved latent state $\xi_t = (b, x_t, y_t)$ taking values in a standard Borel space, and $z_t = E(o_t)$ encodes it. The \textbf{background} $b$ is
constant within an episode, varies across episodes (scene, camera, lighting), and is readable from a single frame: $b$ is $\sigma(o_t)$-measurable. The \textbf{foreground} $x_t$ is the agent-controlled content,
$x_{t+1} = \Phi(x_t, a_t, \varepsilon_{t+1})$, where the innovation $\varepsilon_{t+1}$ is a
random input independent of the past that makes the foreground only \emph{partially}
predictable (contact, occlusion). The \textbf{distractors} $y_t$ move but
are action-independent and unpredictable, $y_{t+1} = \Psi(y_t,\zeta_{t+1})$ with their own
innovation $\zeta_{t+1}$ (other agents, shadows), exogenous in the sense of \citet{Efroni2021ProvableRW,pmlr-v162-wang22c}. A distractor that moves predictably (a clock, a fan) has no innovation, and nothing below separates it from the background.

\begin{definition}[Slow-feature collapse]
\label{def:sf-collapse}
An encoder is \emph{slow-feature collapsed} when $z_t$ is measurable with respect to the
episode-invariant $\sigma$-algebra $\mathcal{C} = \sigma(b)$, the events determined by the background alone, so $z_t$ is constant within every episode. Its \emph{degree}
is the within-episode variance share
\begin{equation}
\label{eq:chi}
\chi(E) = \frac{\mathbb{E}_{b}\!\left[\operatorname{tr}\operatorname{Var}_{t}(z_t \mid b)\right]}
{\operatorname{tr}\operatorname{Var}_{b,t}(z_t)},
\end{equation}
the complement of the between-episode share measured in \S\ref{app:varcomp}, and $0$ at exact collapse.
\end{definition}

\subsection{Where the latent variance goes}
\label{app:varcomp}

\begin{wrapfigure}{r}{0.46\textwidth}
\vspace{-7pt}
\centering
\includegraphics[width=\linewidth]{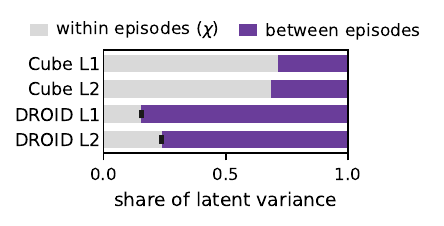}
\caption{\textbf{Scene diversity sends
latent variance between episodes.} Within-episode (the slow-feature degree $\chi$ of
\Cref{def:sf-collapse}) and between-episode shares of the latent variance.
Bars are $\pm$1 standard error (SE) over three training seeds.
}
\label{fig:varcomp}
\vspace{-18pt}
\end{wrapfigure}

\S\ref{sec:droid} claims scene diversity separates DROID from Cube. We measure
it directly in the latents, splitting the total latent variance into a within-episode and a
between-episode part (Fig.~\ref{fig:varcomp}). Neither latent size nor camera site
explains the gap.
The within-episode part is $\chi(E)$ of \Cref{def:sf-collapse}, so the diverse-scene
environments start close to collapse. The cost is capacity, not
reachability: a per-episode mean cancels in a same-episode state-minus-goal difference, but
it leaves little of the latent for the arm. Level 2 barely moves the share, hence the inverse-dynamics term at both levels.

\subsection{No regularizer on the marginal law can exclude it}
\label{app:sf-feasibility}

\begin{proposition}[Feasibility]
\label{prop:sf-feasibility}
Let $\mathcal{R}$ be any functional of the marginal law of $z$, the distribution of a single
embedding $E(o)$ over the pooled corpus irrespective of time, with $\mathcal{R}(\mu_0) = 0$
for some law $\mu_0$ on $\mathbb{R}^D$, and let $\operatorname{law}(b)$ be atomless. Then some $E$ and $F$ attain $\mathcal{L}_{\mathrm{pred}} = 0$ and
$\mathcal{R}(\operatorname{law}(z)) = 0$ simultaneously, with $F$ the identity in its state
argument.
\end{proposition}

\begin{proof}
An atomless law on a standard Borel space admits $U$ with $U(b)\sim\mathrm{Unif}[0,1]$, and $Q$ pushing $\mathrm{Unif}[0,1]$ to $\mu_0$ exists for any $\mu_0$; take $E = Q\circ U\circ b$, well defined on $o$ because $b$ is $\sigma(o)$-measurable (\S\ref{app:sf-defn}), and $F(z,a) = z$. Then $z_{t+1} = z_t$ within every episode, so $\mathcal{L}_{\mathrm{pred}} = 0$, and $\operatorname{law}(z) = \mu_0$.
\end{proof}

The atomless hypothesis is a population idealization. A corpus of finitely many scenes makes $\operatorname{law}(b)$ atomic and $\operatorname{law}(z)$ supported on at most that many points, which a normality test on far fewer embeddings per group cannot distinguish from $\mu_0$; the construction fails outright only in the point-mass limit of a single scene (\S\ref{app:sf-optimality}). Variance--covariance regularization, whitening, entropy penalties and SIGReg (Eq.~\eqref{eq:sigreg}) all look only at the marginal law, so none can exclude the collapse; a barrier has to act on the joint law of $z$ across time.

\subsection{The objective selects by predictability, and the collapse is an optimum}
\label{app:sf-optimality}

Since $F$ ranges over measurable functions, $\min_F\mathcal{L}_{\mathrm{pred}}$ equals the
\emph{irreducible risk}
$\mathcal{E}(E) = \mathbb{E}\|z_{t+1}-\mathbb{E}[z_{t+1}\mid z_t,a_t]\|_2^2$~\citep{Weghenkel2016GraphbasedPF}. Define the \emph{innovation sensitivity}
$\nu(E) = \mathbb{E}\operatorname{tr}\operatorname{Var}(E(o_{t+1})\mid \xi_t,a_t)$, the
next-embedding variance that survives knowing the full latent state and the action.

\begin{theorem}[Predictability bound]
\label{thm:sf-optimality}
For any measurable $E$, $\;\mathcal{E}(E)\geq\nu(E)\geq 0$, and $\mathcal{E}(E) = 0$ if and
only if $\sigma(z)$ is \emph{action-closed}, i.e.\ $z_{t+1} = \Gamma(z_t,a_t)$ almost surely
for some measurable $\Gamma$.
\end{theorem}

\begin{proof}
$\sigma(z_t,a_t)\subseteq\sigma(\xi_t,a_t)$ because $o_t = O(\xi_t)$, and coarser conditioning cannot reduce expected conditional variance, which gives $\mathcal{E}(E) = \mathbb{E}\operatorname{tr}\operatorname{Var}(z_{t+1}\mid z_t,a_t) \geq \nu(E)$; $\mathcal{E}(E) = 0$ iff $\operatorname{Var}(z_{t+1}\mid z_t,a_t) = 0$ a.s., i.e.\ iff $z_{t+1}$ is a measurable function of $(z_t,a_t)$.
\end{proof}

Reading only $b$ gives $\nu = 0$ and $\Gamma = \mathrm{id}$, hence $\mathcal{E} = 0$; reading
$x_t$ or $y_t$ pays the variance injected by $\varepsilon$ or $\zeta$, unless the feature read is itself action-closed (a component of $x_t$ that $(x_t,a_t)$ determines, such as arm state under deterministic dynamics), which attains $\mathcal{E} = 0$ as well. Under any marginal constraint, which \Cref{prop:sf-feasibility} shows the background can satisfy exactly, the collapse is therefore an optimum the objective cannot exclude. Among zero-loss solutions the theorem does not select; the implicit bias of training does~\citep{littwin2024how}. In our runs it selects the collapse, whose predictor is the identity: action sensitivity, the relative change of the predicted rollout when the
true actions are replaced by those of another episode, falls below $0.05$ in all fifteen seeds of the SIGReg ladder (\Cref{tab:sf-idm0}). In the tabular case the theorem is a corollary of \citet[Thm.~6]{tang2023understanding_self_predictive}, whose non-collapse guarantee needs semi-gradient targets and a near-optimal predictor; our collapsed runs use the former and lack the latter. Conversely, a background identical in every episode makes $\operatorname{law}(b)$ a point mass, so \Cref{prop:sf-feasibility} fails and SIGReg alone excludes the collapse: the fixed-camera environments of \S\ref{sec:long-horizon} need no inverse-dynamics term, and DROID, whose scenes change every episode, does (\S\ref{sec:droid}, Fig.~\ref{fig:anticollapse}).

\paragraph{Experiments of \citet{sobal2022slowfeatures}.}
The taxonomy and \Cref{thm:sf-optimality} account for three of their results. \emph{Fixed} distractors are fatal and \emph{changing} ones
harmless: fixed noise is background and becomes the
collapse target, changing noise is a distractor and is shed. Their three-dot probe errors
order as $\nu$ does: stationary $0.066$, action-controlled $0.277$, random $0.273$, chance
$0.235$. And their inverse-dynamics model recovers the action-controlled dot ($0.035$) but
loses the stationary one ($0.272$), as \Cref{prop:sf-idm} predicts.

\subsection{Pooling time into the sample axis: a real but small barrier}
\label{app:sf-barrier}

\Cref{prop:sf-feasibility} concerns the distribution of a single embedding, but SIGReg is computed on a minibatch. At exact collapse every sequence contributes $T$ identical copies of one embedding, and a grouping with time in the sample axis, such as a single test on the flattened $(NT, D)$
tensor, puts them into the same normality test.

\begin{proposition}[The barrier, per grouping]
\label{prop:sf-barrier}
Let $\kappa$ be the expected value of $\mathcal{T}_{\mathrm{EP}}$ (\S\ref{app:epps-pulley}) on $n$
samples drawn from its Gaussian target; it does not depend on $n$. At exact collapse, with the per-episode embeddings distributed as the target (the construction of \Cref{prop:sf-feasibility}), the SIGReg term exceeds its null value $\lambda\kappa$ by
\begin{equation}
\label{eq:sf-barrier}
\Delta = \lambda\,(T-1)\,\kappa
\end{equation}
whenever the grouping places time in the sample axis, so that a group holds several
timesteps of the same sequence, and by zero whenever it places time in the group axis.
\end{proposition}

\begin{proof}
$\mathcal{T}_{\mathrm{EP}}$ compares an empirical characteristic function, an unweighted average over the $n$ samples of a group, against its target, and scales the squared discrepancy by $n$. Repeating each of $k = n/T$ distinct values $T$ times leaves that average unchanged, so its variance about the target remains that of $k$ samples while the prefactor stays $n$, giving $\mathbb{E}[\mathcal{T}_{\mathrm{EP}}] = T\kappa$. With time in the group axis every group holds $n$ distinct embeddings.
\end{proof}

The grouping this paper uses, $(T,N,D)$, places time in the group axis: each normality test
sees one timestep of $N$ different sequences, so $\Delta = 0$ and SIGReg carries no barrier
against slow-feature collapse. Where a grouping does carry one, the barrier is
measured against the null value $\lambda\kappa$, not against the value the statistic takes
on a trained encoder, so a collapsed encoder need not score above a healthy one.
In \Cref{tab:sf-idm0} it does not prevent collapse: every run of the SIGReg ladder collapses,
later at larger $\lambda$ but to the same endpoint, with $\chi$ between $0.008$ and $0.016$
on every run past epoch~$50$.

\begin{table}[ht]
\centering
\footnotesize
\setlength{\tabcolsep}{3.1pt}   %
\caption{\textbf{Training without the inverse-dynamics term: the SIGReg
ladder (top) and time repulsion at its largest coefficient (middle).} Level-1 models at
$5$\,fps, $\gamma = 0$, three train seeds, $100$ epochs. These runs feed SIGReg the flattened
$(NT, D)$ tensor, one normality test over all $NT$ embeddings of the minibatch, which places
time in the sample axis, so they carry the $\Delta$ of \Cref{prop:sf-barrier};
the barrier is what the ladder probes. ``Onset'' is the first epoch at which Fr\'echet fidelity (Eq.~\eqref{eq:skill}) drops clearly below the zero-action floor, the next column the epoch at which the predictor's action sensitivity first falls below $0.05$ (seed spread $\pm 2$ epochs). Best fidelity is the highest before turning over; final fidelity and sensitivity are read over epochs $98$--$100$; $\chi$ is Eq.~\eqref{eq:chi} on the final checkpoint (one seed for the repulsion runs, three for the IDM-on row). Ladder rows carry no terminal value (at most two seeds reach epoch $100$).}
\label{tab:sf-idm0}
\begin{tabular}{llcccccc}
\toprule
$\lambda$ & $\lambda_{\mathrm{sim}}$ & Onset $\uparrow$ & Sens.\ $<0.05$ $\uparrow$ & Best fidelity $\uparrow$ & Final fidelity $\uparrow$ & Final sens. & $\chi$ \\
\midrule
$9.77\times10^{-6}$ & $0$ & $7$  & $13$ & $-0.64${\small\,(2.74)} & \textendash & \textendash & \textendash \\
$1.95\times10^{-5}$ & $0$ & $21$ & $20$ & $3.35${\small\,(1.40)}  & \textendash & \textendash & \textendash \\
$3.91\times10^{-5}$ & $0$ & $27$ & $26$ & $0.63${\small\,(5.49)}  & \textendash & \textendash & \textendash \\
$7.81\times10^{-5}$ & $0$ & $34$ & $39$ & $5.86${\small\,(5.10)}  & \textendash & \textendash & \textendash \\
$1.56\times10^{-4}$ & $0$ & $48$ & $57$ & $8.38${\small\,(2.29)}  & $-81.03$ & $0.014$ & $0.011$ \\
\midrule
$1.56\times10^{-4}$ & $-0.05$ & \textendash & none & $6.46$  & $\mathbf{2.44}${\small\,(5.80)} & $0.301${\small\,(0.001)} & $0.17$ \\
$1.56\times10^{-4}$ & $-0.5$  & \textendash & none & $-8.05$ & $-15.18${\small\,(17.39)} & $0.210${\small\,(0.024)} & $0.84$ \\
$1.56\times10^{-4}$ & $-5.0$  & \textendash & none & $-3.26$ & $-5.96${\small\,(2.35)} & $0.529${\small\,(0.083)} & $0.92$ \\
\midrule
IDM-on, $\gamma = 50$ & $0$ & none & none & $\mathbf{30.66}${\small\,(1.05)} & $\mathbf{30.66}${\small\,(1.05)} & $0.280${\small\,(0.016)} & $0.12$ \\
\bottomrule
\end{tabular}
\end{table}

\subsection{Time repulsion: a barrier that needs no action labels}
\label{app:sf-simt}

A collapsed encoder (\Cref{def:sf-collapse}) produces embeddings that do not change within an
episode. A direct counter is to reward change over time: we add the temporal-similarity term
$\mathcal{L}_{\mathrm{sim}} = \frac{1}{T-1}\sum_t\|z_{t+1}-z_t\|_2^2$ with a weight
$\lambda_{\mathrm{sim}} < 0$. Unlike SIGReg, it depends on consecutive pairs $(z_t,z_{t+1})$, not
only on the distribution of $z_t$, so \Cref{prop:sf-feasibility} does not apply to it. It
needs no action labels.

To see which features it favors, take a scalar feature $f(o_t)$ at unit variance. Its
temporal roughness $\varrho(f) = \mathbb{E}\,(f(o_{t+1})-f(o_t))^2$ measures how much it moves
between frames; slow feature analysis~\citep{wiskott2002slow} learns features that minimize
it, and $\lambda_{\mathrm{sim}} < 0$ does the opposite. Its irreducible risk $\mathcal{E}(f)$
(\S\ref{app:sf-optimality}) is the error of the best predictor of $f(o_{t+1})$ from $f(o_t)$
and $a_t$. Predicting no change is one such predictor, with error $\varrho(f)$, so
$\theta(f) = \mathcal{E}(f)/\varrho(f) \in [0,1]$ is the fraction of the feature's motion that
stays unpredictable~\citep{weghenkel2018slowness_predictability}: $\theta = 0$ for an arm
driven deterministically by the action, $1/2$ for a feature resampled independently at every
frame, $1$ for a random walk. A background feature costs nothing in either term. Encoding $f$
instead adds $\mathcal{E}(f)$ to the prediction loss and $\lambda_{\mathrm{sim}}\varrho(f)$ to
the temporal term, a total of $\varrho(f)\,(\theta(f)-|\lambda_{\mathrm{sim}}|)$, so the
objective prefers $f$ to the background exactly when $\theta(f) < |\lambda_{\mathrm{sim}}|$.

Let $\theta_F$ and $\theta_D$ be the ratios $\theta$ of features reading only the foreground $x_t$
and only the distractors $y_t$. The weights that prefer the foreground to the background but
still reject distractors satisfy $\theta_F < |\lambda_{\mathrm{sim}}| < \theta_D$, a window that is
nonempty exactly when $\theta_F < \theta_D$.

In the middle block of \Cref{tab:sf-idm0}, time repulsion does what no marginal regularizer can: all three repulsion runs keep action sensitivity far above the control at $\lambda_{\mathrm{sim}} = 0$, so the predictor no longer degenerates to the identity. But no repulsion run clears the zero-action floor, which the IDM-on baseline exceeds by a wide margin. On the IDM-on checkpoint, the trained predictor's one-step error is $0.51$ times that of the identity predictor. This ratio, $\hat\theta = 0.51$, is the measured $\theta$ of the encoded features: about half of their motion is unpredictable. Taking $\hat\theta$ as $\theta_F$, the window is thin at best: an uncorrelated distractor also has $\theta = 1/2$ (see above), a correlated one more. On $\theta$'s scale the three weights are $0.057$, $0.57$ and $5.7$, and the window does not select: the best run lies far below it, at $0.057$, and the run inside it, at $0.57$, fails worst. A small repulsion thus acts as a motion floor that keeps the predictor off the identity, not as a selector. Above $|\lambda_{\mathrm{sim}}| = 1$ the marginal cost $\varrho(f)(\theta(f)-|\lambda_{\mathrm{sim}}|)$ is negative for every feature, including a random walk ($\theta = 1$), so the term rewards any temporal change and strips even the episode identity: $\chi$ rises far above the healthy parent's at the two larger weights (\Cref{tab:sf-idm0}).

\subsection{The inverse-dynamics term is a barrier of a different order}
\label{app:sf-idm}

\begin{proposition}[IDM floor]
\label{prop:sf-idm}
If $E$ is slow-feature collapsed (\Cref{def:sf-collapse}), i.e.\ $z_t$ is
$\mathcal{C}$-measurable at every $t$ with $\mathcal{C} = \sigma(b)$, then, for any head $I$
and with $d_a$ the action dimension,
\begin{equation}
\label{eq:sf-idm-floor}
\mathcal{L}_{\mathrm{idm}} \;\geq\; \frac{1}{d_a}\,\frac{1}{T-1}\sum_{t=1}^{T-1}
\left(\operatorname{tr}\operatorname{Var}(a_t) - \operatorname{tr}\operatorname{Var}\!\left(\mathbb{E}[a_t\mid\mathcal{C}]\right)\right),
\end{equation}
a constant fixed by the data alone.
\end{proposition}

\begin{proof}
Under collapse $(z_t,z_{t+1})$ is $\mathcal{C}$-measurable, so for each $t$ the Bayes risk of regressing
$a_t$ on it is at least $\mathbb{E}\operatorname{tr}\operatorname{Var}(a_t\mid\mathcal{C})$; the law of total variance
and the $1/d_a$ normalization of Eq.~\eqref{eq:idm-loss} give Eq.~\eqref{eq:sf-idm-floor} term by term.
\end{proof}

With actions normalized to unit variance per dimension ($\operatorname{tr}\operatorname{Var}(a_t) = d_a$) and not predictable from episode identity, the floor is $1$: at $\gamma = 50$ the collapse costs $50$, against a SIGReg barrier $\Delta$ that is zero for our grouping. The floor does not shrink with batch size; it is a statement at exact collapse, and the near-collapsed runs of \Cref{tab:sf-idm0} ($\chi\leq 0.016$) are covered by the measurements, not by the bound. The term never penalizes keeping distractors~\citep{pathak2017curiosity,lamb2023guaranteed,levine2024multistep}. It is the only $\Omega(1)$ barrier here and needs action labels. This explains why \S\ref{sec:droid} adds it, and why SIGReg alone restores rank but not fidelity (Fig.~\ref{fig:anticollapse}).

\section{Extended related work}
\label{app:related-work}

\subsection{World models and planning}
Predicting the next sensory state has long been thought central to how animals learn the
structure of their world~\citep{tolman1948cognitive,rao1999predictive,clark2013whatever}. As early as
\citet{suttonbarto1981adaptive}, AI research has sought agents that build an internal
model of the world and use it for planning and control.
Modern instances learn such models directly from high-dimensional
observations, in pixel space or in a learned latent space, and scale to both
simulation and real-world robotics. Models such as
PlaNet, the Dreamer line, TD-MPC and
MuZero~\citep{haschmidhuber2018worldmodels,watter2015e2c,hafner2019planet,hafner2020dreamer,hafner2023dreamerv3,hansen2022tdmpc,hansen2024tdmpc2,schrittwieser2020muzero,micheli2023iris,janner2022diffuser,wu2022daydreamer}
plan or learn behavior against a reward or value signal. A distinct robotics line plans
directly with action-conditioned video prediction (visual foresight) to control
real robots from raw pixels~\citep{finn2017visualforesight,ebert2018visualforesight}.
More recently, task-agnostic
generative world models trained on large-scale real-world video produce realistic and
diverse simulations of physical
environments~\citep{bruce2024genie,yang2024unisim,brooks2024sora,nvidia2025cosmos,wu2024ivideogpt,du2023unipi},
showing that predictive modeling scales beyond task-specific datasets. Both directions
remain coupled to a reward signal or to modeling the pixels themselves.

JEPAs drop both couplings: they predict
\emph{latent} targets directly, with no decoder and no reward model~\citep{lecun2022path,assran2023ijepa,bardes2024vjepa}.
Recent action-conditioned approaches plan in this latent space~\citep{terver2026whatdrives,NWM,baldassarre2025dinoworld,wang2026adajepa,wang2026straightening,terver2026ebjepa,kuang2026lpwm,nilaksh2026reconstructionsemantics}.
They differ mainly in how the encoder is obtained: DINO-WM~\citep{zhou2025dinowm} trains a dynamics model on
\emph{frozen} pretrained features; V-JEPA\,2~\citep{assran2025vjepa2} pretrains a video JEPA and then attaches an action-conditioned predictor, \citet{terver2026whatdrives}
unifies the two in the JEPA-WMs framework; and end-to-end variants such as PLDM~\citep{sobal2025pldm} and LeWM~\citep{maes2026lewm} train encoder and predictor together from pixels.
H-JEPA builds directly on this end-to-end, action-conditioned setting and
extends it from a single latent space to a stack of JEPA world models at increasing
temporal abstraction.

\subsection{Preventing collapse in non-contrastive SSL}
Predicting latent targets risks a trivial \emph{collapsed} solution in which the
encoder and predictor emit a constant vector, driving the prediction loss to zero while
encoding nothing about the input. Contrastive methods~\citep{chen2020simclr,caron2021dino} avoid it with negatives; non-contrastive methods
instead regularize the embedding distribution: BYOL and SimSiam~\citep{grill2020byol,chen2021simsiam,ponce2026dualperspectives} through architectural asymmetries, VICReg and
Barlow Twins~\citep{bardes2022vicreg,zbontar2021barlow} through explicit
variance/covariance criteria. LeJEPA~\citep{balestriero2025lejepa} makes this a
distributional target, pushing embeddings toward an isotropic Gaussian via a sketched
normality test (SIGReg). Rectified LpJEPA~\citep{kuang2026rectifiedlpjepa}
extends this distribution-matching approach to rectified generalized Gaussian targets,
enabling explicit control over embedding sparsity.

Temporal prediction introduces a failure mode. \citet{sobal2022slowfeatures}
show that a JEPA can encode static background noise while discarding the moving agent:
the embeddings vary across episodes but remain constant within each episode, so
an identity predictor achieves zero prediction error while satisfying
variance--covariance constraints. They evaluate IDM,
which predicts actions from consecutive embeddings, as a way to retain information
about the moving agent. This use of inverse dynamics builds on representation learning
for curiosity-driven exploration~\citep{pathak2017curiosity}.
PLDM~\citep{sobal2025pldm} and the action-conditioned world models in
EB-JEPA~\citep{terver2026ebjepa} combine latent prediction with IDM and
variance--covariance regularization. Sensorimotor World Models~\citep{sensorimotorworldmodels}
instead use IDM as the sole regularizer alongside latent prediction.
Our LeWM+IDM level-1 model (\S\ref{sec:droid}) follows this use of inverse dynamics,
with SIGReg providing marginal anti-collapse regularization.

\subsection{Hierarchical World Models}
\label{sec:hier-abstraction}
Recent work builds hierarchical world models in which different levels occupy different
learned representation spaces. Hierarchical latent-variable video
models give each level its own latent and advance the coarser levels only once every
several frames, on a fixed stride (Clockwork VAE~\citep{saxena2021clockwork}) or at
learned boundaries (VTA~\citep{kim2019vta}), so the top levels come to represent
slower-changing content. STREAMER~\citep{mounir2023streamer} and
SUNTA~\citep{iiyama2026sunta} share a common structure: each reconstructs the input at the
lowest level but, at higher levels, predicts the \emph{latent} of the level below rather
than pixels, segmenting the stream into events and
summarizing each into a coarser, lower-detail code. Hierarchies of distinct
representations also appear in recent language models, via content-based dynamic chunking~\citep{pagnoni2024blt,hwang2025hnet} or a
coarse latent predicting several steps ahead~\citep{shi2026hilp}, though there the
hierarchy is built over text tokens rather than perceptual observations of an environment.

None of these methods is action-conditioned or used for planning, and all model the raw
input at the lowest level: reconstructing pixels (the video models) or predicting
tokens (the language models). H-JEPA differs on both counts: it is latent-predictive at
\emph{every} level (even the base is a JEPA trained without
reconstruction), and it is \emph{action-conditioned} and used for \emph{planning and
control} (\S\ref{sec:method-planning}).

\subsection{Hierarchical control}

\begin{table}[t]
\centering
\caption{\textbf{Model-based hierarchical control methods along four axes that define H-JEPA.}
\emph{Pixels}: operates from high-dimensional visual observations rather than a
low-dimensional state.
\emph{Task-agnostic}: reward-free training with no task signal; the task objective
(a goal or a planning cost) is defined only at test time.
\emph{Reconstruction Free}: learns representations and dynamics with no reconstruction
objective, predicting latent targets rather than decoding observations.
\emph{Hierarchy of learned latent spaces}: a stack
of \emph{distinct} latent spaces, each a coarser abstraction of the one below that
discards detail (vs.\ a single/shared space, parallel spaces, or a hierarchy of
decisions only). \cmark~yes, \xmark~no, \pmark~partial. H-JEPA is the only method satisfying all four;
HWM, our baseline, differs on exactly one axis.}
\label{tab:hier-control}
\resizebox{\linewidth}{!}{%
\begin{tabular}{lcccc}
\toprule
\makecell[l]{Method} & \makecell{Pixels (In)} & \makecell{Task-Agnostic\\World Models} & \makecell{Reconstruction\\Free} & \makecell{Hierarchy of Learned\\Latent Spaces} \\
\midrule
Classical hierarchical MPC~\citep{fang2019cavin,li2021latentaction,koegel2022safehmpc} & \xmark & \cmark & \xmark & \xmark \\
Director~\citep{hafner2022director} & \cmark & \xmark & \xmark & \xmark \\
Puppeteer~\citep{hansen2025puppeteer} & \cmark & \xmark & \cmark & \xmark \\
SPlaTES~\citep{guertler2025splates} & \xmark & \xmark & \cmark & \pmark \\
Schiewer et al.~\citep{schiewer2024limits} & \xmark & \xmark & \xmark & \cmark \\
THICK~\citep{gumbsch2024thick} & \cmark & \pmark & \xmark & \xmark \\
GCP~\citep{pertsch2020gcp} & \cmark & \cmark & \xmark & \xmark \\
HWM~\citep{zhang2026hierarchical} & \cmark & \cmark & \cmark & \xmark \\
\midrule
\textbf{H-JEPA (ours)} & \cmark & \cmark & \cmark & \cmark \\
\bottomrule
\end{tabular}%
}
\end{table}

Hierarchical control has a long history in reinforcement learning, where the
hierarchy is one of \emph{decisions}: a higher-level policy commits to a subgoal
or sub-policy that a low-level controller executes over many steps, as in the
options framework and feudal approaches~\citep{sutton1999options,dayan1992feudal,vezhnevets2017feudal,nachum2018hiro}.
Classical hierarchical model predictive control makes an analogous split for
optimal control, pairing a high-level planner over subgoals
with a low-level tracking controller~\citep{fang2019cavin,li2021latentaction,koegel2022safehmpc}.
Both organize \emph{control} across timescales: RL
methods learn subgoals and controllers by maximizing reward, and classical
MPC derives its levels from supplied dynamics models and hand-designed state
abstractions.

More recent work brings hierarchy into deep model-based RL, learning the
hierarchical world model itself within a reward-driven, task-specific training
loop~\citep{hafner2022director,hansen2025puppeteer,gumbsch2024thick,guertler2025splates,schiewer2024limits}.
Setting reward aside, these methods differ in their hierarchy of
\emph{representations}. \citet{schiewer2024limits} learn a separate
state-space model per level, though the models are generative and reward-driven and
each upper latent is only indirectly coupled to the level below.
The high-level space of SPlaTES~\citep{guertler2025splates} is an
independent low-dimensional projection of the state; THICK~\citep{gumbsch2024thick}
factorizes its latent into fast- and slow-changing components and uses sparse
changes in the slow component to define event-scale time steps, all within a
single latent space; and
the two world models of Puppeteer~\citep{hansen2025puppeteer} share one grain, bridged
by a hand-engineered command interface. None of
them builds the hierarchy as a stack of \emph{distinct} latent spaces, each
recursively derived from the one below at increasing temporal abstraction, as
H-JEPA does.

Closest to us in setting are pixel-based methods that plan toward a goal
zero-shot at test
time. GCP~\citep{pertsch2020gcp} recursively predicts visual subgoals, decoding
intermediate \emph{images}. Our baseline
HWM~\citep{zhang2026hierarchical} is a JEPA-based hierarchical world model with
predictors over different time horizons. It plans hierarchically at test time, but
all its predictors operate in a single shared latent space rather than a hierarchy of distinct representations.

\end{document}